\documentclass[sigconf,final]{acmart}

\usepackage{amsmath}
\usepackage{amsthm}
\usepackage{mathtools}
\usepackage{bm}

\usepackage{microtype}

\usepackage{algorithm}
\usepackage{algorithmic}

\usepackage{booktabs}
\usepackage{multirow}
\usepackage{makecell}

\usepackage{graphicx}
\usepackage{subcaption}
\usepackage{tikz}

\usepackage{listings}
\usepackage{newfloat}

\usepackage{cleveref}

\usepackage[normalem]{ulem}

\usepackage{graphicx}
\usepackage{booktabs}
\usepackage{multirow}
\usepackage{hyperref}
\newcommand{\distill}{\textsc{Distill-Belief}}

\newcommand{\eg}{\textit{e.g.}}

\newcommand{\SR}{\mathrm{SR}}
\newcommand{\TE}{\mathrm{TE}}
\newcommand{\REV}{\mathrm{REV}}
\newcommand{\LPS}{\mathrm{LPS}}
\newcommand{\UQ}{\mathrm{UQ}}
\newcommand{\SLE}{\mathrm{SLE}}
\newcommand{\FPE}{\mathrm{FPE}}
\AtBeginDocument{}

\setcopyright{acmlicensed}
\setcopyright{none}

\copyrightyear{2026}
\acmYear{2026}

\acmConference[KDD '26]{The 32nd ACM SIGKDD Conference on Knowledge Discovery and Data Mining}
{August 9--13, 2026}{International Convention Center Jeju (ICC Jeju), Jeju, Republic of Korea}

\acmISBN{978-1-4503-XXXX-X/2018/06}

\begin{document}

\title{Distill-Belief: Closed-Loop Inverse Source Localization and Characterization in Physical Fields}


\author{Yiwei Shi $^1$, Zixing Song $^1$, Mengyue Yang $^1$, Cunjia Liu $^2$, Weiru Liu $^1$}
\affiliation{%
  \institution{$^1$ SEMT, University of Bristol, Bristol, UK, $^2$ LUCAS,  Loughborough University, Loughborough, UK}
  \city{}
  \country{}
  }

\renewcommand{\shortauthors}{Trovato et al.}

\begin{abstract}
{Closed-loop inverse source localization and characterization (ISLC) requires a mobile agent to select measurements that localize sources and infer latent field parameters under strict time constraints.}
{The core challenge lies in the belief-space objective: valid uncertainty estimation requires expensive Bayesian inference, whereas using fast learned belief model leads to reward hacking, in which the policy exploits approximation errors rather than actually reducing uncertainty.}
{We propose \textbf{Distill-Belief}, a teacher--student framework that decouples correctness from efficiency.
A Bayes-correct particle-filter teacher maintains the posterior and supplies a dense information-gain signal, while a compact student distills the posterior into belief statistics for control and an uncertainty certificate for stopping.
At deployment, only the student is used, yielding constant per-step cost.}
{Experiments on seven field modalities and two stress tests show that Distill-Belief consistently reduces sensing cost and improves success, posterior contraction, and estimation accuracy over  baselines, while mitigating reward hacking.}
\end{abstract}


\keywords{Do, Not, Use, This, Code, Put, the, Correct, Terms, for, Your, Paper}

\received{20 February 2007}
\received[revised]{12 March 2009}
\received[accepted]{5 June 2009}

\maketitle

\section{Introduction}
\label{sec:Intro}

{Autonomous scientific sensing missions increasingly operate in the wild: a drone or ground robot is dispatched after a suspected gas \cite{dogniaux2025global,cusworth2024quantifying}, pollutant \cite{zhang2021system,LiuAerial2020}, or radiation event \cite{tran2022internet,liang2021electrostatic} , and must decide \emph{where to measure next} under strict time and energy budgets. Each measurement is noisy and expensive, and there is often no dense task reward-what matters is whether the system can \emph{localize and characterize the source with calibrated uncertainty} \cite{Holzschuh2023SolvingIP,Chung2022DiffusionPS,Papamakarios2018SequentialNL} quickly enough to enable downstream intervention \cite{Foster2021DeepAD,Kleinegesse2020BayesianED,Foster2019VariationalBO}.} 

{We study this setting through closed-loop \textit{inverse source localization and characterization} (ISLC), also known as the {\textbf{Source Term Estimation}} problem \cite{Chow2005SourceIF,Monache2008BayesianIA,Keats2007BayesianIF} in physical fields \cite{Jervis2025GlobalES}. An agent sequentially samples a field governed by a transport model with {the unknown parameters vector $\boldsymbol{\Theta}$} \cite{Hanna1982HandbookOA,Lushi2009AnIG} (e.g., source location/strength and environmental factors), and updates a Bayesian posterior $p(\boldsymbol{\Theta}\mid o_{1:t},\boldsymbol{p}_{1:t})$ from noisy observations. {The scientific objective is not merely to reach a high-signal region, but to \emph{actively choose measurements \cite{vergassola2007infotaxis,vergassola2007infotaxis,Rigolli2021LearningTP} that contract the posterior} and yield calibrated uncertainty so that the episode can terminate when uncertainty falls below an application-specified tolerance\cite{Stockie2011TheMO,Turner1994WorkbookOA,Hanna1982HandbookOA}.}} Closed-loop ISLC exposes a fundamental tension between scientific correctness and operational deployment. First, the objective lives in \emph{belief space}: we seek posterior contraction and uncertainty calibration, not a heuristic proxy in observation space (otherwise the agent can chase transient peaks without reliably reducing epistemic uncertainty). Second, scientific field tasks rarely provide reliable \emph{dense} task rewards \cite{Shi2024AutonomousGD,Liu2022MetaRewardNetID,Park2022SourceTE}: \textbf{success} is often sparse, late, or even undefined, and the episode should terminate when uncertainty is sufficiently low. Meanwhile, real deployments require \emph{real-time} decision making: per-step computation cannot scale with expensive Bayesian inference. Third, if a learned belief surrogate is used simultaneously as the policy input and as the basis for intrinsic rewards or stopping, the agent can exploit approximation artifacts-inflating reward or triggering premature stopping without genuinely contracting the true posterior. These considerations impose a  coupled set of requirements that any deployable ISLC algorithm must satisfy.

Formally, closed-loop ISLC must satisfy \emph{four coupled requirements}:
\textbf{(R1)} optimize in \emph{belief space} to directly reduce epistemic uncertainty;
\textbf{(R2)} learn under \emph{no reliable dense task reward} (i.e., to address success is sparse/late or undefined);
\textbf{(R3)} remain \emph{deployable}, i.e., to address per-step computation cannot scale with expensive Bayesian inference;
{\textbf{(R4)} be \emph{robust to reward hacking}, the policy must not be able to manipulate approximation errors in a learned belief surrogate (e.g., under-estimated posterior spread) to obtain high intrinsic reward or to satisfy the stopping rule without genuine posterior contraction.}
{Table~\ref{tab:reqs} summarizes requirement coverage for representative design families; none satisfies (R1)-(R4) simultaneously.} To satisfy (R1)-(R4) simultaneously, we argue that Bayes correctness must be decoupled from deployment-time computation.
The learning signal should be derived from a \emph{Bayes-consistent} belief update, yet the deployed controller cannot afford to run such updates online.
This leaves a narrow set of practical designs: a \emph{Bayes-correct teacher} computes an information-theoretic objective in belief space, while a {\emph{fast-learning student} compresses the teacher posterior over $\Theta$ into  belief statistics (e.g., mean and diagonal covariance) that can be updated in \emph{O(1) time per step with respect to the particle budget}.} Notably, only the actor-critic updates the policy parameters; the PF teacher is \emph{not} a planner and provides only Bayes belief updates for reward computation and distillation targets.

\begin{table}[htbp]
\centering
\caption{Requirement coverage.}
\label{tab:reqs}
\resizebox{\linewidth}{!}{%
\begin{tabular}{lcccc}
\toprule
Method & (R1) belief-space & (R2) no dense reward & (R3) test-time & (R4) no hacking \\
\midrule
PF info planner \cite{hutchinson2019source,seo2025kalman,park2021autonomous} & \checkmark & \checkmark & \texttimes & \checkmark \\
Obs-reward RL \cite{lee2025enhanced,zhao2022deep}     & \texttimes & \texttimes & \checkmark & \checkmark \\
Learned belief \cite{hutchinson2018information,vergassola2007infotaxis,masson2009chasing} & \checkmark & \checkmark & \checkmark & \texttimes \\
\textbf{Distill-Belief (ours)} & \checkmark & \checkmark & \checkmark & \checkmark \\
\bottomrule
\end{tabular}}
\end{table}

A key question is why we distill beliefs rather than directly learning in an end-to-end manner.
If we remove the teacher and define intrinsic rewards using the same learned belief that conditions the policy, the policy may exploit modeling errors to artificially increase the reward or reduce the spread certificate without genuinely reducing posterior uncertainty. Conversely, if we keep Bayes-consistent belief updates to prevent such artifacts, then inference remains in the deployment loop with cost scaling linearly in the particle budget, violating real-time constraints.
Belief distillation resolves this tension: it transfers the teacher posterior into a parametric student, yielding constant-time, uncertainty-calibrated belief statistics for both control and stopping \emph{while keeping the reward Bayes-aligned and computed exclusively from the teacher during training}.
Without distillation, one must trade off \emph{deployability} (PF at test time) against \emph{statistical alignment and robustness} (learned belief defining both reward and control), and cannot satisfy (R1)--(R4) at once.

To bridge these gaps, we propose a \textbf{teacher-student belief-optimization framework} for closed-loop ISLC. A particle-filter (PF) teacher maintains a Bayes-consistent posterior over the parameter vector $\boldsymbol{\Theta}$. This teacher provides a dense intrinsic reward defined as the discrete KL divergence between consecutive beliefs, serving as a high-fidelity proxy for one-step information gain. {A fast student distills the teacher posterior into a compact diagonal-Gaussian belief. The distilled belief yields constant-time features for a belief-conditioned actor-critic, and provides a spread-based uncertainty certificate for principled stopping.} At test time, we discard the PF entirely and rely only on the student's  belief statistics, making inference and cessation independent of particle budget. This separation keeps intrinsic rewards Bayes-aligned while keeping deployment free of PF inference.

{Our main contributions are: (1) We cast ISLC as belief-space control and introduce a coupled inference--execution architecture that decouples Bayes-correct objectives from deployment-time computation via PF teaching and student belief approximation. (2) We propose a dense information-gain intrinsic reward based on the one-step KL divergence between consecutive teacher posteriors, directly aligning RL optimization with posterior contraction while preventing reward hacking \emph{by construction}: intrinsic rewards are computed only from the PF teacher posterior, while the learned student belief is used solely to condition the policy (which outputs the next sensing action) and to compute a deployment-time stopping certificate. (3) We enable deployable and reliable closed-loop inference via student beliefs and a spread-based stopping certificate that explicitly controls the accuracy--budget trade-off, with the PF teacher entirely removed at test/deployment time.}

\section{Related Work}
\label{sec:related}

\subsection{Information-Theoretic Planning for ISLC}
Closed-loop ISLC, also referred to as source term estimation, is a recurring primitive in field sensing: a mobile agent must adaptively decide where to measure next to localize hidden emitters and estimate physical parameters with calibrated uncertainty. Representative deployments include source-term estimation for atmospheric releases with mobile robots/UAVs \cite{bourgault2002information,hollinger2014sampling}, radioactive source localization for safety monitoring \cite{jarman2011bayesian,huo2020autonomous}, and broader contaminant/source reconstruction settings in setting sensing \cite{bagtzoglou2005mathematical,jiang2021two}. Methodologically, these problems are closely related to sequential Bayesian experimental design \cite{shi2025attention,Shi2024AutonomousGD}, where actions are chosen to maximally reduce uncertainty about unknown parameters value $\Theta$.
A route to closed-loop ISLC \cite{shi2025attention} couples \emph{Bayesian sequential inference} \cite{arulampalam2002tutorial,johansen2009tutorial,zhang2023optimal} with \emph{information-theoretic action selection}: the agent \cite{yan2025goal,da2025new} maintains a posterior over unknown source/transport parameters $\Theta$ and chooses the next sensing action to maximally contract this belief. Representative information-based search frameworks instantiate this idea with sequential Monte Carlo \cite{arulampalam2002tutorial,zhang2023optimal} / particle \cite{kantas2015particle,andrieu2010particle} to approximate $p(\Theta \mid o_{1:t}, \mathbf{p}_{1:t})$, and evaluate candidate actions using information utilities such as expected information gain \cite{chaloner1995bayesian}, mutual information \cite{singh2006efficient}, or KL divergence \cite{rahbar2019algorithm} between beliefs. Alongside these explicit information-gain planners, many \emph{non-learning / static} strategies adopt a similar two-block structure---(i) an estimation module (Bayes/PF update) and (ii) a greedy controller that optimizes an uncertainty-related surrogate. Typical examples include \emph{Infotaxis} \cite{hutchinson2018information,vergassola2007infotaxis,masson2009chasing} and \emph{Entrotaxis} \cite{hutchinson2018entrotaxis,zhao2024regression,ristic2016study}, which drive exploration by reducing belief uncertainty (e.g., variance/entropy), and \emph{dual-control-based approaches (DCEE)} ~\cite{li2024cooperative,chen2021dual} explicitly trade off exploitation and exploration through composite objectives that combine progress-to-estimate and uncertainty-reduction terms. \textcolor{cyan}{However, the online control loop typically requires repeated belief updates and (often) lookahead evaluation of candidate actions, so the per-step cost grows with the particle budget and planning horizon, which is a key obstacle for real-time deployment and large-scale evaluation (violating deployability requirement (R3) in Sec.~\ref{sec:Intro})}.

\subsection{RL for Active Sensing and Localization}
A complementary direction amortizes decision making with RL \cite{Mnih2015HumanlevelCT,Schulman2017ProximalPO,Lillicrap2015ContinuousCW}, learning a policy that maps observations (and possibly belief features) to sensing actions. In ISLC and related active sensing tasks \cite{Ristic2016ASO}, \textit{actor-critic} \cite{Mnih2016AsynchronousMF,Lillicrap2015ContinuousCW} methods are frequently adopted, where the state representation augments raw observations with compact belief summaries from PF, e.g., posterior moments or parametric compressions such as Gaussian mixture models \cite{Park2022SourceTE,Ladosz2020GaussianPB}. This replaces explicit online planning with a single policy forward pass, improving deployment-time efficiency. To further improve deployability, many works \cite{Wang2023RoboticOS,Hu2019PlumeTV,Wang2021OlfactoryBasedNV} compress particle-based posteriors into low-dimensional belief representations (moments, mixture fits, or learned set encoders) and, more broadly, use amortized inference to predict posterior statistics in (approximate) constant time, avoiding iterative Bayesian updates during control \cite{Barto1983NeuronlikeAE}. Nevertheless, RL-based approaches often rely on observation-space reward shaping \cite{Pathak2017CuriosityDrivenEB,Ng1999PolicyIU}(e.g., concentration improvement) or sparse terminal success, which can be misaligned with posterior contraction and encourage shortcut behaviors. Moreover, in many scientific field tasks, \textit{success} \cite{Shi2024AutonomousGD} is not explicitly labeled and should instead be determined implicitly by sufficiently low uncertainty; self-cessation and goal-detection mechanisms address this by using belief dispersion as a stopping trigger~\cite{Shi2024AutonomousGD}. \textcolor{cyan}{Overall, RL-based methods tend to satisfy deployability (R3), but may struggle with belief-space objectives (R1) and sparse/implicit supervision (R2) unless the learning signal is carefully designed.} Using approximate beliefs inside the control loop can also break scientific semantics. If the same learned belief surrogate both conditions the policy and defines intrinsic rewards or stopping \cite{Burda2018ExplorationBR}, the agent may exploit surrogate artifacts, earning reward or stopping early without genuine Bayes-posterior contraction (violating (R4)). \textcolor{cyan}{This motivates teacher-student \cite{Hinton2015DistillingTK,Romero2014FitNetsHF,vezhnevets2017feudalNHRL,bacon2017option} designs that decouple Bayesian objectives from deployment-time computation: a Bayes-correct inference module can serve as a source of supervision, while a fast amortized model provides belief statistics for real-time control.}

\begin{figure}[!h]
  \centering
  \begin{subfigure}{0.48\linewidth}
    \centering
    \includegraphics[width=\linewidth]{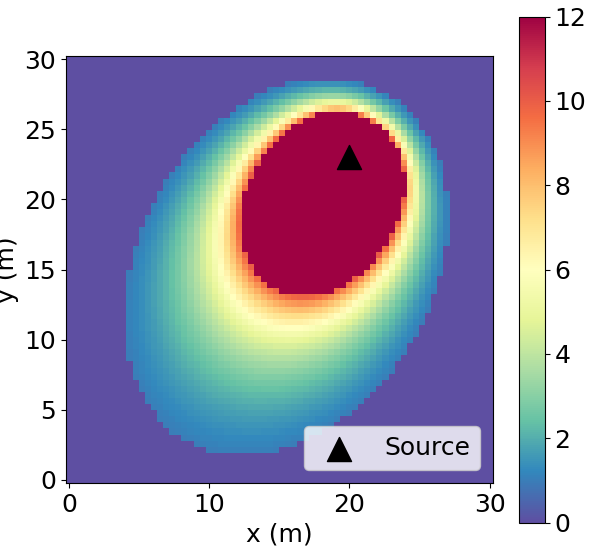}
    \caption{Steady-state concentration field generated by a point source}
    \label{fig:sub1}
  \end{subfigure}
  \hfill
  \begin{subfigure}{0.48\linewidth}
    \centering
    \includegraphics[width=\linewidth]{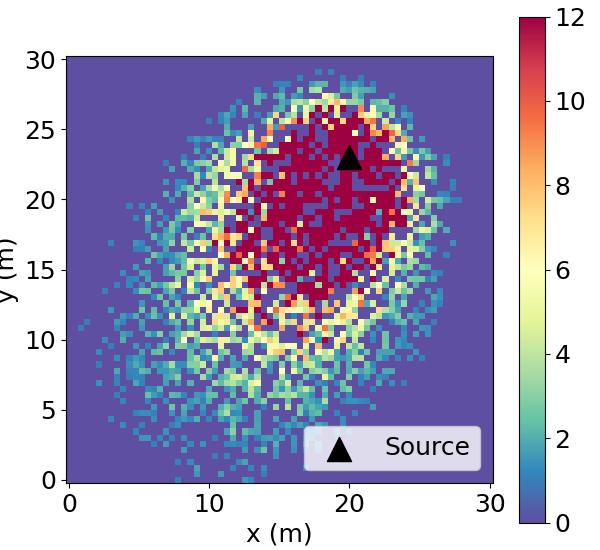}
    \caption{Noisy observations available to the agent}
    \label{fig:sub2}
  \end{subfigure}
  \caption{ Gaussian Plume Model for source localization.}
  \label{fig:GPM_visualization}
\end{figure}

\section{Preliminaries}
\label{sec:Preliminaries}
\subsection{Unified Framework for Field Modeling}
A wide range of natural phenomena, including \textit{pollutant dispersion}, \textit{gas diffusion}, and \textit{electric field distributions}, can be described within a unified physical framework. Despite their apparent differences, these processes are governed by three fundamental terms: \textbf{diffusion}, \textbf{convection}, and \textbf{external sources}. Such terms are commonly captured by the general \textit{convection-diffusion equation} (CDE) \cite{holley1969unified}, which provides a versatile mathematical abstraction:
\begin{equation}
\alpha \nabla^2 \phi(x,y) - \vec{v} \cdot \nabla \phi(x,y) + \gamma \phi(x,y) + L(x,y) = 0,
\end{equation}
where $\phi(x,y)$ denotes the field variable (e.g., concentration, temperature, or potential). The terms $\alpha \nabla^2 \phi$, $-\vec{v}\cdot\nabla \phi$, and $L(x,y)$ capture diffusion, transport (advection), and external sources, respectively, while $\gamma$ accounts for linear decay or reaction effects. In this work, we adopt the steady-state linear CDE as an effective field model that provides a unified abstraction of \textbf{diffusion–transport–source} mechanisms. Different field modalities correspond to different physical interpretations/units of $\phi$ and different parameterizations of $(\alpha,\vec v,\gamma,L)$, potentially together with modality-specific boundary conditions and observation/noise models. With appropriate parameter choices, this formulation covers a broad class of spatial fields, from heat conduction and pollutant dispersion to other source-driven diffusion–transport phenomena.

\subsection{Gaussian Plume Model}

As a classical analytical instantiation of the convection--diffusion framework, the \textit{Gaussian Plume Model} (GPM), illustrated in Fig.~\ref{fig:GPM_visualization}, provides a steady-state solution that balances modeling fidelity and computational efficiency. Under suitable assumptions, the resulting field distribution is given by
\begin{equation}
\phi(x, y) =
\frac{q_s}{4 \pi \alpha \|\boldsymbol{p}-\boldsymbol{p}_s\|}
\exp\left(
-\frac{\|\boldsymbol{p}-\boldsymbol{p}_s\|}{\lambda}
-\frac{(x - x_s)u_x + (y - y_s)u_y}{2 \alpha}
\right),\nonumber
\end{equation}
where $q_s$ denotes the source strength, $\boldsymbol{p}=(x,y)$ and $\boldsymbol{p}_s=(x_s,y_s)$ represent the observation point and source location, respectively, and $(u_x,u_y)$ are the components of the convection velocity $\vec{v}$. The parameter $\lambda$ controls the exponential decay, and $\alpha$ is the diffusion coefficient, consistent with the notation in the general CDE. Owing to its analytical form, the GPM is widely used for efficient modeling of steady-state dispersion phenomena.

\subsection{Partially Observable MDP}

In practical field modeling scenarios, uncertainties arise due to incomplete state information, sensor noise, and unknown environmental parameters. These challenges naturally motivate a probabilistic sequential decision-making formulation based on the \textit{Partially Observable Markov Decision Process (POMDP)}. A POMDP is defined by the tuple $(S, \Omega, A, T, O, R, \gamma)$, where $S$ denotes the state space, $A$ the action space, $\Omega$ the observation space, $T$ the state transition model, $O$ the observation model, $R$ the reward function, and $\gamma \in (0,1]$ the discount factor. At each time step $t$, the agent receives an observation $o_t \in \Omega$ conditioned on the current state $s_t$ and the previous action $a_{t-1}$, according to $O(o_t \mid s_t, a_{t-1})$. After executing action $a_t$, the environment transitions to a new state $s_{t+1} \sim T(s_{t+1} \mid s_t, a_t)$. Due to partial observability, policies are typically conditioned on belief states or histories of actions and observations. We define the belief over field-model parameters as
$b_t(\boldsymbol{\Theta}) \triangleq p(\boldsymbol{\Theta} \mid o_{1:t})$, where $\boldsymbol{\Theta}$ denotes the {unknown parameters vector} of the field model (e.g., source location, strength, and setting factors). In many ISLC settings, the environment does not provide a dense task reward. Instead, we optimize an intrinsic reward that explicitly measures \emph{posterior contraction}:$
r_t^{IG} \triangleq D_{\mathrm{KL}}\!\left(b_t(\boldsymbol{\Theta}) \,\Vert\, b_{t-1}(\boldsymbol{\Theta})\right)$,
which serves as a one-step information-gain proxy. The overall objective is to learn a policy $\pi(a_t \mid \psi_t)$ that maximizes the expected discounted cumulative information gain, $\max_{\pi}\ \mathbb{E}_{\pi}\left[\sum_{t=1}^{T} \gamma^{t-1} r_t^{IG}\right]$, where $\psi_t$ (belief state) is a belief policy input (e.g., concatenating the observation, agent position, and belief features). Finally, we allow early termination using an uncertainty certificate to realize an explicit accuracy--budget trade-off.  Accordingly, \textcolor{cyan}{our objective is to infer the posterior distribution
$p(\boldsymbol{\Theta} \mid o_{1:t})$
over the field-model parameters
$\boldsymbol{\Theta} = [x_s, y_s, q_s, u_s, \phi_s, \alpha, \lambda]^\top$
from sequential observations $o_{1:t}$ and belief state $\psi_t$ at time step $t$ within the POMDP}.

\section{Methodology}
\label{sec:method}


We formulate closed-loop ISLC as a belief-space control problem whose objective is posterior contraction.
Subsec.~\ref{sec:method:overview_} defines the belief state used by the policy.
Subsec.~\ref{sec:method:pf}-\ref{sec:method:student} describe how the belief is maintained and amortized.
Subsec.~\ref{sec:method:kl} specifies the KL-based intrinsic reward used for training.
Subsec.~\ref{sec:policy}-\ref{sec:method:stop} present the belief-conditioned actor-critic and the stopping criterion.

\begin{figure}[htbp]
  \centering
  \includegraphics[width=\linewidth]{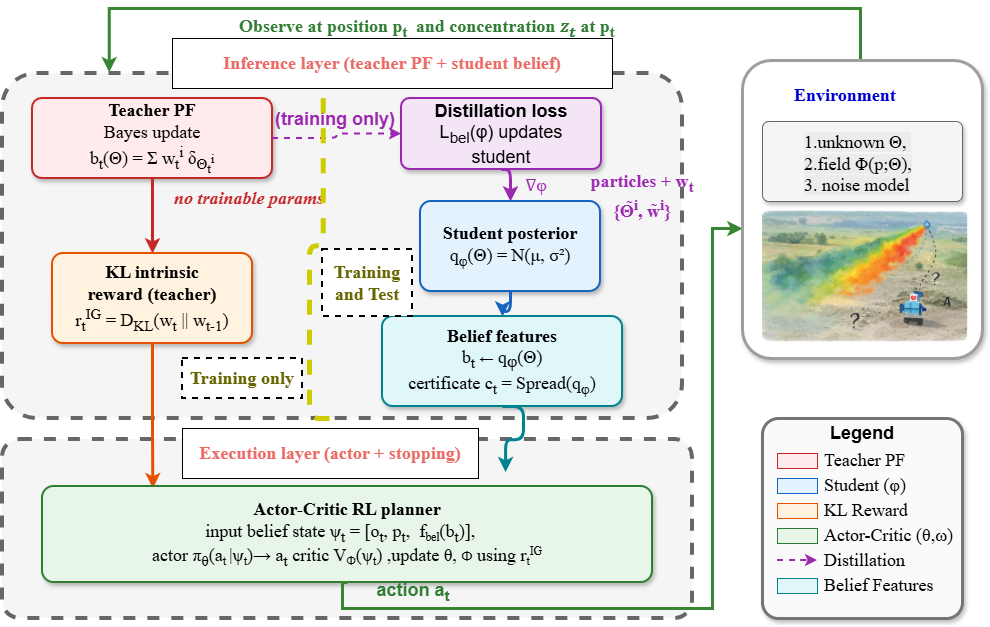}

  \caption{Teacher–Student Belief Distillation}
  \label{fig:framework}
\end{figure}

\subsection{Problem Setup and Belief-State Interface}
\label{sec:method:overview_}

Building on the POMDP formulation in Sec.~\ref{sec:Preliminaries}, we maintain a Bayesian belief
$b_t(\boldsymbol{\Theta}) \approx p(\boldsymbol{\Theta}\mid o_{1:t},\boldsymbol{p}_{1:t})$
over the \emph{full} parameter vector $\boldsymbol{\Theta}\in\mathbb{R}^d$ (e.g., source location/strength and environmental factors) and use it for closed-loop control.
At time $t$ we form a belief state
$\boldsymbol{\psi}_t=\big[ o_t^\top,\boldsymbol{p}_t^\top, f_{\text{bel}}(b_t)^\top \big]^\top$,
where $f_{\text{bel}}(b_t)$ is the belief distribution (Sec.~\ref{sec:policy}), and act via $\pi_\theta(a_t\mid \boldsymbol{\psi}_t)$. We realize this with two coupled layers.
\textbf{Inference layer:} a Bayes-consistent PF teacher updates $b_t$ and defines a KL-based intrinsic reward in belief space, while a student amortizes the teacher posterior via distillation.
\textbf{Execution layer:} an actor--critic learns $\pi_\theta$ conditioned on the student belief features and terminates when a spread-based uncertainty certificate falls below $\zeta$.
The PF teacher has no trainable parameters and never optimizes a policy; intrinsic rewards are computed \emph{only} from the teacher during training to prevent reward hacking.
At test time (deployment), we discard the teacher and run only the student and policy, yielding constant-cost belief features and stopping statistics independent of the particle budget (Framework in Fig \ref{fig:framework}).

\subsection{Teacher Belief via Particle Filtering}
\label{sec:method:pf}

The PF teacher is the only component where we require Bayes-correct updates.
It maintains a weighted particle approximation to the posterior
$b_t(\boldsymbol{\Theta})
= \sum_{i=1}^{N} w_t^{(i)}\, \delta_{\boldsymbol{\Theta}_t^{(i)}}(\boldsymbol{\Theta})$, where $\delta$ denotes the Dirac mass and each particle carries a parameter vector
$\boldsymbol{\Theta}_t^{(i)} \in \mathbb{R}^d$.
Let $\ell(o_t \mid \boldsymbol{p}_t, \boldsymbol{\Theta})$ be the likelihood induced by the Gaussian plume model in Sec.~\ref{sec:Preliminaries}.
In ISLC we treat $\boldsymbol{\Theta}$ as static {during an episode} and adopt the bootstrap choice $q \approx T$ in sequential importance sampling, so the weights are updated by simple reweighting:
$\tilde{w}_t^{(i)} ~\propto~ w_{t-1}^{(i)} \,
\ell\!\big(o_t \mid \boldsymbol{p}_t, \boldsymbol{\Theta}_{t-1}^{(i)}\big),
w_t^{(i)} = \frac{\tilde{w}_t^{(i)}}{\sum_j \tilde{w}_t^{(j)}}$.
We monitor the effective sample size
$\mathrm{ESS} = 1 / \sum_i (w_t^{(i)})^2$ and apply systematic resampling when $\mathrm{ESS} < \rho N$. To mitigate sample impoverishment after resampling without biasing the posterior
$\pi_t(\boldsymbol{\Theta}) \propto p(\boldsymbol{\Theta}\mid o_{1:t},\boldsymbol{p}_{1:t})$,
we apply a short Metropolis-Hastings (MH) move with a Gaussian random-walk proposal:
$ \boldsymbol{\Theta}_{\text{new}}
= \boldsymbol{\Theta}
+ h_t \boldsymbol{\Sigma}_{\Theta,t}^{1/2} \boldsymbol{\xi},
\boldsymbol{\xi} \sim \mathcal{N}(\boldsymbol{0}, \mathbf{I})$,
where $\boldsymbol{\Sigma}_{\Theta,t}$ is the current weighted covariance of the particles (in the full $\boldsymbol{\Theta}$-space) and $h_t$ is a step size.
The proposal is accepted with the usual MH ratio; this leaves $\pi_t$ invariant.
The PF teacher is where \emph{Bayes correctness} lives: it yields a statistically consistent posterior increment under mild assumptions, supports standard SMC diagnostics (ESS, resampling, MH moves), and never learns in weight space, so the target posterior remains unaffected by adaptation.

\subsection{Student Posterior via PF Distillation}
\label{sec:method:student}

While the PF teacher is Bayes-correct, its per-step cost scales with the number of particles $N$, which is undesirable at test time.
We therefore introduce a \emph{student} that serves as a compressed, amortized surrogate of the teacher belief for real-time control.
\textbf{Importantly, the student models the posterior over the parameter vector} $\boldsymbol{\Theta}\in\mathbb{R}^d$. We model the student posterior as a factorized Gaussian:
$q_\varphi(\boldsymbol{\Theta} \mid o_t, \boldsymbol{p}_t)
= \mathcal{N}\big(\boldsymbol{\Theta}; \boldsymbol{\mu}_{t}, \mathrm{diag}(\boldsymbol{\sigma}_{t}^2)\big),
[\boldsymbol{\mu}_{t}, \log \boldsymbol{\sigma}_{t}^2]
= f_\varphi(o_t, \boldsymbol{p}_t)$,
where $f_\varphi$ is a small MLP (two hidden layers with ReLU activations).
(For notational simplicity we write $(o_t,\boldsymbol{p}_t)$; in implementation these inputs may include a history embedding or recurrent state.)
Each dimension of $\boldsymbol{\Theta}$ thus has a learned mean and variance.

During training, after the PF update (including resampling/MH), we obtain a discrete teacher posterior
$\{\tilde{\boldsymbol{\Theta}}^{(i)}, \tilde{w}^{(i)}\}_{i=1}^N$.
We fit the student by minimizing the weighted negative log-likelihood over \emph{full} parameter vectors:
$\mathcal{L}_{\text{bel}}(\varphi)
= - \sum_{i=1}^{N}
\tilde{w}^{(i)} \log \mathcal{N}\!\big(
\tilde{\boldsymbol{\Theta}}^{(i)}; \boldsymbol{\mu}_{t},
\mathrm{diag}(\boldsymbol{\sigma}_{t}^2)
\big)$.

We use three stabilizers:
(i) clipping $\log \boldsymbol{\sigma}_{t}^2 \in [\log \sigma^2_{\min}, \log \sigma^2_{\max}]$ with $\sigma_{\min}\approx 10^{-3}$, $\sigma_{\max}\approx 10$;
(ii) $\varepsilon$-stabilized normalized weights $\tilde{w}^{(i)} = w^{(i)}/(\sum_j w^{(j)}+\varepsilon)$ and stopping gradients through $w^{(i)}$;
(iii) online standardization of inputs $(o_t,\boldsymbol{p}_t)$. At test time the teacher is \emph{dropped} and only the student is used: given $(o_t,\boldsymbol{p}_t)$ we compute $[\boldsymbol{\mu}_{t}, \log \boldsymbol{\sigma}_{t}^2] = f_\varphi(o_t,\boldsymbol{p}_t)$ in $\mathcal{O}(1)$ time.
We treat $q_\varphi(\boldsymbol{\Theta}\mid \cdot)$ as an approximate belief and extract task-aligned features from its moments (Sec.~\ref{sec:policy}).
The full posterior surrogate provides an uncertainty-aware estimate for downstream reporting.The full posterior surrogate provides an uncertainty-aware estimate for downstream reporting, while the policy consumes only a compact,  belief summary for control.

\subsection{KL-based Intrinsic Reward}
\label{sec:method:kl}

We define an intrinsic reward that measures how much a \emph{newly acquired} observation changes the \emph{teacher} belief over $\boldsymbol{\Theta}$.
Let $\mathbf{w}_t \in \Delta^{N-1}$ be the PF weight vector after observing $(o_t,\boldsymbol{p}_t)$.
We define the one-step information-gain associated with action $a_{t-1}$ as
\begin{equation}
r_t^{\mathrm{IG}}
=
D_{\mathrm{KL}}\!\big(\mathbf{w}_t \,\Vert\, \mathbf{w}_{t-1}\big)
=
\sum_{i=1}^N
w_t^{(i)} \log \frac{w_t^{(i)}}{w_{t-1}^{(i)} + \varepsilon},
\label{eq:kl}
\end{equation}
with a small $\varepsilon$ for numerical stability.
Under Sequential Importance Sampling (SIS) assumptions and \emph{before resampling}, $r_t^{\mathrm{IG}}$ is a single-sample Monte Carlo estimator of the Bayes posterior-prior KL, i.e., the conditional mutual information between $\boldsymbol{\Theta}$ and $o_t$ given $o_{1:t-1}$. We compute $r_t^{\mathrm{IG}}$ \emph{exclusively} from the teacher weights and clip rare large values at a high empirical percentile to off-policy updates.

\subsection{Belief Features and Policy Learning}
\label{sec:policy}
\begin{algorithm}[!t]
\caption{Distill-Belief Training (Inference vs.\ Execution)}
\label{alg:train}
\begin{algorithmic}[1]
\REQUIRE Particle budget $N$, student params $\varphi$, policy params $\theta$, stop threshold $\zeta$, horizon $H$
\STATE Initialize student posterior $q_\varphi$ and policy $\pi_\theta$
\FOR{each training episode}
  \STATE Initialize PF particles $\{\boldsymbol{\Theta}^{(i)}\}_{i=1}^N$
  \STATE Initialize weights $w_0^{(i)} \leftarrow \frac{1}{N}$ for all $i$
  \FOR{$t=1$ \TO $H$}
    \STATE Observe $(o_t,\boldsymbol{p}_t)$ \textbf{\textcolor{cyan}{/*Inference layer /}}
    \FOR{$i=1$ \TO $N$}
      \STATE $w_t^{(i)} \leftarrow w_{t-1}^{(i)} \cdot \ell(o_t \mid \boldsymbol{p}_t,\boldsymbol{\Theta}^{(i)})$
    \ENDFOR
    \STATE Normalize $\{w_t^{(i)}\}_{i=1}^N$
    \STATE $r_t^{IG} \leftarrow D_{\mathrm{KL}}(w_t \,\|\, w_{t-1})$ \ \ \textit{/* before resampling /}
    \STATE Optionally resample / MH-move
    \STATE Update $\varphi$ by minimizing $\mathcal{L}_{bel}$ on $\{(\boldsymbol{\Theta}^{(i)}, w_t^{(i)})\}_{i=1}^N$
    \STATE Compute $q_\varphi(\boldsymbol{\Theta}\mid o_t,\boldsymbol{p}_t)$ and features $f_{bel}(b_t)$
    \IF{$\mathrm{Spread}(b_t) < \zeta$} 
      \STATE \textbf{break} \textbf{\textcolor{cyan}{/*Execution layer*/}}
    \ENDIF
    \STATE Sample $a_t \sim \pi_\theta(\cdot\mid \boldsymbol{\psi}_t= \{o_t,\boldsymbol{p}_t,f_{bel}(b_t))\}$ and execute
  \ENDFOR
  \STATE Update $\theta$ with PPO using rewards $\{r_t^{IG}\}_{t=1}^{H}$
\ENDFOR
\end{algorithmic}
\end{algorithm}

The Bayes belief $b_t(\boldsymbol{\Theta}) = p(\boldsymbol{\Theta}\mid o_{1:t}, \boldsymbol{p}_{1:t})$ is a sufficient statistic for optimal control.
However, the belief over $\boldsymbol{\Theta}\in\mathbb{R}^d$ is high-dimensional and expensive to represent.
Our design therefore separates \emph{scientific correctness} from \emph{deployment-time computation}:
(i) we maintain and distill a belief over the \emph{full} physical parameter vector $\boldsymbol{\Theta}$ so that training is driven by posterior contraction over \emph{all} unknown field parameters;
(ii) we expose to the policy only a compact, task-relevant marginal (source location) to stabilize learning and keep the policy state  and efficient.
Crucially, intrinsic rewards are computed from the PF teacher belief, while the student belief is used only as a policy input and a stopping signal, preventing reward hacking by design. {We partition the parameter vector as \(\boldsymbol{\Theta} = [(x_s,y_s)^\top,\ \boldsymbol{\Theta}_R^\top]^\top), where (\boldsymbol{\Theta}_R\in\mathbb{R}^{d-2}\) denotes the \textit{subvector} of the remaining unknown physical parameters (e.g., \(q_s,u_s,\phi_s,\alpha,\lambda\)).} The location marginal belief is $
b_t^L(\boldsymbol{\ell}) \;\triangleq\; \int b_t(\boldsymbol{\ell},\boldsymbol{\Theta}_R)\,d\boldsymbol{\Theta}_R$. We summarize $b_t^L$ by {its mean and covariance}:
$\boldsymbol{\mu}_L(b_t)=\mathbb{E}_{b_t^L}[\boldsymbol{\ell}]\in\mathbb{R}^2,
\boldsymbol{\Sigma}_L(b_t)=\mathrm{Cov}_{b_t^L}[\boldsymbol{\ell}]\in\mathbb{R}^{2\times 2}$. {We restrict to \((x_s,y_s)\) not because other parameters are unobserved, but because localization is the task objective; the remaining parameters are inferred jointly and marginalized out in \(b_t^L\).} For the PF teacher, $\boldsymbol{\mu}_L,\boldsymbol{\Sigma}_L$ are computed from weighted particles in the full $\boldsymbol{\Theta}$-space and then restricted to the $(x_s,y_s)$ coordinates.
For the student Gaussian $q_\varphi(\boldsymbol{\Theta}\mid o_t,\boldsymbol{p}_t)=\mathcal{N}(\boldsymbol{\mu}_t,\mathrm{diag}(\boldsymbol{\sigma}_t^2))$, we have
$\boldsymbol{\Sigma}_L(b_t)=\mathrm{diag}(\sigma_{x_s,t}^2,\sigma_{y_s,t}^2)$. We define a scalar dispersion certificate over location:
$\mathrm{Spread}(b_t)
\;\triangleq\;
\sqrt{\mathrm{tr}(\boldsymbol{\Sigma}_L(b_t))}
=
\sqrt{\lambda_1+\lambda_2}$, where $\lambda_1,\lambda_2$ are the eigenvalues of $\boldsymbol{\Sigma}_L(b_t)$.
This choice has two useful properties.
First, it is rotation-invariant and consistent across teacher and student (even if the teacher covariance has off-diagonal terms).
Second, it admits a direct statistical interpretation:
$\mathrm{Spread}(b_t)^2
=
\mathrm{tr}\!\big(\boldsymbol{\Sigma}_L(b_t)\big)
=
\mathbb{E}[\|\boldsymbol{\ell}-\boldsymbol{\mu}_L(b_t)\|_2^2]$, i.e., $\mathrm{Spread}(b_t)$ is the root-mean-square (RMS) posterior uncertainty of the location.
Moreover, by applying Markov's inequality to $\|\boldsymbol{\ell}-\boldsymbol{\mu}_L\|_2^2$, we obtain a simple certificate:
$\mathbb{P}\!\left(\|\boldsymbol{\ell}-\boldsymbol{\mu}_L(b_t)\|_2 \ge \delta\right)
\le \frac{\mathrm{Spread}(b_t)^2}{\delta^2}$,
so enforcing $\mathrm{Spread}(b_t)\!<\!\zeta$ directly upper-bounds the probability of large localization error at a chosen scale $\delta$.
This makes Spread-based stopping a principled accuracy--budget control mechanism rather than a heuristic termination rule. We expose to the policy from $b_t^L$:
$ g(b_t)\in
\{\boldsymbol{\mu}_L(b_t);$ $
[\boldsymbol{\mu}_L(b_t),$ $~\mathrm{diag}(\boldsymbol{\Sigma}_L(b_t))^{1/2},~\mathrm{Spread}(b_t)]
\}$. Intuitively, $\boldsymbol{\mu}_L$ supports exploitation (move toward the current best estimate), while $\mathrm{diag}(\boldsymbol{\Sigma}_L)^{1/2}$ and $\mathrm{Spread}$ quantify remaining uncertainty to guide exploration and enable self-cessation.
The belief features consumed by the policy are $f_{\text{bel}}(b_t)=g(b_t)$ and the full policy input is $\boldsymbol{\psi}_t=[o_t, \boldsymbol{p}_t, f_{\text{bel}}(b_t)]$. Our control objective is to maximize the expected discounted cumulative information gain,
$J(\theta)=\mathbb{E}_{\pi_\theta}\left[\sum_{t=1}^{T}\gamma^{t-1} r^{\mathrm{IG}}_{t}\right]$,
where $r^{\mathrm{IG}}_{t}$ is computed exclusively from the PF teacher belief (Eq.~\eqref{eq:kl}).
Because $r^{\mathrm{IG}}_{t}$ is a dense but noisy one-step Monte Carlo estimator of posterior contraction, we adopt an on-policy actor--critic with generalized advantage estimation (GAE) to reduce variance and stabilize learning.

We learn a continuous-control policy $\pi_\theta(a_t\mid \boldsymbol{\psi}_t)$ using Proximal Policy Optimization (PPO).
We parameterize $\pi_\theta$ as a diagonal Gaussian,
$\pi_\theta(a_t\mid \boldsymbol{\psi}_t)
= \mathcal{N}\!\big(a_t;\,\boldsymbol{\mu}_\theta(\boldsymbol{\psi}_t),\,\mathrm{diag}(\boldsymbol{\sigma}_\theta(\boldsymbol{\psi}_t)^2)\big)$,
and learn a value function $V_\phi(\boldsymbol{\psi}_t)$.
Using on-policy rollouts, we compute GAE:
$\delta_t = r^{\mathrm{IG}}_{t+1} + \gamma V_\phi(\boldsymbol{\psi}_{t+1}) (1-\texttt{done}_{t+1}) - V_\phi(\boldsymbol{\psi}_t),
\hat{A}_t = \sum_{l=0}^{\infty} (\gamma\lambda)^l \delta_{t+l}$,
with $\lambda\in[0,1]$.
Let $\pi_{\theta_{\mathrm{old}}}$ denote the behavior policy and define the ratio
$r_t(\theta) = \frac{\pi_\theta(a_t\mid \boldsymbol{\psi}_t)}{\pi_{\theta_{\mathrm{old}}}(a_t\mid \boldsymbol{\psi}_t)}$.
PPO maximizes the clipped surrogate objective
$ \mathcal{L}_{\text{PPO}}(\theta)
=
\mathbb{E}_t\!\left[
\min\!(
r_t(\theta)\hat{A}_t,\;
\mathrm{clip}(r_t(\theta),1-\epsilon,1+\epsilon)\hat{A}_t
)
\right]$, and we optimize
$ \max_{\theta,\phi}\;
\mathcal{L}_{\text{PPO}}(\theta)
- c_v\,\mathbb{E}_t\!\left( V_\phi(\boldsymbol{\psi}_t)-\hat{R}_t \right)^2
+ c_{\text{ent}}\,\mathbb{E}_t\!\left[ \mathcal{H}\!\big(\pi_\theta(\cdot\mid \boldsymbol{\psi}_t)\big) \right]$, where $\hat{R}_t$ is the bootstrapped return and $\mathcal{H}(\cdot)$ is entropy.
During data collection, exploration is achieved by sampling $a_t\sim \pi_\theta(\cdot\mid \boldsymbol{\psi}_t)$; at evaluation time, we use the mean action.  \textcolor{cyan}{Training is summarized in Algorithm~\ref{alg:train}; deployment removes the PF teacher, yielding a constant-time loop (Algorithm~\ref{alg:deploy}).}

\subsection{Stopping Rule and Complexity}
\label{sec:method:stop}

We terminate an episode once the posterior over the source location is sufficiently concentrated or a maximum horizon is reached.
Formally, our stopping rule is \textcolor{cyan}{$\mathrm{Spread}(b_t) = \sqrt{\mathrm{tr}(\boldsymbol{\Sigma}(b_t))} < \zeta$}, for a user-chosen threshold $\zeta>0$, which ensures the expected localization error is \textcolor{red}{below a user-specified tolerance
(Appendix~\ref{app:stopping})}. Per time step during training, the PF teacher costs $\mathcal{O}(N)$, while the student forward pass and policy inference are $\mathcal{O}(1)$.
At test time, the teacher is completely absent: we run only the student posterior $q_\varphi(\boldsymbol{\Theta}\mid o_t,\boldsymbol{p}_t)$ and the policy $\pi_\theta$. Let $N$ be the particle budget and $d=\mathrm{dim}(\boldsymbol{\Theta})$.
During training, the PF teacher performs (i) likelihood evaluation and reweighting $O(N\cdot C_{\text{lik}})$, (ii) occasional resampling $O(N)$, and (iii) optional MH rejuvenation whose dominant cost is also $O(N\cdot C_{\text{lik}})$ plus computing particle moments (up to $O(Nd^2)$ if a full covariance is used). Computing the KL-based intrinsic reward from PF weights is $O(N)$. The student and policy forward passes are $O(1)$ w.r.t.\ $N$ (dependent on network size). 
At deployment, we discard the PF entirely and run only the student posterior and policy networks, yielding $O(1)$ per-step time and $O(1)$ memory w.r.t.\ $N$.

\section{Experiments}
\label{sec:exp}

{We evaluate on \textcolor{cyan}{\textbf{ISLCenv (App. ~\ref{app:ISLCenv}}, \cite{chen2021dual}), a suite of \textbf{physics-grounded closed-loop field-sensing environments}} for inverse source localization and characterization.
Each episode instantiates a latent scalar field $\phi(\cdot;\Theta)$ using the steady-state convection--diffusion forward model (Gaussian plume instantiation in Sec.~\ref{sec:Preliminaries} or \ref{app:ISLCenv}), and the agent observes a noisy intensity reading $o_t$ at its pose $\mathbf{p}_t$.} The agent controls its motion via actions $a_t$ and must actively collect informative measurements to infer the unknown field parameters $\boldsymbol{\Theta}$ and localize the source. \textcolor{cyan}{We consider multiple field types (Temp. \cite{hite2019localization}, Conc.\cite{stockie2011mathematics}, Mag.\cite{brandenburg2005astrophysical}, Elec.\cite{cheney1999electrical}, Gas \cite{chen2021dual}, En.\cite{arridge1999optical}, and Noise \cite{picaut2002numerical}) to cover diverse signal characteristics and dynamics.} At the beginning of each episode, we sample the source and environmental parameters from predefined distributions (Table~\ref{tab:training_parameters}), including the source location $(x_s,y_s)$, release strength $q_s$, wind-related parameters, decay, and diffusivity. The agent starts from a random initial position in a designated region and moves with a fixed step size. Each episode terminates when either (i) a maximum horizon $H$ is reached, or (ii) the uncertainty certificate falls below a threshold $\zeta$.

\textbf{Baselines and Metrics:}
\label{sec:baselines}
We compare our method against a diverse set of baselines covering RL, planning, and Bayesian-inference-driven strategies: \textcolor{cyan}{AGDC \cite{Shi2024AutonomousGD}, Infotaxis \cite{vergassola2007infotaxis}, DCEE \cite{chen2021dual}, Entrotaxis \cite{hutchinson2018entrotaxis}, PCDQN \cite{zhao2022deep}, GMM-PFRL \cite{Park2022SourceTE}, and GMM-IG \cite{lee2025enhanced}}. These baselines are grouped into: (i) RL-based approaches that learn a policy from interaction, possibly augmented with belief/inference modules; and (ii) planning-based approaches that explicitly optimize information-related objectives using a Bayesian belief update. All methods are evaluated under the same environment settings and episode budgets. We evaluate both task performance and inference quality using five metrics:
\textbf{(1) \textcolor{cyan}{Success Rate (SR)}}, the fraction of episodes that satisfy the stopping criterion (e.g., $\mathrm{Spread}_t<\zeta$) within the maximum horizon;
\textbf{(2) \textcolor{cyan}{Trajectory Efficiency (TE)}}, measured by the number of steps $T_k$ until termination (lower is better);
\textbf{(3) \textcolor{cyan}{Source Localization Error (SLE)}}, the Euclidean error between the posterior-mean location estimate and the ground-truth source location at termination;
and \textbf{(4) \textcolor{cyan}{Full-Parameter Estimation Error (FPE)}}, which measures estimation accuracy of the full physical parameter vector $\boldsymbol{\Theta}\in\mathbb{R}^d$ at termination and \textbf{(5) \textcolor{cyan}{Uncertainty Quality (UQ)}}, measured by the negative log-likelihood (NLL) of the ground-truth parameters under the predicted posterior, $\mathrm{NLL} = -\log q_{\varphi}(\boldsymbol{\Theta})$ (see Appendix~\ref{app:metrics}).

\textcolor{cyan}{\textbf{Research Questions (RQ):}}
Our experiments are organized to answer the following research questions: 
\begin{enumerate}
\item \textcolor{olive}{\textbf{(RQ1. in Sec. \ref{sec:Protocol})}
Does Distill-Belief outperform consistently strong RL- and planning-based baselines on standard single-source ISLC across diverse physical field modalities in practice, while jointly improving task performance and inference/uncertainty quality?}
\item \textcolor{orange}{\textbf{(RQ2. in Sec. \ref{sec:MultiSource})} How well does the method scale in realistic scenarios as the number of simultaneous sources increases (i.e., increasingly multi-modal posteriors)?}
\item \textcolor{purple}{\textbf{(RQ3. in Sec. \ref{sec:Obstacle})} Can Distill-Belief maintain high success and efficiency in obstacle-constrained (non-convex) environments where reachability limits informative sensing?} 
\item \textcolor{teal}{\textbf{(RQ4. in Sec. \ref{sec:Ablation})} What are the contributions of key design choices , and how do these choices affect robustness to shortcut/reward-hacking behaviors in realistic settings?}
\item \textcolor{brown}{\textbf{(RQ5. in Sec. \ref{sec:Deployment})} What is the deployment-time inference cost in terms of particle budget of the distilled student compared with PF-based inference?} 
\item \textcolor{violet}{\textbf{(RQ6). in Sec. \ref{sec:Sensitivity}} How sensitive are results to PF hyperparameters, and where is the practical performance-cost frontier as particle budgets vary in practice?}
\end{enumerate}

\subsection{\textcolor{olive}{Single-Source Cross-Field Results}}
\label{sec:Protocol}

We evaluate in-distribution (ID) performance on a held-out set of randomly generated scenarios. 
Unless stated otherwise, all methods share the same environment configurations, training budgets, and network architectures (when applicable), and we report mean and standard deviation over multiple random seeds. 
We use the metrics defined in Sec.~\ref{sec:exp} to jointly assess task performance and belief/uncertainty quality, and follow each baseline's recommended hyperparameters unless explicitly stated otherwise. 
Additional implementation details are provided in the appendix \ref{sec:main-results}.

\begin{table}[htbp]
\centering
\caption{Comparison of Baselines Under Different Scenarios}
\label{tab:fundamental_experiments}
\resizebox{\linewidth}{!}{%
\begin{tabular}{llllllll}
\toprule
\textbf{Method} & \textbf{Temp.} & \textbf{Conc.} & \textbf{Mag.} & \textbf{Elec.} & \textbf{Gas} & \textbf{En.} & \textbf{Noise} \\
\midrule
\multicolumn{8}{c}{\textbf{SR (Success Rate) $\uparrow$}} \\
\midrule
\textbf{Distill-Belief} & \textcolor{teal}{0.95±0.05} & \textcolor{teal}{0.94±0.05} & \textcolor{teal}{0.94±0.05} & \textcolor{teal}{0.82±0.04} & \textcolor{teal}{0.96±0.05} & \textcolor{teal}{0.63±0.03} & \textcolor{teal}{0.94±0.05} \\
GMM-IG     & 0.90±0.05 & 0.91±0.05 & 0.89±0.04 & 0.77±0.04 & 0.92±0.05 & 0.61±0.03 & 0.91±0.05 \\
GMM-PFRL  & 0.80±0.04 & 0.81±0.04 & 0.81±0.04 & 0.68±0.03 & 0.79±0.04 & 0.51±0.03 & 0.80±0.04 \\
PCDQN     & 0.87±0.04 & 0.88±0.04 & 0.86±0.04 & 0.74±0.04 & 0.86±0.04 & 0.57±0.03 & 0.86±0.04 \\
AGDC      & 0.90±0.05 & 0.91±0.05 & 0.89±0.04 & 0.73±0.04 & 0.89±0.04 & 0.57±0.03 & 0.89±0.04 \\
Infotaxis & 0.85±0.04 & 0.86±0.04 & 0.85±0.04 & 0.75±0.04 & 0.84±0.04 & 0.55±0.03 & 0.80±0.04 \\
Entrotaxis & 0.24±0.01 & 0.23±0.01 & 0.25±0.01 & 0.15±0.01 & 0.22±0.01 & 0.14±0.01 & 0.23±0.01 \\
DCEE     & 0.58±0.03 & 0.59±0.03 & 0.58±0.03 & 0.43±0.02 & 0.56±0.03 & 0.36±0.02 & 0.57±0.03 \\
\midrule
\multicolumn{8}{c}{\textbf{TE (Trajectory Efficiency) $\downarrow$}} \\
\midrule
\textbf{Distill-Belief} & \textcolor{teal}{20±1.0} & \textcolor{teal}{19±1.0} & \textcolor{teal}{18±0.9} & \textcolor{teal}{19±0.8} & \textcolor{teal}{17±0.9} & \textcolor{teal}{19±0.5} & \textcolor{teal}{19±1.0} \\
GMM -IG        & 23±1.2 & 22±1.1 & 22±1.1 & 23±0.9 & 20±1.0 & 22±0.6 & 21±1.1 \\
GMM-PFRL       & 25±1.3 & 24±1.2 & 25±1.2 & 20±1.0 & 22±1.1 & 24±0.7 & 23±1.2 \\
PCDQN                 & 25±1.3 & 24±1.2 & 24±1.2 & 19±1.0 & 21±1.1 & 23±0.7 & 22±1.1 \\
AGDC                  & 45±2.3 & 43±2.2 & 42±2.1 & 44±1.8 & 40±2.0 & 45±1.3 & 42±2.1 \\
Infotaxis             & 50±2.5 & 48±2.4 & 51±2.4 & 58±1.9 & 43±2.2 & 47±1.4 & 45±2.3 \\
Entrotaxis            & 62±3.1 & 60±3.0 & 59±3.0 & 61±2.5 & 56±2.8 & 55±1.8 & 58±2.9 \\
DCEE                  & 57±2.9 & 55±2.8 & 54±2.7 & 55±2.3 & 51±2.6 & 57±1.6 & 53±2.7 \\
\midrule
\multicolumn{8}{c}{\textbf{REV (Robustness and Error Variation) $\downarrow$}} \\
\midrule
\textbf{Distill-Belief} & \textcolor{teal}{0.15±0.08} & \textcolor{teal}{0.14±0.07} & \textcolor{teal}{0.14±0.07} & \textcolor{teal}{0.12±0.06} & \textcolor{teal}{0.14±0.07} & \textcolor{teal}{0.09±0.05} & \textcolor{teal}{0.13±0.07} \\
GMM-IG           & 0.1 ±0.05 & 0.1 ±0.05 & 0.1 ±0.05 & 0.09±0.05 & 0.1 ±0.05 & 0.07±0.04 & 0.1 ±0.05 \\
GMM-PFRL                  & 0.1 ±0.05 & 0.1 ±0.05 & 0.1 ±0.05 & 0.09±0.05 & 0.1 ±0.05 & 0.07±0.04 & 0.1 ±0.05 \\
PCDQN                 & 0.1 ±0.05 & 0.1 ±0.05 & 0.1 ±0.05 & 0.09±0.05 & 0.1 ±0.05 & 0.07±0.04 & 0.1 ±0.05 \\
AGDC                  & 1.6 ±0.08 & 1.5 ±0.08 & 1.5 ±0.08 & 1.2 ±0.06 & 1.3 ±0.07 & 0.8 ±0.04 & 1.4 ±0.07 \\
Infotaxis             & 1.5 ±0.08 & 1.4 ±0.07 & 1.4 ±0.07 & 1.2 ±0.06 & 1.3 ±0.07 & 0.8 ±0.04 & 1.3 ±0.07 \\
Entrotaxis            & 1.4 ±0.07 & 1.3 ±0.07 & 1.3 ±0.07 & 1.1 ±0.06 & 1.2 ±0.06 & 0.7 ±0.04 & 1.2 ±0.06 \\
DCEE                  & 1.4 ±0.07 & 1.3 ±0.07 & 1.3 ±0.07 & 1.1 ±0.06 & 1.2 ±0.06 & 0.7 ±0.04 & 1.2 ±0.06 \\
\midrule
\multicolumn{8}{c}{\textbf{LPS (Local Posterior Spread) $\downarrow$}} \\
\midrule
\textbf{Distill-Belief} & \textcolor{teal}{0.08±0.01} & \textcolor{teal}{0.05±0.01} & \textcolor{teal}{0.06±0.01} & \textcolor{teal}{0.05±0.01} & \textcolor{teal}{0.06±0.01} & \textcolor{teal}{0.05±0.01}& \textcolor{teal}{0.06±0.01} \\
GMM-IG              & 0.2 ±0.01 & 0.2 ±0.01 & 0.2 ±0.01 & 0.17±0.01 & 0.2 ±0.01 & 0.13±0.01 & 0.2 ±0.01 \\
GMM-PFRL                & 0.25±0.01 & 0.24±0.01 & 0.24±0.01 & 0.20±0.01 & 0.22±0.01 & 0.14±0.01 & 0.23±0.01 \\
PCDQN                 & 0.23±0.01 & 0.22±0.01 & 0.22±0.01 & 0.18±0.01 & 0.20±0.01 & 0.12±0.01 & 0.21±0.01 \\
AGDC                  & 0.25±0.01 & 0.24±0.01 & 0.23±0.01 & 0.19±0.01 & 0.21±0.01 & 0.13±0.01 & 0.22±0.01 \\
Infotaxis             & 0.6 ±0.03 & 0.6 ±0.03 & 0.6 ±0.03 & 0.51±0.02 & 0.6 ±0.03 & 0.39±0.02 & 0.6 ±0.03 \\
Entrotaxis            & 0.7 ±0.04 & 0.7 ±0.04 & 0.7 ±0.04 & 0.60±0.03 & 0.7 ±0.04 & 0.46±0.02 & 0.7 ±0.04 \\
DCEE                  & 0.6 ±0.03 & 0.6 ±0.03 & 0.6 ±0.03 & 0.51±0.02 & 0.6 ±0.03 & 0.39±0.02 & 0.6 ±0.03 \\
\bottomrule
\end{tabular}%
}
\end{table}

Table~\ref{tab:fundamental_experiments} summarizes the core single-source results across seven field types.
Overall, \textsc{Distill-Belief} achieves the strongest \textbf{task success and efficiency}: it consistently attains the highest (or near-highest) SR while requiring markedly fewer steps (TE) to reach the stopping criterion.
The gains are most pronounced on challenging modalities such as \emph{Elec.} and \emph{En.}, where planning-based baselines degrade substantially, while our method maintains a substantially higher success rate with shorter trajectories.

A key observation is that high SR is not obtained by ``early stopping'' alone.
Our method also yields substantially lower LPS across all field types, indicating that the policy is indeed driving \emph{belief contraction} rather than terminating prematurely.
In contrast, planning baselines (Infotaxis / Entrotaxis / DCEE) tend to suffer from either myopic information seeking or overly conservative exploration under noisy observations, leading to long trajectories and low SR. Among RL-based competitors, we observe a clear trade-off between exploration efficiency and inference quality.
While some RL baselines can achieve competitive REV in easier modalities, they typically require longer trajectories and/or exhibit inferior LPS, suggesting that they do not reduce posterior dispersion as reliably.
By explicitly coupling a Bayes-correct teacher reward with amortized belief features, \textsc{Distill-Belief} attains both high success and fast termination, demonstrating that belief-space objectives provide a stronger training signal than sparse or heuristic rewards.

\subsection{\textcolor{orange}{Multi-Source Localization}}
\label{sec:MultiSource}
Table~\ref{tab:multi_source_results_sr_steps} reports multi-source localization performance in the Temperature field with $2/3/4$ sources. While single-source localization is a standard benchmark, many real scenarios involve multiple emitters (e.g., multiple leaks or hotspots), leading to observation superposition and a multi-modal posterior. This setting is therefore substantially more challenging: the agent must both disambiguate multiple peaks in the belief space and allocate sensing trajectories to reduce uncertainty across sources. By increasing the number of sources, we explicitly test whether methods that perform well under near-unimodal posteriors can scale to multi-modal beliefs without collapsing into local sensing patterns. We summarize performance using ASLE, WCSE, and BCR to capture both average accuracy and worst-case reliability across sources.

\begin{table}[htbp]
\centering
\caption{Multi-source performance in the Temperature field.}
\label{tab:multi_source_results_sr_steps}
\resizebox{\linewidth}{!}{%
\begin{tabular}{lccc}
\toprule
\textbf{Method} 
& \textbf{2 Sources} (SR$\uparrow$ / TE$\downarrow$) 
& \textbf{3 Sources} (SR$\uparrow$ / TE$\downarrow$) 
& \textbf{4 Sources} (SR$\uparrow$ / TE$\downarrow$) \\
\midrule
\textbf{Distill-Belief} & \textcolor{teal}{0.77$\pm$0.03 / 28$\pm$3} & \textcolor{teal}{0.70$\pm$0.03 / 34$\pm$4} & \textcolor{teal}{0.61$\pm$0.04 / 40$\pm$5} \\
AGDC          & 0.71$\pm$0.03 / 34$\pm$4 & 0.63$\pm$0.04 / 42$\pm$5 & 0.53$\pm$0.05 / 50$\pm$6 \\
GMM-IG        & 0.66$\pm$0.04 / 39$\pm$5 & 0.58$\pm$0.05 / 48$\pm$6 & 0.47$\pm$0.05 / 58$\pm$7 \\
PCDQN         & 0.63$\pm$0.04 / 42$\pm$6 & 0.53$\pm$0.05 / 52$\pm$7 & 0.41$\pm$0.06 / 62$\pm$8 \\
GMM-PFRL      & 0.68$\pm$0.03 / 55$\pm$6 & 0.56$\pm$0.05 / 65$\pm$8 & 0.43$\pm$0.06 / 78$\pm$10 \\
Infotaxis     & 0.65$\pm$0.04 / 60$\pm$7 & 0.51$\pm$0.06 / 72$\pm$9 & 0.36$\pm$0.07 / 85$\pm$12 \\
DCEE          & 0.40$\pm$0.02 / 72$\pm$6 & 0.33$\pm$0.03 / 86 $\pm$7 & 0.26$\pm$0.04 / 100$\pm$9 \\
\bottomrule
\end{tabular}}
\end{table}

As shown in Table~\ref{tab:multi_source_results_sr_steps}, performance degrades as the number of sources increases, which is expected due to observation superposition and multi-modality.
Nevertheless, \textsc{Distill-Belief} remains the best-performing method in both SR and TE across $2/3/4$ sources.
This indicates that the learned policy does not collapse to a single local sensing pattern; instead, it continues to gather informative measurements that reduce global ambiguity.
Planning baselines exhibit a sharper drop in SR and a rapid increase in TE, consistent with their difficulty in handling multi-modal posteriors and long-horizon disambiguation under motion constraints.

\subsection{\textcolor{purple}{Obstacle-Constrained Environments}}
\label{sec:Obstacle}
Table~\ref{tab:obstacle_results_sr_steps} evaluates performance under obstacle-constrained environments with varying obstacle densities. Obstacles induce non-convex feasible regions and often prevent the agent from taking a direct path to the most informative areas, which can expose brittleness in greedy information-seeking baselines and in policies that do not properly account for reachability. This suite is included to assess practical deployability: the agent must maintain localization quality while producing efficient and feasible trajectories under sparse, moderate, and dense obstacle layouts. We report SR, TE, and LPS to jointly reflect localization performance, sensing efficiency, and path-level behavior in constrained navigation.

\begin{table}[htbp]
\centering
\caption{Performance under obstacle-rich environments.}
\label{tab:obstacle_results_sr_steps}
\resizebox{\columnwidth}{!}{%
\begin{tabular}{lccc}
\toprule
Method
& Sparse (SR$\uparrow$ / TE$\downarrow$)
& Moderate (SR$\uparrow$ / TE$\downarrow$)
& Dense (SR$\uparrow$ / TE$\downarrow$) \\
\midrule
\textbf{Distill-Belief (ours)}
& \textcolor{teal}{0.90$\pm$0.04 / 21$\pm$1}
& \textcolor{teal}{0.86$\pm$0.05 / 25$\pm$1}
& \textcolor{teal}{0.80$\pm$0.06 / 31$\pm$2} \\
GMM-IG
& 0.85$\pm$0.05 / 24$\pm$1
& 0.81$\pm$0.05 / 28$\pm$2
& 0.74$\pm$0.06 / 35$\pm$2 \\
AGDC
& 0.85$\pm$0.05 / 46$\pm$2
& 0.80$\pm$0.05 / 52$\pm$3
& 0.72$\pm$0.06 / 61$\pm$3 \\
PCDQN
& 0.82$\pm$0.05 / 26$\pm$1
& 0.77$\pm$0.06 / 31$\pm$2
& 0.69$\pm$0.06 / 38$\pm$2 \\
GMM-PFRL
& 0.75$\pm$0.05 / 26$\pm$1
& 0.70$\pm$0.06 / 31$\pm$2
& 0.62$\pm$0.07 / 38$\pm$2 \\
Infotaxis
& 0.80$\pm$0.06 / 52$\pm$3
& 0.74$\pm$0.06 / 60$\pm$3
& 0.65$\pm$0.07 / 71$\pm$4 \\
DCEE
& 0.55$\pm$0.07 / 59$\pm$3
& 0.48$\pm$0.08 / 68$\pm$4
& 0.38$\pm$0.09 / 80$\pm$4 \\
\bottomrule
\end{tabular}}
\end{table}

Table~\ref{tab:obstacle_results_sr_steps} shows that obstacles reduce SR and increase TE for all methods as density grows from sparse to dense, since non-convex reachability limits access to maximally informative regions.
Across all densities, \textsc{Distill-Belief} remains the strongest method, achieving the highest success rates (SR: $0.90/0.86/0.80$ under sparse/moderate/dense) while terminating with the fewest steps (TE: $21/25/31$), indicating belief-driven yet constraint-aware exploration.
Among RL baselines, \textsc{GMM-IG} is the closest competitor but still trails in success and efficiency (SR: $0.85/0.81/0.74$, TE: $24/28/35$), while \textsc{AGDC} suffers especially poor efficiency under obstacles (TE: $46/52/61$), suggesting difficulty coordinating informative sensing with feasible navigation.
Planning-based approaches degrade more sharply in dense layouts (\textsc{Infotaxis} TE $71$; \textsc{DCEE} SR $0.38$ with TE $80$), consistent with over-committing to locally informative but globally inefficient routes and failing the certificate within the horizon.

\subsection{\textcolor{teal}{Ablation Studies}}
\label{sec:Ablation}

\begin{figure}[t]
  \centering

  \begin{subfigure}[t]{0.42\linewidth}
    \centering
    \includegraphics[width=\linewidth]{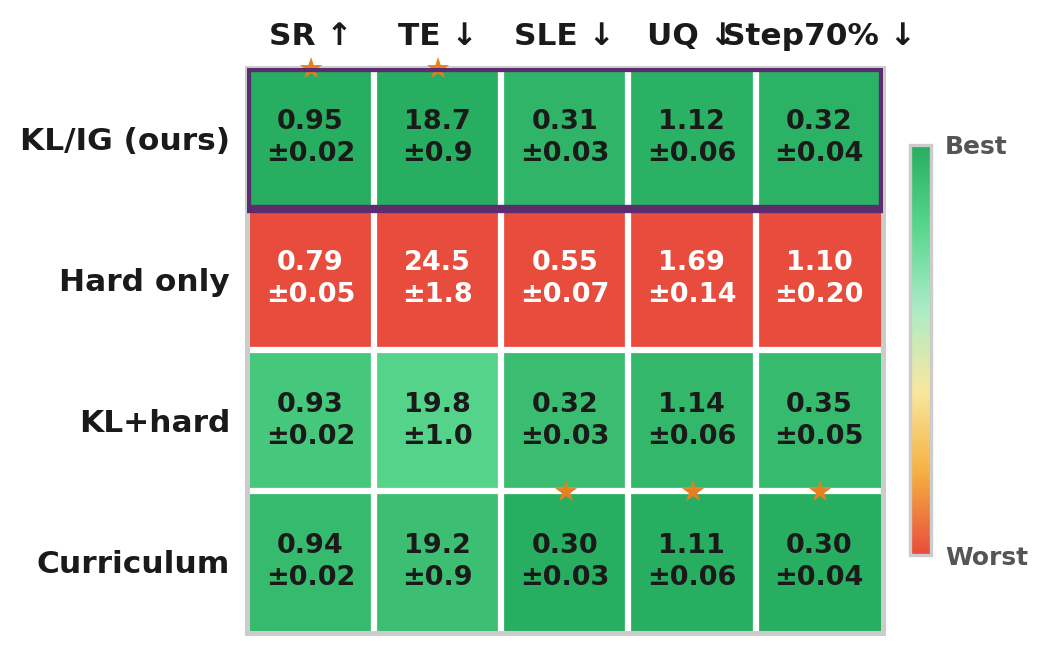}
    \caption{Ablation Reward.}
    \label{fig:ablation_Reward}
  \end{subfigure}
  \hfill
  \begin{subfigure}[t]{0.55\linewidth}
    \centering
    \includegraphics[width=0.78\linewidth]{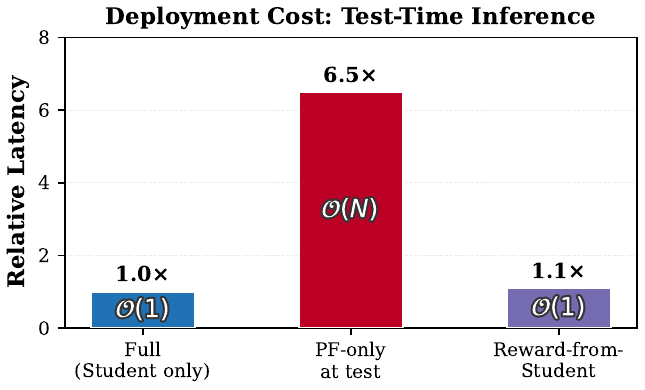}
    \caption{Test-time inference latency. }
    \label{fig:deployment_latency}
  \end{subfigure}

  \begin{subfigure}[t]{\linewidth}
    \centering
    \includegraphics[width=\linewidth]{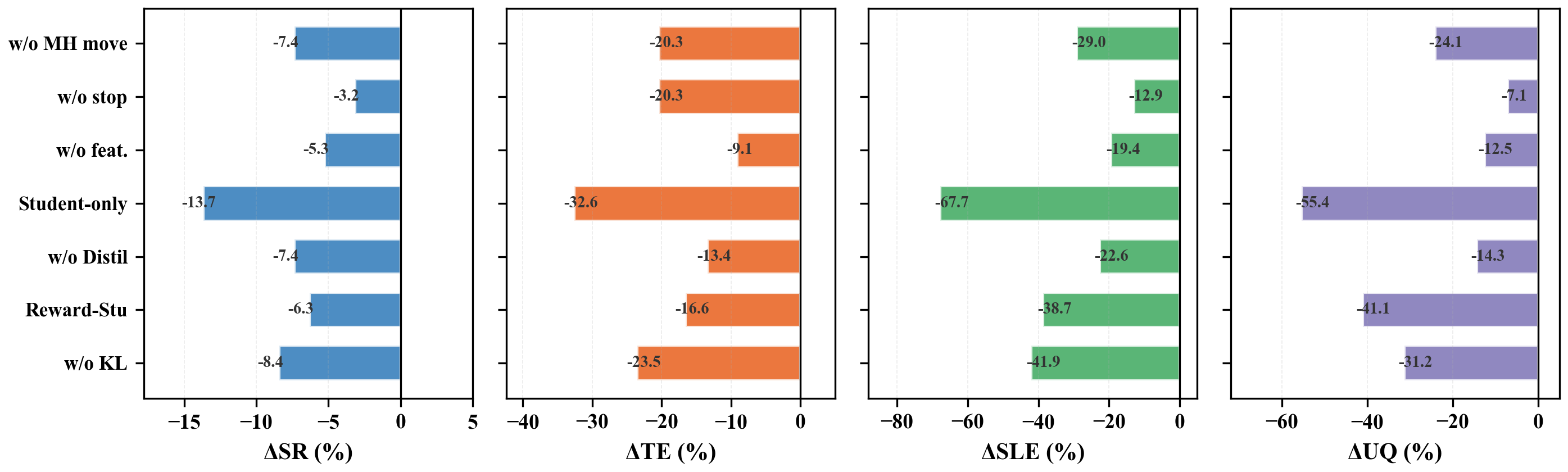}
    \caption{Ablation on Belief-optimization pipeline.}
    \label{fig:ablation_delta}
  \end{subfigure}

  \caption{Ablation study and deployment analysis.}
  \label{fig:ablation_combined}
\end{figure}
We conduct two complementary ablations to isolate the contributions of our belief-optimization pipeline and reward shaping.
Table~\ref{tab:ablation_main} ablates key components in the teacher-student belief pipeline.
Table~\ref{tab:ablation_reward} further isolates reward design by comparing dense KL-based information gain with sparse hard-success feedback.
Unless otherwise noted, we report SR, TE, SLE, FPE and UQ(NLL) for reward design we additionally report Steps@70\% SR to quantify sample efficiency. \textcolor{cyan}{\textbf{Belief-optimization pipeline.}}
Table~\ref{tab:ablation_main} and Figure~\ref{fig:ablation_delta} shows that the gains are not due to a single heuristic.
Removing the KL-based IG reward substantially reduces SR and increases TE, indicating that dense belief-space shaping is important for sample-efficient exploration.
Computing the KL objective from the student belief degrades both performance and UQ, consistent with shortcutting when the same approximation is used for both reward and policy input.
Distillation is critical for efficient deployment: PF-only testing remains competitive but forfeits amortized inference, while student-only training without PF supervision significantly harms SR/SLE/FPE/UQ, highlighting the need for Bayes-correct teacher guidance.
Finally, removing Spread features or Spread-based stopping mainly hurts efficiency and calibration, and disabling MH rejuvenation degrades stability and posterior quality, suggesting that PF diversity improves supervision. \textcolor{cyan}{\textbf{Reward design.}}
Table~\ref{tab:ablation_reward} and Figure~\ref{fig:ablation_Reward} confirms that sparse hard-success feedback is substantially less learnable, leading to lower SR, worse UQ, and markedly poorer sample efficiency.
In contrast, the dense KL-based information-gain reward provides shaped, belief-aligned feedback throughout an episode, accelerating learning and improving both trajectory efficiency and posterior quality.
The mixed and curriculum variants suggest that once reliable exploration emerges, task-success signals can be added without sacrificing the benefits of KL shaping.

\subsection{\textcolor{brown}{Deployment Cost and Amortized Inference}}
\label{sec:Deployment}

Table~\ref{tab:ablation_cost} focuses on test-time cost. PF-based belief updates scale linearly with the number of particles, which can become a bottleneck for real-time decision making or deployment across large numbers of scenarios. Our teacher--student design is motivated precisely to amortize Bayesian inference: the student predicts belief features in constant time while preserving the benefits of Bayes-correct training signals. This table explicitly disentangles \emph{performance} from \emph{deployability} by contrasting student-only inference against PF-only testing, and by showing how methods that rely on PF at test time incur $\mathcal{O}(N)$ per-step overhead. Table~\ref{tab:ablation_cost} and Figure \ref{fig:deployment_latency} highlights the practical motivation of our teacher--student design.
A PF update scales as $\mathcal{O}(N)$ per step and quickly becomes a deployment bottleneck, whereas the distilled student predicts belief features in $\mathcal{O}(1)$ time.
Importantly, the constant-time deployment does not come from weakening the training objective: Bayes correctness is enforced during training via the PF teacher, while the student inherits this behavior through distillation.

\subsection{\textcolor{violet}{Sensitivity to Budget and Thresholds}}
\label{sec:Sensitivity}
Figure \ref{fig:sensitivity_combined} and Table~\ref{fig:sens_line} evaluates sensitivity to PF hyperparameters. Because particle count and resampling/stopping thresholds are common sources of confounding, we include this study to demonstrate that our gains are not due to a narrowly tuned setting. We vary the particle budget $N$, the ESS resampling threshold $\tau_{\text{ESS}}$, and the stopping threshold $\tau_{\text{stop}}$, and report both performance metrics (SR, TE, SLE, UQ) and per-step latency. This analysis clarifies the practical trade-off between computational cost and uncertainty quality, and verifies that our default configuration lies in a stable operating regime rather than a brittle optimum.

\begin{figure}[htbp]
  \centering
  \begin{subfigure}[htbp]{\linewidth}
    \includegraphics[width=\linewidth]{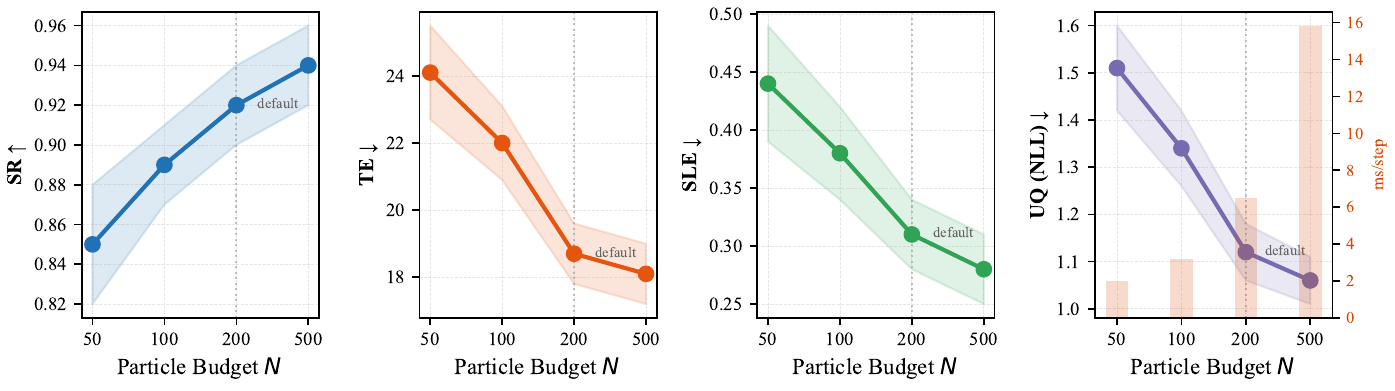}\caption{Sensitivity to particle budget $N$ ($\tau_{\mathrm{ESS}} = 0.5$).}
   \label{fig:sens_line}
  \end{subfigure}

  \begin{subfigure}[htbp]{\linewidth}    \includegraphics[width=\linewidth]{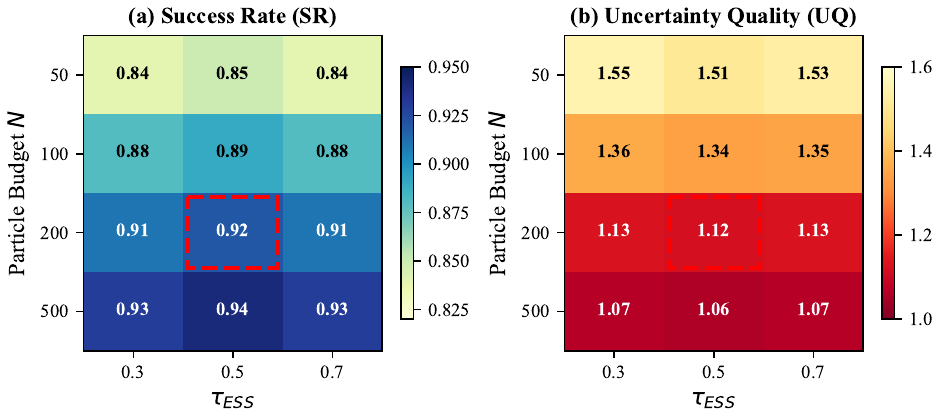}
    \caption{Sensitivity heatmap over the $N \times \tau_{\mathrm{ESS}}$ grid for SR and UQ.}
    \label{fig:sens_heatmap}
  \end{subfigure}

  \caption{Hyperparameter sensitivity analysis.}
  \label{fig:sensitivity_combined}
\end{figure}

Table~\ref{tab:sensitivity_particles_thresholds} and Figure \ref{fig:sens_heatmap} shows a clear performance-cost trade-off with the particle budget.
Increasing $N$ improves SR, reduces TE/SLE, and yields better UQ, but with diminishing returns beyond $N=200$ while training cost grows rapidly.
Varying $\tau_{\text{ESS}}$ has a comparatively smaller impact, indicating that our method operates in a stable regime rather than relying on a narrowly tuned resampling threshold.
These results justify our default configuration as a balanced operating point that provides strong belief supervision without excessive training-time overhead.


\section{Limitations and Ethical Considerations}
Experiments use physics-grounded simulators with stochastic sensing; real deployments may add complexities. Two limitations remain. 1) training depends on a Bayes-consistent particle-filter teacher for information-gain rewards and targets, which can be costly as the parameter space grows. 2) the compact test-time belief may be less effective in multi-source cases with multi-modal posteriors. This study uses no human-subject or personally identifiable data; real deployments should follow institutional policies and consent requirements. Finally, we emphasize that \textcolor{cyan}{we have already validated} the overall sensing-and-localization pipeline in \textcolor{cyan}{real-world/physical experiments} \cite{hutchinson2019experimental,hutchinson2019source} with  non-AI methods, and this paper extends that validated setting with a distillation-based policy.

\section{Conclusion}
We propose Distill-Belief, a teacher–student framework for closed-loop ISLC, where a particle-filter teacher supplies Bayes-consistent KL information-gain rewards during training and a compact student belief enables constant-cost control with uncertainty-based stopping at test time. Across seven physics-grounded modalities and stress tests, it improves success, sample efficiency, and uncertainty quality over strong baselines while mitigating reward hacking.

\bibliographystyle{ACM-Reference-Format}
\bibliography{sample-base}

\clearpage

\appendix
\section{Justification of the Spread-Based Stopping Certificate}
\label{app:stopping}

\textcolor{cyan}{\textbf{Purpose and connection to the main text.}}
In Sec.~\ref{sec:policy}-\ref{sec:method:stop}, we stop an episode once the \emph{location} posterior is sufficiently
concentrated, quantified by $\mathrm{Spread}(b_t)=\sqrt{\mathrm{tr}(\Sigma_L(b_t))}$.
This appendix explains why $\sqrt{\mathrm{tr}(\cdot)}$ is an interpretable and principled
uncertainty certificate: it (i) is a rotation-invariant scalar summary of posterior dispersion,
(ii) equals the Bayes mean-squared error (MSE) of the posterior-mean estimator under squared loss,
and (iii) yields simple expectation and tail bounds for Euclidean localization error.
For notational simplicity, we write the location posterior covariance as $\Sigma_t$ in this appendix;
it corresponds exactly to $\Sigma_L(b_t)$ in the main text.

\textcolor{cyan}{\textbf{Setup and notation.}}
Let the unknown source location be $\theta_L \in \mathbb{R}^2$ (e.g., $\theta_L = (x_s,y_s)$).
Let $D_t$ denote the information available up to time $t$ (e.g., observations and agent poses).
Let the (location) posterior be
\[
b_t^L(\theta_L) := p(\theta_L \mid D_t).
\]
Define its mean and covariance
\[
\mu_t := \mathbb{E}[\theta_L \mid D_t], \qquad
\Sigma_t := \mathrm{Cov}(\theta_L \mid D_t).
\]
We define the scalar spread certificate
\[
\mathrm{Spread}(b_t) := \sqrt{\mathrm{tr}(\Sigma_t)}.
\]

Stopping requires a \emph{one-dimensional} criterion that is easy to compute online and easy to
interpret. The trace summarizes the total marginal uncertainty across coordinates (sum of variances),
is invariant to rotations of the coordinate system, and $\sqrt{\mathrm{tr}(\Sigma_t)}$ has the same
units as distance, making it directly comparable to localization error thresholds.

We stop when $\mathrm{Spread}(b_t) < \zeta$ for a user-chosen tolerance $\zeta>0$.

\subsection{Why $\sqrt{\mathrm{tr}(\Sigma_t)}$ measures posterior concentration}

\textcolor{cyan}{\textbf{Intuition.}}
In 2D, $\mathrm{tr}(\Sigma_t)$ equals the sum of the two principal variances (eigenvalues)
$\lambda_1+\lambda_2$. Thus $\mathrm{Spread}(b_t)=\sqrt{\lambda_1+\lambda_2}$ can be viewed as the
\emph{root-sum-of-variances} across the two spatial coordinates: it is small only when the posterior
mass concentrates tightly around its mean in \emph{all} directions, not merely along one axis.
This makes it a conservative “how concentrated is the belief?” scalar.

\begin{proposition}[Trace--MSE identity]
\label{prop:trace_mse}
Assume $b_t^L$ has finite second moment.
Then
\[
\mathbb{E}\!\left[\|\theta_L-\mu_t\|_2^2 \mid D_t\right] = \mathrm{tr}(\Sigma_t).
\]
Equivalently,
\[
\mathrm{Spread}(b_t) = 
\sqrt{\mathbb{E}\!\left[\|\theta_L-\mu_t\|_2^2 \mid D_t\right]},
\]
i.e., $\mathrm{Spread}(b_t)$ is the posterior root-mean-square (RMS) Euclidean deviation from the posterior mean.
\end{proposition}

\textbf{Remark.}
Proposition~\ref{prop:trace_mse} is purely algebraic and does \emph{not} assume Gaussianity.
It shows that $\mathrm{Spread}(b_t)$ is an exact posterior RMS deviation from the mean---hence a direct,
model-agnostic measure of posterior concentration.

\begin{proof}
Using $\|x\|_2^2 = x^\top x$ and $\mathrm{tr}(a) = a$ for scalars,
\[
\mathbb{E}[\|\theta_L-\mu_t\|_2^2\mid D_t]
= \mathbb{E}[\mathrm{tr}((\theta_L-\mu_t)(\theta_L-\mu_t)^\top)\mid D_t].
\]
Swap trace and expectation (linearity):
\[
= \mathrm{tr}\!\left(\mathbb{E}[(\theta_L-\mu_t)(\theta_L-\mu_t)^\top\mid D_t]\right)
= \mathrm{tr}(\Sigma_t).
\]
\end{proof}

\subsection{Decision-theoretic justification: Spread is the Bayes MSE of the reported estimate}

A stopping rule should certify the quality of the \emph{reported estimate at termination}.
In our setting, the reported location is the posterior mean (consistent with the SLE metric in the main text).
Under squared Euclidean loss, the posterior mean is Bayes-optimal, and its Bayes risk is exactly the posterior variance.
Therefore, controlling $\mathrm{Spread}(b_t)$ directly controls the best achievable posterior-expected squared error
among all point estimators at that time step.

\begin{proposition}[Posterior mean is Bayes-optimal under squared Euclidean loss]
\label{prop:bayes_opt_mean}
For any point estimate $a\in\mathbb{R}^2$, define the conditional squared risk
\[
\mathcal{R}_t(a) := \mathbb{E}\!\left[\|\theta_L-a\|_2^2 \mid D_t\right].
\]
Then
\[
\mathcal{R}_t(a) = \mathrm{tr}(\Sigma_t) + \|\mu_t-a\|_2^2,
\]
so $\mathcal{R}_t(a)$ is minimized at $a=\mu_t$, and the minimum value is
\[
\min_{a\in\mathbb{R}^2}\mathcal{R}_t(a)=\mathcal{R}_t(\mu_t)=\mathrm{tr}(\Sigma_t)=\mathrm{Spread}(b_t)^2.
\]
\end{proposition}

\textbf{Remark.}
Proposition~\ref{prop:bayes_opt_mean} shows that $\mathrm{Spread}(b_t)^2$ is not an arbitrary statistic:
it is precisely the \emph{Bayes MSE} of the estimator we actually output at termination.

\begin{proof}
Decompose $\theta_L-a=(\theta_L-\mu_t)+(\mu_t-a)$ and expand:
\[
\|\theta_L-a\|_2^2 = \|\theta_L-\mu_t\|_2^2 + \|\mu_t-a\|_2^2
+2(\theta_L-\mu_t)^\top(\mu_t-a).
\]
Take $\mathbb{E}[\cdot\mid D_t]$.
The cross term vanishes because $\mathbb{E}[\theta_L-\mu_t\mid D_t]=0$.
Thus
\[
\mathcal{R}_t(a)=\mathbb{E}[\|\theta_L-\mu_t\|_2^2\mid D_t]+\|\mu_t-a\|_2^2
=\mathrm{tr}(\Sigma_t)+\|\mu_t-a\|_2^2,
\]
where the last equality uses Proposition~\ref{prop:trace_mse}.
\end{proof}

\begin{theorem}[Stopping rule as a Bayes-MSE certificate]
\label{thm:stopping_certificate}
Suppose the reported source location estimate at termination is the posterior mean $\hat\theta_t := \mu_t$.
Then the condition
\[
\mathrm{Spread}(b_t) < \zeta
\]
is equivalent to requiring that the conditional Bayes mean-squared localization error is below $\zeta^2$:
\[
\mathbb{E}\!\left[\|\theta_L-\hat\theta_t\|_2^2\mid D_t\right] < \zeta^2.
\]
\end{theorem}

\textcolor{cyan}{\textbf{Remark.}}
Theorem~\ref{thm:stopping_certificate} provides the key semantics used in the main text:
$\zeta$ is an \emph{accuracy tolerance} in RMS sense. This is exactly the kind of interpretable
certificate needed to realize an explicit accuracy--budget trade-off.

\begin{proof}
By Proposition~\ref{prop:bayes_opt_mean}, with $\hat\theta_t=\mu_t$,
\[
\mathbb{E}[\|\theta_L-\hat\theta_t\|_2^2\mid D_t]
=\mathrm{tr}(\Sigma_t)=\mathrm{Spread}(b_t)^2.
\]
Thus $\mathrm{Spread}(b_t)<\zeta \iff \mathrm{Spread}(b_t)^2<\zeta^2$.
\end{proof}

\subsection{From squared error to Euclidean error and high-probability bounds}

In experiments, localization quality is typically reported in Euclidean distance (e.g., SLE).
The certificate above controls a squared-error quantity. This subsection translates the squared-error
control into (i) an expected Euclidean error bound and (ii) a simple tail bound on large Euclidean errors.
These results justify using $\mathrm{Spread}(b_t)$ as a stopping trigger even when the evaluation metric is Euclidean.

\begin{corollary}[Expected Euclidean error bound]
\label{cor:expected_L2}
Let $\hat\theta_t=\mu_t$. Then
\[
\mathbb{E}\!\left[\|\theta_L-\hat\theta_t\|_2 \mid D_t\right] 
\le \mathrm{Spread}(b_t).
\]
Consequently, $\mathrm{Spread}(b_t)<\zeta$ implies
$\mathbb{E}[\|\theta_L-\hat\theta_t\|_2 \mid D_t] < \zeta$.
\end{corollary}

\textcolor{cyan}{\textbf{Takeaway.}}
Corollary~\ref{cor:expected_L2} explains why $\mathrm{Spread}(b_t)$ is often read as an “RMS distance scale”:
it upper-bounds the posterior-expected Euclidean localization error.

\begin{proof}
Apply Jensen to the concave function $\sqrt{\cdot}$ with $X=\|\theta_L-\mu_t\|_2^2$:
\[
\mathbb{E}[\|\theta_L-\mu_t\|_2\mid D_t]
=\mathbb{E}[\sqrt{X}\mid D_t]
\le \sqrt{\mathbb{E}[X\mid D_t]}
=\sqrt{\mathrm{tr}(\Sigma_t)}.
\]
\end{proof}

\begin{proposition}[Distribution-free tail bound (Markov)]
\label{prop:markov_tail}
For any $\varepsilon>0$,
\[
\mathbb{P}\!\left(\|\theta_L-\mu_t\|_2 \ge \varepsilon \mid D_t\right)
\le \frac{\mathrm{Spread}(b_t)^2}{\varepsilon^2}.
\]
\end{proposition}

\textcolor{cyan}{\textbf{Remark.}}
Proposition~\ref{prop:markov_tail} is conservative but assumption-light:
it gives a distribution-free bound on the probability of a large localization error event.
In particular, if $\mathrm{Spread}(b_t)\ll \varepsilon$, then the posterior mass outside radius $\varepsilon$
around the mean must be small. When additional distributional assumptions (e.g., near-Gaussian posteriors)
are appropriate, stronger credible-radius statements are available.

\begin{proof}
Let $Z=\|\theta_L-\mu_t\|_2^2\ge 0$. Then
$\mathbb{P}(\|\theta_L-\mu_t\|_2\ge \varepsilon\mid D_t)=\mathbb{P}(Z\ge \varepsilon^2\mid D_t)$.
By Markov,
\[
\mathbb{P}(Z\ge \varepsilon^2\mid D_t)\le \frac{\mathbb{E}[Z\mid D_t]}{\varepsilon^2}
= \frac{\mathrm{tr}(\Sigma_t)}{\varepsilon^2}
= \frac{\mathrm{Spread}(b_t)^2}{\varepsilon^2}.
\]
\end{proof}

\subsection{Why additional measurements help}

A sequential information-gathering policy is expected to reduce uncertainty as new data arrive.
However, due to noise and partial observability, a \emph{single realized} measurement can occasionally increase posterior variance.
This subsection clarifies the correct notion of monotonicity: while variance need not decrease pointwise for every measurement outcome,
it \emph{does} decrease \emph{in expectation} over the next observation. This supports the use of $\mathrm{Spread}^2$ as a progress measure
for termination and explains why early stopping will typically be reached under informative sensing.

\begin{proposition}[Law of total variance (conditioning reduces variance in expectation)]
\label{prop:total_variance}
Let $D_t = (D_{t-1}, \text{new data at }t)$.
Then
\[
\mathrm{Cov}(\theta_L\mid D_{t-1})
=
\mathbb{E}\!\left[\mathrm{Cov}(\theta_L\mid D_t)\mid D_{t-1}\right]
+
\mathrm{Cov}\!\left(\mathbb{E}[\theta_L\mid D_t]\mid D_{t-1}\right).
\]
The second term is positive semidefinite, hence
\[
\mathbb{E}\!\left[\mathrm{Cov}(\theta_L\mid D_t)\mid D_{t-1}\right]
\preceq \mathrm{Cov}(\theta_L\mid D_{t-1}).
\]
Taking trace yields
\[
\mathbb{E}\!\left[\mathrm{Spread}(b_t)^2\mid D_{t-1}\right]\le \mathrm{Spread}(b_{t-1})^2.
\]
\end{proposition}

\textcolor{cyan}{\textbf{Remark.}}
In expectation over the next measurement, posterior uncertainty contracts.
Equivalently, $\mathrm{Spread}^2$ forms a supermartingale-like progress statistic for sequential data acquisition.
This justifies using $\mathrm{Spread}$ as a termination signal: while it may fluctuate locally, it is biased toward decreasing
as informative measurements accumulate.

\subsection{Gaussian interpretation as a credible-radius bound}

The preceding results are distribution-free and thus conservative.
If the posterior is approximately Gaussian (e.g., after sufficient data or under a Gaussian-belief approximation),
one can map $\mathrm{Spread}(b_t)$ to an explicit $(1-\delta)$ credible radius. This can be useful when users want to pick $\zeta$
to match a target confidence level.

\begin{proposition}[2D Gaussian credible radius]
\label{prop:gaussian_credible_full}
Assume $\theta_L \mid D_t \sim \mathcal{N}(\mu_t,\Sigma_t)$ in $\mathbb{R}^2$ and
$\Sigma_t \succ 0$ (symmetric positive definite).
Let $\delta\in(0,1)$ and define $c_\delta := -2\ln\delta$.
Then
\[
\mathbb{P}\!\left(\|\theta_L-\mu_t\|_2 \le \sqrt{c_\delta}\,\sqrt{\mathrm{tr}(\Sigma_t)} \ \middle|\ D_t\right)\ge 1-\delta.
\]
\end{proposition}

\begin{proof}
\textbf{Step 1: Whiten the Gaussian and obtain a $\chi^2$ variable.}
Since $\Sigma_t \succ 0$, it admits a symmetric square root $\Sigma_t^{1/2}$ and inverse square root
$\Sigma_t^{-1/2}$ satisfying $\Sigma_t^{1/2}\Sigma_t^{1/2}=\Sigma_t$ and
$\Sigma_t^{-1/2}\Sigma_t^{-1/2}=\Sigma_t^{-1}$.
Define the whitened variable
\[
z := \Sigma_t^{-1/2}(\theta_L-\mu_t).
\]
Because $\theta_L \mid D_t$ is Gaussian and $z$ is an affine transform of $\theta_L$,
$z \mid D_t$ is also Gaussian. Its conditional mean and covariance are
\[
\mathbb{E}[z\mid D_t]
= \Sigma_t^{-1/2}\big(\mathbb{E}[\theta_L\mid D_t]-\mu_t\big)=0,
\]
and
\[
\mathrm{Cov}(z\mid D_t)
= \Sigma_t^{-1/2}\,\mathrm{Cov}(\theta_L\mid D_t)\,\Sigma_t^{-1/2}
= \Sigma_t^{-1/2}\Sigma_t\Sigma_t^{-1/2}
= I_2.
\]
Hence
\[
z \mid D_t \sim \mathcal{N}(0,I_2).
\]
In particular, the coordinates $(z_1,z_2)$ are i.i.d.\ $\mathcal{N}(0,1)$, and therefore
\[
\|z\|_2^2 = z_1^2+z_2^2 \sim \chi^2_2.
\]
Moreover,
\[
\|z\|_2^2
= z^\top z
= (\theta_L-\mu_t)^\top \Sigma_t^{-1/2}\Sigma_t^{-1/2}(\theta_L-\mu_t)
= (\theta_L-\mu_t)^\top \Sigma_t^{-1}(\theta_L-\mu_t).
\]
\textbf{Step 2: Choose $c_\delta$ so that $\mathbb{P}(\chi^2_2 \le c_\delta)=1-\delta$.}
For a $\chi^2_2$ random variable, the CDF has the closed form
\[
\mathbb{P}(\chi^2_2 \le x) = 1-e^{-x/2}, \qquad x\ge 0.
\]
Setting $x=c_\delta$ and choosing $c_\delta=-2\ln\delta$ gives
\[
\mathbb{P}(\chi^2_2 \le c_\delta) = 1-e^{-c_\delta/2}
= 1-e^{\ln\delta}
= 1-\delta.
\]
Thus,
\[
\mathbb{P}\!\left(\|z\|_2^2 \le c_\delta \ \middle|\ D_t\right)=1-\delta.
\]
\textbf{Step 3: Convert the $\chi^2$ ellipsoid to a Euclidean ball using $\lambda_{\max}(\Sigma_t)$.}
We relate $\|\theta_L-\mu_t\|_2$ to $\|z\|_2$ by writing $\theta_L-\mu_t=\Sigma_t^{1/2}z$:
\[
\|\theta_L-\mu_t\|_2^2
= \|\Sigma_t^{1/2}z\|_2^2
= z^\top (\Sigma_t^{1/2})^\top \Sigma_t^{1/2} z
= z^\top \Sigma_t z,
\]
where we used symmetry of $\Sigma_t^{1/2}$.
Since $\Sigma_t$ is symmetric PSD, for any vector $u$ we have the Rayleigh-quotient bound
\[
u^\top \Sigma_t u \le \lambda_{\max}(\Sigma_t)\,u^\top u
= \lambda_{\max}(\Sigma_t)\,\|u\|_2^2.
\]
Applying this with $u=z$ yields
\[
\|\theta_L-\mu_t\|_2^2
= z^\top \Sigma_t z
\le \lambda_{\max}(\Sigma_t)\,\|z\|_2^2.
\]
Therefore, on the event $\{\|z\|_2^2 \le c_\delta\}$ we obtain
\[
\|\theta_L-\mu_t\|_2^2 \le \lambda_{\max}(\Sigma_t)\,c_\delta
\quad\Longrightarrow\quad
\|\theta_L-\mu_t\|_2 \le \sqrt{\lambda_{\max}(\Sigma_t)}\,\sqrt{c_\delta}.
\]
\textbf{Step 4: Upper bound $\lambda_{\max}(\Sigma_t)$ by $\mathrm{tr}(\Sigma_t)$ to match the Spread certificate.}
Let the eigenvalues of $\Sigma_t$ be $\lambda_1,\lambda_2\ge 0$. Then
\[
\lambda_{\max}(\Sigma_t)=\max\{\lambda_1,\lambda_2\}
\le \lambda_1+\lambda_2
= \mathrm{tr}(\Sigma_t).
\]
Hence,
\[
\sqrt{\lambda_{\max}(\Sigma_t)} \le \sqrt{\mathrm{tr}(\Sigma_t)}.
\]
Combining with Step 3, on $\{\|z\|_2^2 \le c_\delta\}$ we have
\[
\|\theta_L-\mu_t\|_2 \le \sqrt{c_\delta}\,\sqrt{\mathrm{tr}(\Sigma_t)}.
\]
\textbf{Step 5: Take probabilities.}
Because the event $\{\|z\|_2^2 \le c_\delta\}$ implies the event
$\{\|\theta_L-\mu_t\|_2 \le \sqrt{c_\delta}\sqrt{\mathrm{tr}(\Sigma_t)}\}$, we get
\[
\mathbb{P}\!\left(\|\theta_L-\mu_t\|_2 \le \sqrt{c_\delta}\,\sqrt{\mathrm{tr}(\Sigma_t)} \ \middle|\ D_t\right)
\ge
\mathbb{P}\!\left(\|z\|_2^2 \le c_\delta \ \middle|\ D_t\right)
= 1-\delta,
\]
where the last equality is Step 2. This proves the claim.
\end{proof}

\textcolor{cyan}{\textbf{Remark.}}
This bound says that, under Gaussianity, $\mathrm{Spread}(b_t)$ controls a high-probability Euclidean error radius up to a factor
$\sqrt{c_\delta}$. The use of $\lambda_{\max}(\Sigma_t)\le \mathrm{tr}(\Sigma_t)$ makes the Euclidean-ball statement conservative
(it upper-bounds the covariance ellipsoid by a ball).

\subsection{Practical computation under PF and diagonal-Gaussian student}

The stopping rule must be computable both (i) from the PF teacher during training and (ii) from the distilled student at deployment.
This subsection provides the explicit estimators used in both cases and makes the complexity claim transparent:
PF computation is $\mathcal{O}(N)$ in the number of particles, whereas the student computation is constant-time.

\textbf{Weighted particles (PF teacher).}
Given location particles $\theta_L^{(i)}$ with normalized weights $w^{(i)}$:
\[
\hat\mu_t = \sum_{i=1}^N w^{(i)}\theta_L^{(i)},\qquad
\hat\Sigma_t = \sum_{i=1}^N w^{(i)}(\theta_L^{(i)}-\hat\mu_t)(\theta_L^{(i)}-\hat\mu_t)^\top,
\]
and
\[
\mathrm{Spread}(b_t)^2=\mathrm{tr}(\hat\Sigma_t)=\sum_{i=1}^N w^{(i)}\|\theta_L^{(i)}-\hat\mu_t\|_2^2.
\]

\textcolor{cyan}{\textbf{Diagonal-Gaussian student.}}
If $\Sigma_t=\mathrm{diag}(\sigma_{x,t}^2,\sigma_{y,t}^2)$, then
\[
\mathrm{Spread}(b_t) = \sqrt{\sigma_{x,t}^2+\sigma_{y,t}^2}.
\]

\textcolor{cyan}{\textbf{Note on calibration.}}
All guarantees above are \emph{posterior-conditional}: they certify error under the inferred belief.
In practice, this is meaningful when the belief is reasonably calibrated; in the main paper we therefore also report uncertainty-quality
metrics (e.g., NLL) to empirically validate calibration.

\section{Extension to 3D Space}
\label{app:3d}

\subsection{Motivation and overview}
The main paper instantiates closed-loop inverse source localization and characterization (ISLC) in a 2D workspace, where the agent pose is
$\mathbf{p}_t=(x_t,y_t)\in\mathbb{R}^2$ and the (single) source location is $(x_s,y_s)$.
Many practical deployments, however, are inherently three-dimensional:
e.g., gas leakage localization in buildings, underwater plume tracing,
and aerial sensing in atmospheric boundary layers.
This appendix extends the entire Distill-Belief pipeline---field model,
Bayes-consistent PF teacher, student posterior distillation, KL-based intrinsic reward,
and the spread-based stopping certificate---to \textcolor{cyan}{\textbf{3D physical space}}.
Importantly, the teacher--student belief-optimization design is dimension-agnostic:
only the forward model and the location-belief geometry change.

\subsection{3D POMDP and parameterization}
\textbf{\textcolor{cyan}{3D workspace.}}
Let the agent pose be $\mathbf{p}_t=(x_t,y_t,z_t)\in\mathbb{R}^3$.
At each step the agent chooses an action $\mathbf{a}_t\in\mathcal{A}$
(e.g., a velocity command or displacement in $\mathbb{R}^3$),
and receives a scalar observation $o_t\in\Omega$ (e.g., concentration, temperature).

\textbf{unknown parameters value.}
We extend the physical parameter vector from 2D to 3D by including
the source altitude and vertical transport:
\begin{equation}
\boldsymbol{\Theta}^{3D}
=
\big[
x_s, y_s, z_s,\;
q_s,\;
\mathbf{v},\;
\alpha,\;
\lambda
\big]^\top
\in\mathbb{R}^d,
\end{equation}
where $\mathbf{p}_s=(x_s,y_s,z_s)$ is the source location,
$q_s>0$ is source strength, $\mathbf{v}\in\mathbb{R}^3$ is the (constant) convection velocity,
$\alpha>0$ is the (isotropic) diffusivity, and $\lambda>0$ is an effective decay length-scale.
If one prefers a directional parameterization, we can write
$\mathbf{v}=u[\cos\phi\cos\vartheta,\;\sin\phi\cos\vartheta,\;\sin\vartheta]^\top$
with speed $u\ge 0$, azimuth $\phi\in[0,2\pi)$, and elevation $\vartheta\in[-\pi/2,\pi/2]$.
All results below hold for either representation.

\textbf{3D belief.}
Let $D_t$ denote the information up to time $t$ (e.g., $(o_{1:t},\mathbf{p}_{1:t})$).
The Bayesian belief is
\begin{equation}
b_t(\boldsymbol{\Theta}) := p(\boldsymbol{\Theta}\mid D_t).
\end{equation}
The control objective remains belief-space contraction via information gain, and the policy
conditions on observations, pose, and distilled belief features.

\subsection{3D field: convection-diffusion-reaction}
\subsubsection{3D steady-state PDE}
We adopt the same diffusion--transport--source abstraction as the main text, now in 3D.
A common steady-state advection--diffusion--reaction model in $\mathbb{R}^3$ is
\begin{equation}
-\alpha \nabla^2 \phi(\mathbf{p})
\;+\;
\mathbf{v}\cdot\nabla \phi(\mathbf{p})
\;+\;
\kappa\,\phi(\mathbf{p})
\;=\;
q_s\,\delta(\mathbf{p}-\mathbf{p}_s),
\label{eq:3d_cdr}
\end{equation}
where $\phi(\mathbf{p})$ is the scalar field,
$\kappa\ge 0$ is an effective linear ``reaction/decay'' coefficient,
and $\delta(\cdot)$ is the Dirac delta. This PDE is translation-invariant under constant coefficients,
so it is natural to derive the Green's function in coordinates
$\mathbf{r}=\mathbf{p}-\mathbf{p}_s$ and then shift back.

\textbf{Connection to the main paper's $\lambda$.}
In the main text, $\lambda$ is used as an exponential decay parameter.
In 3D, it is convenient to define the effective decay rate
\begin{equation}
m := \sqrt{\frac{\kappa}{\alpha} + \frac{\|\mathbf{v}\|_2^2}{4\alpha^2}} \quad (>0),
\label{eq:m_def}
\end{equation}
and then set
\begin{equation}
\lambda := \frac{1}{m}.
\label{eq:lambda_def}
\end{equation}
Equivalently, $\kappa=\alpha/\lambda^2 - \|\mathbf{v}\|_2^2/(4\alpha)$.
This mapping ensures that the closed-form solution can be written with a clean
$\exp(-\|\mathbf{r}\|/\lambda)$ factor, matching the ``decay length'' interpretation.

\subsubsection{Derivation of the 3D Green's function (detailed)}
For clarity, set $\mathbf{p}_s=\mathbf{0}$ (source at the origin) and solve
\begin{equation}
-\alpha \nabla^2 \phi(\mathbf{p})
+\mathbf{v}\cdot\nabla \phi(\mathbf{p})
+\kappa\,\phi(\mathbf{p})
=
q_s\,\delta(\mathbf{p}).
\label{eq:3d_origin}
\end{equation}

\textbf{Step 1: remove the first-order (advection) term.}
Let $\boldsymbol{\beta}:=\mathbf{v}/(2\alpha)$ and define the exponential transform
\begin{equation}
\phi(\mathbf{p}) = \exp(\boldsymbol{\beta}^\top \mathbf{p})\,u(\mathbf{p}).
\label{eq:exp_transform}
\end{equation}
Compute the gradient and Laplacian:
\begin{align}
\nabla \phi
&=
\exp(\boldsymbol{\beta}^\top \mathbf{p})\big(\nabla u + \boldsymbol{\beta} u\big),\\
\nabla^2 \phi
&=
\exp(\boldsymbol{\beta}^\top \mathbf{p})
\big(\nabla^2 u + 2\boldsymbol{\beta}^\top \nabla u + \|\boldsymbol{\beta}\|_2^2 u\big).
\end{align}
Substitute into~\eqref{eq:3d_origin}:
\begin{align}
&-\alpha \exp(\boldsymbol{\beta}^\top \mathbf{p})
\big(\nabla^2 u + 2\boldsymbol{\beta}^\top \nabla u + \|\boldsymbol{\beta}\|_2^2 u\big)
+\mathbf{v}^\top \exp(\boldsymbol{\beta}^\top \mathbf{p})
\big(\nabla u + \boldsymbol{\beta}u\big)
\\&+\kappa \exp(\boldsymbol{\beta}^\top \mathbf{p})u \nonumber=
q_s\,\delta(\mathbf{p}).
\end{align}
Factor out $\exp(\boldsymbol{\beta}^\top \mathbf{p})$ and use $\mathbf{v}=2\alpha\boldsymbol{\beta}$:
\begin{align}
\exp(\boldsymbol{\beta}^\top \mathbf{p})
\Big(\text{term}_1 - \text{term}_2\Big)
&=q_s\,\delta(\mathbf{p}) \nonumber\\
\exp(\boldsymbol{\beta}^\top \mathbf{p})
\Big(-\alpha \nabla^2 u +(\kappa + \alpha \|\boldsymbol{\beta}\|_2^2)u \Big)
&= q_s\,\delta(\mathbf{p}).
\label{eq:u_pde_pre}
\end{align}
where $\text{term}_1 = 2\alpha \boldsymbol{\beta}^\top \nabla u
+2\alpha \|\boldsymbol{\beta}\|_2^2 u
+\kappa u$, $\text{term}_2 = -\alpha \nabla^2 u -2\alpha \boldsymbol{\beta}^\top \nabla u
-\alpha \|\boldsymbol{\beta}\|_2^2 u $
Because $\delta(\mathbf{p})$ concentrates at $\mathbf{p}=0$ and $\exp(\boldsymbol{\beta}^\top\mathbf{0})=1$,
we obtain the transformed PDE:
\begin{equation}
-\alpha \nabla^2 u(\mathbf{p})
+
(\kappa + \alpha \|\boldsymbol{\beta}\|_2^2)\,u(\mathbf{p})
=
q_s\,\delta(\mathbf{p}).
\label{eq:modified_helmholtz}
\end{equation}
Define
\begin{equation}
m^2 := \frac{\kappa}{\alpha} + \|\boldsymbol{\beta}\|_2^2
=
\frac{\kappa}{\alpha} + \frac{\|\mathbf{v}\|_2^2}{4\alpha^2}.
\end{equation}
Then \eqref{eq:modified_helmholtz} becomes the 3D modified Helmholtz equation
\begin{equation}
-\nabla^2 u(\mathbf{p}) + m^2 u(\mathbf{p}) = \frac{q_s}{\alpha}\,\delta(\mathbf{p}).
\label{eq:helmholtz_std}
\end{equation}

\textbf{Step 2: solve the modified Helmholtz equation in 3D.}
The fundamental solution of \eqref{eq:helmholtz_std} in $\mathbb{R}^3$ is well-known
and can be verified by standard distributional calculus:
\begin{equation}
u(\mathbf{p})
=
\frac{q_s}{4\pi \alpha}\,\frac{\exp(-m\|\mathbf{p}\|_2)}{\|\mathbf{p}\|_2}.
\label{eq:u_solution}
\end{equation}
One can directly check that for $\mathbf{p}\neq 0$ it satisfies the homogeneous PDE
$-\nabla^2 u + m^2 u = 0$, and that the singularity at the origin produces the correct delta mass.

\textcolor{cyan}{\textbf{Step 3: transform back and shift the source.}}
Combining \eqref{eq:exp_transform} and \eqref{eq:u_solution} gives, for a source at the origin,
\begin{equation}
\phi(\mathbf{p})
=
\frac{q_s}{4\pi \alpha}\,\frac{1}{\|\mathbf{p}\|_2}
\exp\!\Big(
\boldsymbol{\beta}^\top \mathbf{p} - m\|\mathbf{p}\|_2
\Big).
\label{eq:phi_origin}
\end{equation}
For a source at $\mathbf{p}_s$, replace $\mathbf{p}$ by $\mathbf{r}=\mathbf{p}-\mathbf{p}_s$:
\begin{equation}
\phi(\mathbf{p};\boldsymbol{\Theta}^{3D})
=
\frac{q_s}{4\pi \alpha}\,\frac{1}{\|\mathbf{p}-\mathbf{p}_s\|_2}
\exp\!\Big(
\boldsymbol{\beta}^\top (\mathbf{p}-\mathbf{p}_s)
-
m\|\mathbf{p}-\mathbf{p}_s\|_2
\Big),
\quad
\boldsymbol{\beta}=\frac{\mathbf{v}}{2\alpha}.
\label{eq:phi_general_3d}
\end{equation}
Using $\lambda=1/m$ yields the ``decay length'' form:
\begin{equation}
\phi(\mathbf{p};\boldsymbol{\Theta}^{3D})
=
\frac{q_s}{4\pi \alpha \|\mathbf{p}-\mathbf{p}_s\|_2}
\exp\!\Big(
\frac{\mathbf{v}^\top(\mathbf{p}-\mathbf{p}_s)}{2\alpha}
-
\frac{\|\mathbf{p}-\mathbf{p}_s\|_2}{\lambda}
\Big).
\label{eq:phi_lambda_form}
\end{equation}

\textcolor{cyan}{\textbf{Remark (sign convention to match a wind ``blowing away'').}}
If one prefers the exponent to be $-\mathbf{v}^\top(\mathbf{p}-\mathbf{p}_s)/(2\alpha)$
(as in the main paper's 2D plume expression), simply redefine the convection parameter
as $\tilde{\mathbf{v}}:=-\mathbf{v}$, i.e., store and infer the wind vector with the opposite sign.
The rest of the framework is unchanged.

\subsubsection{ground plane boundary via method of images (Neumann reflection)}
In atmospheric or indoor settings, the field is often restricted to a half-space $z\ge 0$ with
no-flux boundary at the ground (Neumann condition) $\partial \phi/\partial z=0$ at $z=0$.
A simple analytic approximation is obtained by adding an ``image source'' at
$\mathbf{p}_s'=(x_s,y_s,-z_s)$:
\begin{equation}
\phi_{\text{half}}(\mathbf{p})
=
\phi(\mathbf{p};\mathbf{p}_s)
+
\phi(\mathbf{p};\mathbf{p}_s'),
\label{eq:images}
\end{equation}
where $\phi(\cdot;\mathbf{p}_s)$ is given by \eqref{eq:phi_general_3d}.
This preserves the closed form and can be dropped-in as the forward model used by the PF likelihood.

\subsection{3D observation model and likelihood}
We keep the same sensor abstraction: the observation is a noisy measurement of the field.
A standard choice is additive Gaussian noise:
\begin{equation}
o_t = \phi(\mathbf{p}_t;\boldsymbol{\Theta}^{3D}) + \epsilon_t,
\quad
\epsilon_t\sim\mathcal{N}(0,\sigma_{\text{obs}}^2).
\end{equation}
The likelihood needed by the PF teacher is then
\begin{equation}
\ell(o_t \mid \mathbf{p}_t,\boldsymbol{\Theta}^{3D})
=
\mathcal{N}\big(o_t;\,\phi(\mathbf{p}_t;\boldsymbol{\Theta}^{3D}),\,\sigma_{\text{obs}}^2\big).
\label{eq:likelihood_3d}
\end{equation}
All other noise models (e.g., log-normal for strictly positive sensors, censoring, saturation)
can be incorporated similarly; the only requirement for the teacher PF is the ability to evaluate
$\ell(o_t\mid \mathbf{p}_t,\Theta)$ up to a constant factor.

\subsection{Teacher PF in 3D }
\textcolor{cyan}{\textbf{Particle representation.}}
The PF teacher maintains a weighted particle approximation of the full posterior over
$\boldsymbol{\Theta}^{3D}\in\mathbb{R}^d$:
\begin{equation}
b_t(\boldsymbol{\Theta})
\approx
\sum_{i=1}^N w_t^{(i)}\,\delta_{\boldsymbol{\Theta}_t^{(i)}}(\boldsymbol{\Theta}).
\end{equation}

\textcolor{cyan}{\textbf{Weight update.}}
Under the bootstrap (static-parameter) SIS choice, the reweighting step is identical:
\begin{equation}
\tilde w_t^{(i)} \propto w_{t-1}^{(i)}\,\ell(o_t\mid \mathbf{p}_t,\boldsymbol{\Theta}_{t-1}^{(i)}),
\qquad
w_t^{(i)} = \frac{\tilde w_t^{(i)}}{\sum_j \tilde w_t^{(j)}}.
\label{eq:pf_weights_3d}
\end{equation}
Resampling (via ESS) and MH rejuvenation apply without change; only the dimensionality of
$\boldsymbol{\Theta}$ increases (e.g., the location subvector is now 3D).

\subsection{KL-based intrinsic reward remains valid in 3D}
The intrinsic reward in the main text is the one-step KL divergence between consecutive PF beliefs.
This construction does not depend on the spatial dimension; it only depends on the Bayes update
over the parameter vector $\boldsymbol{\Theta}$.

{Population identity (dimension-free).}
Let $D_{t-1}$ denote past data and let $o_t$ be the next observation.
Then the conditional mutual information satisfies the identity
\begin{equation}
I(\boldsymbol{\Theta};o_t\mid D_{t-1})
=
\mathbb{E}_{o_t\sim p(\cdot\mid D_{t-1})}
\big[
D_{\mathrm{KL}}(p(\boldsymbol{\Theta}\mid D_t)\,\|\,p(\boldsymbol{\Theta}\mid D_{t-1}))
\big],
\label{eq:mi_ig_identity}
\end{equation}
which is purely probabilistic and holds for any $\boldsymbol{\Theta}$ (including the 3D case).
Therefore, using a Monte Carlo approximation of the posterior--prior KL remains a principled
information-gain proxy in 3D.

\textcolor{cyan}{\textbf{PF approximation via weight KL.}}
Before resampling, the PF posterior is represented by the same particle support with updated weights.
Thus a natural discrete approximation is
\begin{equation}
r_t^{IG}
=
D_{\mathrm{KL}}(\mathbf{w}_t\|\mathbf{w}_{t-1})
=
\sum_{i=1}^N w_t^{(i)}\log\frac{w_t^{(i)}}{w_{t-1}^{(i)}+\varepsilon},
\label{eq:weight_kl_3d}
\end{equation}
where $\mathbf{w}_t\in\Delta^{N-1}$ is the normalized weight vector and $\varepsilon>0$ stabilizes rare zeros.
This is the same quantity as in the main paper, and its interpretation as a one-step information-gain
estimator is unaffected by moving from 2D to 3D.

\subsection{Student posterior distillation in 3D}
The student remains a constant-time amortized posterior over the parameter vector:
\begin{equation}
q_\varphi(\boldsymbol{\Theta}^{3D}\mid o_t,\mathbf{p}_t)
=
\mathcal{N}\!\Big(\boldsymbol{\Theta}^{3D};\,\boldsymbol{\mu}_t,\mathrm{diag}(\boldsymbol{\sigma}_t^2)\Big),
\quad
[\boldsymbol{\mu}_t,\log \boldsymbol{\sigma}_t^2] = f_\varphi(o_t,\mathbf{p}_t).
\end{equation}
The distillation loss remains the weighted negative log-likelihood (NLL) on teacher particles:
\begin{equation}
\mathcal{L}_{bel}(\varphi)
=
-\sum_{i=1}^N \tilde w_t^{(i)}
\log
\mathcal{N}\!\Big(\tilde{\boldsymbol{\Theta}}_t^{(i)};\,\boldsymbol{\mu}_t,\mathrm{diag}(\boldsymbol{\sigma}_t^2)\Big).
\label{eq:distill_3d}
\end{equation}
No algorithmic changes are required; only the dimensionality of $\boldsymbol{\Theta}$ changes.

\subsection{3D belief features and 3D spread-based stopping certificate}
The policy typically needs only the source-location marginal (and its uncertainty) for control.
Let the 3D location parameter be $\boldsymbol{\theta}_L=(x_s,y_s,z_s)\in\mathbb{R}^3$.
Let the location posterior be $b_{L,t}(\boldsymbol{\theta}_L):=p(\boldsymbol{\theta}_L\mid D_t)$.
Define its mean and covariance:
\begin{equation}
\boldsymbol{\mu}_t := \mathbb{E}[\boldsymbol{\theta}_L\mid D_t]\in\mathbb{R}^3,
\qquad
\boldsymbol{\Sigma}_t := \mathrm{Cov}(\boldsymbol{\theta}_L\mid D_t)\in\mathbb{R}^{3\times 3}.
\end{equation}

\textcolor{cyan}{\textbf{3D Spread.}}
We generalize the spread certificate to 3D as
\begin{equation}
\mathrm{Spread}(b_t)
:=
\sqrt{\mathrm{tr}(\boldsymbol{\Sigma}_t)}.
\label{eq:spread_3d}
\end{equation}
This is rotation-invariant, has units of distance, and summarizes total marginal uncertainty across
all three spatial coordinates.

\textcolor{cyan}{\textbf{Stopping rule.}}
The 3D stopping rule mirrors the main text:
\begin{equation}
\mathrm{Spread}(b_t) < \zeta,
\label{eq:stop_3d}
\end{equation}
for a user-chosen tolerance $\zeta>0$.

\subsubsection{Theory: why $\sqrt{\mathrm{tr}(\Sigma_t)}$ is a principled 3D certificate}
The arguments in Appendix~A of the main paper extend \emph{verbatim} to 3D (indeed, to any dimension).
We restate the key results with $\boldsymbol{\theta}_L\in\mathbb{R}^3$ for completeness.

\begin{proposition}[Trace--MSE identity in $\mathbb{R}^3$]
\label{pb1}
Assume $b_{L,t}$ has finite second moment. Then
\begin{equation}
\mathbb{E}\!\left[\|\boldsymbol{\theta}_L-\boldsymbol{\mu}_t\|_2^2\mid D_t\right]
= \mathrm{tr}(\boldsymbol{\Sigma}_t),
\end{equation}
and therefore
\begin{equation}
\mathrm{Spread}(b_t)
=
\sqrt{
\mathbb{E}\!\left[\|\boldsymbol{\theta}_L-\boldsymbol{\mu}_t\|_2^2\mid D_t\right]
}.
\end{equation}
\end{proposition}

\emph{Proof.}
Use $\|\mathbf{x}\|_2^2=\mathrm{tr}(\mathbf{x}\mathbf{x}^\top)$ and linearity of trace/expectation:
\begin{align}
\mathbb{E}\!\left[\|\boldsymbol{\theta}_L-\boldsymbol{\mu}_t\|_2^2\mid D_t\right]
&=
\mathbb{E}\!\left[\mathrm{tr}\big((\boldsymbol{\theta}_L-\boldsymbol{\mu}_t)(\boldsymbol{\theta}_L-\boldsymbol{\mu}_t)^\top\big)\mid D_t\right]\\
&=
\mathrm{tr}\!\left(
\mathbb{E}\!\left[(\boldsymbol{\theta}_L-\boldsymbol{\mu}_t)(\boldsymbol{\theta}_L-\boldsymbol{\mu}_t)^\top\mid D_t\right]
\right)
=
\mathrm{tr}(\boldsymbol{\Sigma}_t).
\end{align}
$\square$

\begin{proposition}[Posterior mean is Bayes-optimal under squared loss in $\mathbb{R}^3$)]
\label{pb2}
For any estimate $\mathbf{a}\in\mathbb{R}^3$, define the conditional risk
$R_t(\mathbf{a})=\mathbb{E}[\|\boldsymbol{\theta}_L-\mathbf{a}\|_2^2\mid D_t]$.
Then
\begin{equation}
R_t(\mathbf{a}) = \mathrm{tr}(\boldsymbol{\Sigma}_t) + \|\boldsymbol{\mu}_t-\mathbf{a}\|_2^2,
\end{equation}
so the minimizer is $\mathbf{a}^\star=\boldsymbol{\mu}_t$ and
$\min_{\mathbf{a}}R_t(\mathbf{a})=\mathrm{tr}(\boldsymbol{\Sigma}_t)=\mathrm{Spread}(b_t)^2$.
\end{proposition}
\emph{Proof.}
Decompose $\boldsymbol{\theta}_L-\mathbf{a}=(\boldsymbol{\theta}_L-\boldsymbol{\mu}_t)+(\boldsymbol{\mu}_t-\mathbf{a})$,
expand the square, and note the cross term vanishes since
$\mathbb{E}[\boldsymbol{\theta}_L-\boldsymbol{\mu}_t\mid D_t]=0$.
Then apply Proposition~\ref{pb1}. $\square$

\begin{theorem}[Stopping is a Bayes-MSE certificate in 3D]
Let the reported location estimate be the posterior mean $\hat{\boldsymbol{\theta}}_t:=\boldsymbol{\mu}_t$.
Then the stopping condition $\mathrm{Spread}(b_t)<\zeta$ is equivalent to
\begin{equation}
\mathbb{E}\!\left[\|\boldsymbol{\theta}_L-\hat{\boldsymbol{\theta}}_t\|_2^2\mid D_t\right]
<
\zeta^2.
\end{equation}
\end{theorem}
\emph{Proof.}
By Proposition~\ref{pb2}, the conditional MSE of the posterior mean equals $\mathrm{tr}(\boldsymbol{\Sigma}_t)=\mathrm{Spread}(b_t)^2$.
$\square$

\begin{corollary}[Expected Euclidean error bound.]
\begin{equation}
\mathbb{E}\!\left[\|\boldsymbol{\theta}_L-\hat{\boldsymbol{\theta}}_t\|_2\mid D_t\right]
\le
\mathrm{Spread}(b_t).
\end{equation}
\end{corollary}

\emph{Proof.}
Apply Jensen to $\sqrt{\cdot}$ with $X=\|\boldsymbol{\theta}_L-\boldsymbol{\mu}_t\|_2^2$:
$\mathbb{E}[\|\cdot\|_2]\le \sqrt{\mathbb{E}[\|\cdot\|_2^2]}=\sqrt{\mathrm{tr}(\Sigma_t)}.$
$\square$

\begin{proposition}[Distribution-free tail bound.]
For any $\varepsilon>0$,
\begin{equation}
\mathbb{P}\!\left(\|\boldsymbol{\theta}_L-\boldsymbol{\mu}_t\|_2 \ge \varepsilon \mid D_t\right)
\le
\frac{\mathrm{Spread}(b_t)^2}{\varepsilon^2}.
\end{equation}
\end{proposition}
\emph{Proof.}
Let $Z=\|\boldsymbol{\theta}_L-\boldsymbol{\mu}_t\|_2^2\ge 0$ and use Markov:
$\mathbb{P}(Z\ge \varepsilon^2)\le \mathbb{E}[Z]/\varepsilon^2=\mathrm{tr}(\Sigma_t)/\varepsilon^2.$
$\square$

\begin{proposition}[Conditioning reduces variance in expectation]
Let $D_t=(D_{t-1},o_t,\mathbf{p}_t)$. Then
\begin{equation}
\mathbb{E}\left[\boldsymbol{\Sigma}_t \mid D_{t-1}\right]\preceq \boldsymbol{\Sigma}_{t-1}
\Longrightarrow
\mathbb{E}\left[\mathrm{Spread}(b_t)^2\mid D_{t-1}\right]\le \mathrm{Spread}(b_{t-1})^2.
\end{equation}
\end{proposition}
\emph{Proof.}
Apply the law of total variance in $\mathbb{R}^3$:
$\mathrm{Cov}(\theta\mid D_{t-1})=\mathbb{E}[\mathrm{Cov}(\theta\mid D_t)\mid D_{t-1}]
+\mathrm{Cov}(\mathbb{E}[\theta\mid D_t]\mid D_{t-1})$.
The second term is PSD, hence the inequality; take trace. $\square$

\subsubsection{Gaussian credible radius in 3D}
If the location posterior is approximately Gaussian,
$\boldsymbol{\theta}_L\mid D_t\sim \mathcal{N}(\boldsymbol{\mu}_t,\boldsymbol{\Sigma}_t)$,
then the whitened variable
$\mathbf{z}=\boldsymbol{\Sigma}_t^{-1/2}(\boldsymbol{\theta}_L-\boldsymbol{\mu}_t)$ satisfies
$\|\mathbf{z}\|_2^2\sim \chi^2_3$.
Let $c_{3,\delta}$ denote the $(1-\delta)$-quantile of $\chi^2_3$.
Then
\begin{equation}
\mathbb{P}\!\left(
\|\boldsymbol{\theta}_L-\boldsymbol{\mu}_t\|_2
\le
\sqrt{c_{3,\delta}}\;\sqrt{\mathrm{tr}(\boldsymbol{\Sigma}_t)}
\;\middle|\;D_t
\right)
\ge 1-\delta,
\end{equation}
where we used $\lambda_{\max}(\Sigma_t)\le \mathrm{tr}(\Sigma_t)$ to upper-bound the ellipsoidal
credible region by a Euclidean ball. If one wants an explicit closed-form constant,
a standard Laurent--Massart bound gives (for $\mathbf{z}\sim\mathcal{N}(0,I_3)$)
\begin{equation}
\mathbb{P}\!\left(
\|\mathbf{z}\|_2^2 \le 3 + 2\sqrt{3\ln(1/\delta)} + 2\ln(1/\delta)
\right)
\ge 1-\delta,
\end{equation}
which can be plugged into the same conversion step.

\subsection{Practical estimators under PF and diagonal-Gaussian student in 3D}
\textbf{PF teacher.}
Given weighted location particles $\{\boldsymbol{\theta}_{L}^{(i)},w^{(i)}\}_{i=1}^N$,
\begin{align}
\hat{\boldsymbol{\mu}}_t &= \sum_{i=1}^N w^{(i)}\boldsymbol{\theta}_{L}^{(i)},\\
\hat{\boldsymbol{\Sigma}}_t &= \sum_{i=1}^N w^{(i)}(\boldsymbol{\theta}_{L}^{(i)}-\hat{\boldsymbol{\mu}}_t)(\boldsymbol{\theta}_{L}^{(i)}-\hat{\boldsymbol{\mu}}_t)^\top,\\
\mathrm{Spread}(b_t) &= \sqrt{\mathrm{tr}(\hat{\boldsymbol{\Sigma}}_t)}
= \sqrt{\sum_{i=1}^N w^{(i)}\|\boldsymbol{\theta}_{L}^{(i)}-\hat{\boldsymbol{\mu}}_t\|_2^2}.
\end{align}
This is $O(N)$ per step.

\textbf{Diagonal-Gaussian student.}
If the student location marginal is diagonal,
$\boldsymbol{\Sigma}_t=\mathrm{diag}(\sigma_{x,t}^2,\sigma_{y,t}^2,\sigma_{z,t}^2)$, then
\begin{equation}
\mathrm{Spread}(b_t)=\sqrt{\sigma_{x,t}^2+\sigma_{y,t}^2+\sigma_{z,t}^2},
\end{equation}
which is $O(1)$ per step.

\subsection{ISLC from 2D to 3D}
We emphasize that the Distill-Belief architecture is unchanged; only the following elements need updates:
\begin{itemize}
\item \textbf{Forward model:} replace the 2D plume $\phi(x,y)$ with the 3D closed-form $\phi(x,y,z)$ in \eqref{eq:phi_general_3d}
(or the half-space variant \eqref{eq:images} if desired).
\item \textbf{Parameter vector:} include $z_s$ and $v_z$ (or a 3D direction parameterization).
\item \textbf{Belief features:} use the 3D location mean $\boldsymbol{\mu}_t\in\mathbb{R}^3$ and covariance
$\boldsymbol{\Sigma}_t\in\mathbb{R}^{3\times 3}$ (or its diagonal) instead of 2D.
\item \textbf{Stopping:} use $\mathrm{Spread}(b_t)=\sqrt{\mathrm{tr}(\boldsymbol{\Sigma}_t)}$ with the same semantics:
it certifies Bayes MSE and controls expected Euclidean error in 3D.
\item \textbf{Policy/action:} if the agent moves in 3D, increase the action dimension accordingly.
The intrinsic reward, teacher PF, and student distillation remain identical.
\end{itemize}

\textbf{Remark.}
With these changes, the method supports 3D ISLC while preserving:
(i) Bayes-aligned learning signals computed solely from the PF teacher,
(ii) constant-time deployment via the distilled student posterior,
and (iii) an uncertainty-aware stopping certificate with a direct Bayes-risk meaning in 3D.

\section{Environment and Metric Details}
\label{app:ISLCenv}

Building on the \textbf{open-source} \textbf{AutoSTE} codebase hosted on GitHub \href{https://github.com/Cunjia-Liu/AutoSTE}{https://github.com/Cunjia-Liu/AutoSTE}, we take its information-theoretic framework for mobile sensing, covering closed-loop path planning and source term estimation for atmospheric releases, as a concrete and reproducible starting point. In particular, AutoSTE and the associated line of work (e.g., information-based search and dual control for exploration–exploitation developed by collaborators) provide a clear reference implementation of how belief updates and information measures can drive autonomous sensing policies. Motivated by this foundation, our paper extends the setting from a single atmospheric-dispersion task to a broader, physics-grounded AI4Science formulation: we construct procedural, multi-field simulators (different governing equations and parameterizations) under a unified stochastic sensor model, and develop a belief-centric inference-and-control pipeline that remains unchanged across domains while only the physics backend (the forward operator) varies. This shift enables systematic evaluation under higher-dimensional parameters and prior misspecification, and targets robust belief contraction and accurate source characterization across diverse physical field modalities.

\begin{table*}[!t]
\centering
\caption{\textbf{ISLCenv multi-field simulator backends (procedural task generator).}
Each episode samples $\Theta_{\mathcal F}$ and instantiates a physics simulator
that provides point queries via $h_{\mathcal F}(\mathbf p;\Theta_{\mathcal F})$.
All modalities share the same Bernoulli-gated Gaussian-mixture sensor model
; domain differences enter only through the
backend that implements $h_{\mathcal F}$.}
\label{tab:multi_field_simulator_api}
\renewcommand{\arraystretch}{1.15}
\resizebox{\textwidth}{!}{
\begin{tabular}{l p{0.33\textwidth} p{0.49\textwidth} p{0.16\textwidth}}
\toprule
Field $\mathcal F$ &
Episode parameters $\Theta_{\mathcal F}$ (examples) &
Backend instantiation + forward-query API $h_{\mathcal F}(\mathbf p;\Theta_{\mathcal F})$ &
Shared sensor model \\
\midrule
Temperature (Temp.) &
Source parameters in $S(\phi,x,y)$; airflow $\vec v$; thermal diffusivity $\alpha(\phi)$ &
$\phi_T \leftarrow \texttt{Solve}(\mathcal G_T,\Theta_T)$;\;
$h_T(\mathbf p;\Theta_T)=\phi_T(\mathbf p)$ &
\multirow{6}{*}{\parbox{0.16\textwidth}{
$z \sim (1-P_d)\mathcal N(0,\sigma^2)$\\
$\quad +\,P_d\,\mathcal N(h_{\mathcal F}(\mathbf p;\Theta_{\mathcal F}),\bar\sigma^2)$
}}\\

Concentration (Conc.) &
Source $S(x,y)$; diffusion $\alpha$; flow $\vec v$; degradation $k_r$; turbulence $\vec\tau$ &
$\phi_C \leftarrow \texttt{Solve}(\mathcal G_C,\Theta_C)$;\;
$h_C(\mathbf p;\Theta_C)=\phi_C(\mathbf p)$ & \\

Magnetic (Mag.) &
Source $S(x,y)$; $\alpha$; effective flow $\vec v$; external field $\vec B$ &
$\phi_M \leftarrow \texttt{Solve}(\mathcal G_M,\Theta_M)$;\;
$h_M(\mathbf p;\Theta_M)=\phi_M(\mathbf p)$ & \\

Electric (Elec.) &
Conductivity $\sigma(x,y)$; charge density $\rho(x,y)$ &
$\phi_E \leftarrow \texttt{Solve}(\mathcal G_E,\Theta_E)$;\;
$h_E(\mathbf p;\Theta_E)=\phi_E(\mathbf p)$ & \\

Energy (En.) &
Source $S(x,y)$; $\alpha$; transport $\vec v$; absorption $\sigma_a$; scattering $\sigma_s$ &
$\phi_{En} \leftarrow \texttt{Solve}(\mathcal G_{En},\Theta_{En})$;\;
$h_{En}(\mathbf p;\Theta_{En})=\phi_{En}(\mathbf p)$ & \\

Noise (Noise) &
Source $S(x,y,f)$; $\alpha$; flow $\vec v$; attenuation $\gamma(f)$ &
$\phi_N \leftarrow \texttt{Solve}(\mathcal G_N,\Theta_N)$;\;
$h_N(\mathbf p;\Theta_N)=\phi_N(\mathbf p,f_0)$ (or band-aggregated) & \\
\bottomrule
\end{tabular}}
\end{table*}

\textbf{Procedural, physics-based tasks (not dataset-based).}
Our evaluation does \emph{not} rely on a fixed offline dataset. Instead,
each episode is generated \emph{on the fly} by a physics-based field simulator:
(i) sample a field modality $\mathcal{F}$ and a latent parameter vector
$\Theta_{\mathcal{F}}$ (source parameters and physical coefficients);
(ii) instantiate a continuous field realization $\phi_{\mathcal F}(\mathbf x)$
by solving the field-specific governing equation (or an equivalent surrogate);
(iii) interactively render agent observations through a shared stochastic
sensor model. Changing $\Theta_{\mathcal F}$ produces qualitatively different
field geometries (e.g., plume shapes, potential contours, attenuation patterns),
yielding effectively unbounded scenario diversity while remaining grounded in
physical mechanisms.

\textcolor{cyan}{\textbf{Episode generator and prior misspecification.}}
At the beginning of each episode, $\Theta_{\mathcal F}$ is sampled from a
training prior $p_{\text{train}}(\Theta_{\mathcal F})$. To stress-test robustness
to prior misspecification, we additionally evaluate a \emph{Moderate error}
setting where $50\%$ of test episodes are generated from parameters inside the
prior support and $50\%$ from outside.

\textcolor{cyan}{\textbf{Unified simulator interface (forward-query API).}}
Across all modalities, the environment exposes the same forward-query interface
\begin{equation}
h_{\mathcal{F}}(\mathbf p;\Theta_{\mathcal{F}})\;:=\;\phi_{\mathcal{F}}(\mathbf p),
\end{equation}
i.e., the simulator returns the \emph{noise-free} scalar field value at a queried
location $\mathbf p$. Thus, domain differences enter \emph{only} through the
simulator backend that implements $h_{\mathcal F}$.

\textcolor{cyan}{\textbf{Shared observation likelihood across fields.}}
To ensure a fair cross-domain comparison, we keep the same stochastic sensor
model across all fields and change \emph{only} $h_{\mathcal F}$.
At step $t$, the agent at $\mathbf p_t$ receives a scalar observation
\begin{equation}
z_t = D_t\big(h_{\mathcal F}(\mathbf p_t;\Theta_{\mathcal F}) + \bar v_t\big)
      + (1-D_t)\,v_t,
D_t \sim \mathrm{Bernoulli}(P_d),
\end{equation}
where $\bar v_t \sim \mathcal N(0,\bar\sigma_t^2)$ denotes measurement noise and
$v_t \sim \mathcal N(0,\sigma_t^2)$ denotes background noise. Equivalently,
\begin{equation}
p(z_t \mid \Theta_{\mathcal F})
=(1-P_d)\,\mathcal N(z_t;0,\sigma_t^2)
+P_d\,\mathcal N\!\big(z_t;h_{\mathcal F}(\mathbf p_t;\Theta_{\mathcal F}),\bar\sigma_t^2\big).
\end{equation}

\textcolor{cyan}{\textbf{Key point (what changes vs.\ what stays the same).}}
Across modalities, the ISLC task interface is identical: the agent selects
sensing actions, receives noisy scalar observations, and maintains a belief
over $\Theta_{\mathcal F}$. Therefore, our belief-update and policy-learning
pipeline is unchanged across domains; \emph{only the physics-driven simulator
backend (the implementation of $h_{\mathcal F}$) changes.}

\textcolor{cyan}{\textbf{
Table~\ref{tab:multi_field_simulator_api} summarizes the simulator backends
for all field families.}} We next briefly summarize each field family and its main physical effects, highlighting how different choices of $\Theta_{\mathcal F}$ and the corresponding solver backend shape the resulting field geometry and the induced inverse problem difficulty.

\begin{itemize}
\item  \textbf{Temperature (Temp.)} \cite{hite2019localization}
The temperature field is driven by localized heat sources and shaped by thermal diffusion and advective airflow, typically forming elongated heat plumes aligned with the flow direction. When the thermal diffusivity depends on temperature, the field becomes nonlinear and spatially heterogeneous, producing non-uniform gradients that complicate source localization.

\item \textbf{Concentration (Conc.)} \cite{stockie2011mathematics} The concentration field models the transport of pollutants or gases under diffusion, advection, and chemical degradation. The resulting plume often exhibits long-tailed decay and turbulence-induced variability, causing weak and noisy signals in the far field and increasing the difficulty of robust inverse inference.

\item \textbf{Magnetic (Mag.)} \cite{brandenburg2005astrophysical}
The magnetic field is represented through a scalar potential influenced by internal sources and external magnetic components. External field terms can significantly distort the potential contours, introducing strong directional bias and coupling between source parameters and background effects, which increases ambiguity in localization.

\item \textbf{Electric (Elec.)} \cite{cheney1999electrical}
The electric potential field arises from spatially varying conductivity and charge distributions and is governed by a variable-coefficient elliptic equation. Heterogeneous conductivity can bend and concentrate equipotential lines, leading to regions with sharply different sensitivities and making source parameters harder to disentangle from medium properties.

\item \textbf{Energy (En.)} \cite{arridge1999optical}
The energy density field captures radiative or energy transport with diffusion, advection, absorption, and scattering effects. These processes cause rapid attenuation and spatial smoothing, resulting in blurred field structures and reduced long-range observability, which places higher demands on adaptive sensing strategies.

\item \textbf{Noise (Noise)} \cite{picaut2002numerical}
The noise intensity field models acoustic propagation with frequency-dependent attenuation and source spectra. Different frequencies decay at different rates and interact with the medium anisotropically, producing multi-scale spatial patterns and increasing the dimensionality and complexity of the underlying inference problem.

\end{itemize}

\subsection{Detailed Definitions of Evaluation Metrics}
\label{app:metrics}

We evaluate both task performance and inference quality using the following metrics. \textbf{(1) \textcolor{cyan}{Success Rate (SR)}} measures the fraction of episodes that terminate successfully under the predefined stopping criterion. An episode is considered successful if the stopping condition (e.g., $\mathrm{Spread}_t < \zeta$) is satisfied within the maximum horizon. Formally, $\mathrm{SR} = \frac{1}{N_{\text{epi}}} \sum_{k=1}^{N_{\text{epi}}} \mathbb{I}$ \{\text{episode } k \text{ terminates successfully}\}. \textbf{(2) \textcolor{cyan}{Trajectory Efficiency (TE)}} quantifies the control cost required to complete a task. We report either the total traveled distance $ \mathrm{TE}_{\text{dist}} = \sum_{t=1}^{T_k} \|\boldsymbol{p}_{t+1}-\boldsymbol{p}_t\|,$ or equivalently the number of steps $T_k$ until termination.
Lower values indicate more efficient exploration and planning. \textbf{(3) \textcolor{cyan}{Source Localization Error (SLE)}} measures the Euclidean distance between the estimated
and true source locations at termination:
$ \mathrm{SLE} = \big\|\hat{\boldsymbol{p}}_s - \boldsymbol{p}_s\big\|_2,$ where $\hat{\boldsymbol{p}}_s$ is obtained from the posterior mean of $(x_s,y_s)$.
We report the mean and standard deviation over test episodes. \textbf{(4) \textcolor{cyan}{Full-Parameter Estimation Error (FPE)}} measures the estimation accuracy of the full physical parameter vector $\boldsymbol{\Theta}\in\mathbb{R}^d$ at termination, beyond source localization (e.g., source strength, wind parameters, diffusion coefficients). Let $\hat{\boldsymbol{\Theta}}$ denote the posterior mean estimate given by the belief. We define the aggregate estimation error using RMSE (or MAE):
$ \mathrm{FPE}_{\mathrm{RMSE}}
= \sqrt{\frac{1}{d}\sum_{j=1}^{d}\big(\hat{\Theta}_{j}-\Theta_{j}\big)^{2}}$,$
\mathrm{FPE}_{\mathrm{MAE}}
= \frac{1}{d}\sum_{j=1}^{d}\big|\hat{\Theta}_{j}-\Theta_{j}\big|. $ This metric captures whether the method accurately recovers non-location parameters such as source strength, wind, and diffusion. 
(5) \textcolor{cyan}{Uncertainty Quality (UQ)}. When applicable, we evaluate the quality of the predicted posterior uncertainty. Specifically, we report the negative log-likelihood (NLL) of ground-truth parameters under the predicted posterior, $NLL = - \log q_{\varphi}(\Theta)$, and 
calibration metrics that compare empirical estimation errors against predicted posterior variances. These metrics assess whether the learned belief is not only accurate but also well-calibrated.

\begin{algorithm}[htbp]
\caption{Deployment (PF teacher removed)}
\label{alg:deploy}
\begin{algorithmic}[1]
\REQUIRE Trained student $q_\varphi$, policy $\pi_\theta$, stop threshold $\zeta$, horizon $H$
\FOR{$t=1$ \TO $H$}
  \STATE Observe $(o_t,\boldsymbol{p}_t)$
  \STATE Compute $q_\varphi(\boldsymbol{\Theta}\mid o_{1:t},\boldsymbol{p}_{1:t})$ and $\mathrm{Spread}(b_t)$
  \IF{$\mathrm{Spread}(b_t) < \zeta$}
    \STATE \textbf{stop}
  \ENDIF
  \STATE Select $a_t \leftarrow \mathbb{E}_{\pi_\theta}[a \mid o_t,\boldsymbol{p}_t,f_{bel}(b_t)]$ and execute
\ENDFOR
\end{algorithmic}
\end{algorithm}

\begin{table}[htbp]
  \centering
  \caption{Scenario parameter distributions for training.}
  \label{tab:training_parameters}
  \begin{tabular}{ll}
    \hline
    \textbf{Parameter} & \textbf{Distribution} \\ \hline
    Source location $x_s$ & $\mathcal{U}(5,20)$ \\
    Source location $y_s$ & $\mathcal{U}(10,20)$ \\
    Release strength $q_s$ & $\mathcal{U}(10,3000)$ \\
    Wind components $(u_x,u_y)$ & $\mathcal{U}(0,6)\times \mathcal{U}(0,6)$ \\
    Decay parameter $\lambda$ & $\mathcal{U}(0,8)$ \\
    Diffusivity $\alpha$ & $\mathcal{U}(1,5)$ \\ \hline
  \end{tabular}
\end{table}

\section{Additional Experimental Results}

\subsection{Additional Experimental Context}
\label{sec:method:overview}

Building on the POMDP formulation in Sec.~\ref{sec:Preliminaries}, we describe how to \emph{maintain} and \emph{exploit} a Bayesian belief over the unknown field parameters. Our focus is on (i) approximating the posterior $p(\boldsymbol{\Theta}\mid o_{1:t},\boldsymbol{p}_{1:t})$ in a Bayes-consistent yet tractable way, and (ii) turning this belief into a control signal that tells the agent where to measure next. At each step $t$ we maintain a belief distribution $b_t(\boldsymbol{\Theta})$ and construct a belief state
$\boldsymbol{\psi}_t=\big[ o_t^\top,\boldsymbol{p}_t^\top, f_{\text{bel}}(b_t)^\top \big]^\top$,
where $[\cdot,\cdot,\cdot]$ denotes vector concatenation.
This state $\boldsymbol{\psi}_t$ is fed into a parametric policy $\pi_\theta(a_t\mid \boldsymbol{\psi}_t)$.
The map $f_{\text{bel}}(b_t)$ does not output the (unknown) parameter vector itself, but the \emph{distribution} induced by the belief (details in Sec.~\ref{sec:policy}). We maintain a posterior over {the \emph{full} parameter vector
$\boldsymbol{\Theta}\in\mathbb{R}^d$ (e.g., source location/strength and environmental factors)}.

Our architecture realizes this through two coupled layers:
$\bullet$ an \textbf{\textcolor{cyan}{inference layer}} whose teacher is a statistically consistent particle filter (PF) that performs Bayes-correct updates over the full $\boldsymbol{\Theta}$, and whose student is an amortized posterior approximator trained by distillation to match the teacher posterior;
$\bullet$ an \textbf{\textcolor{cyan}{execution layer}} that learns a policy from a KL-based intrinsic reward defined in belief space, uses belief features from the student, and terminates when the spread-based certificate falls below a threshold. Crucially, only the actor-critic updates the sensing policy. {The PF teacher has no trainable parameters: it performs Bayes-consistent filtering updates to produce posterior beliefs and intrinsic rewards, but it never optimizes a policy.} Moreover, the intrinsic reward is computed \emph{exclusively} from the PF teacher during training, while at test time we discard the teacher and rely solely on the student to supply belief features and stopping statistics at constant cost. This separation preserves Bayes semantics during learning while keeping deployment overhead independent of the particle budget.

\subsection{Additional Ablation Details}
\label{app:ablation_details}

This appendix provides a more detailed discussion of the ablation studies summarized in
Table~\ref{tab:ablation_main} and Table~\ref{tab:ablation_reward}.
Our goal is to clarify the role of each component in the belief-optimization pipeline and
to explain how different reward designs affect learning dynamics and posterior quality.
Unless otherwise stated, all ablations use the same environments, interaction budgets,
network architectures, and optimization hyperparameters as the full method.

\subsubsection{Belief-Optimization Pipeline Ablations}
\label{app:belief_ablation}

Table~\ref{tab:ablation_main} ablates key components of the proposed teacher--student
belief-optimization framework. Each variant removes or modifies a single ingredient while
keeping all others fixed.

\textbf{Removing the KL-based information-gain reward.}
In this variant, the dense KL-based intrinsic reward is replaced with a surrogate signal
that is not explicitly aligned with belief contraction.
As shown in Table~\ref{tab:ablation_main}, this change leads to a substantial drop in success
rate (SR) and a marked increase in trajectory efficiency (TE), together with degraded SLE,
FPE, and uncertainty quality (UQ).
These results indicate that dense, belief-space shaping is critical for efficient exploration:
without it, the agent receives weak or delayed learning signals and struggles to acquire
informative sensing policies within a limited budget.

\textbf{Reward-from-Student (KL computed from the student belief).}
Here we compute the information-gain objective using the student’s approximate belief instead
of the particle-filter (PF) teacher.
Although this variant retains a dense reward, it consistently degrades both task performance
and uncertainty quality, with the largest deterioration observed in UQ (NLL).
This supports the concern that using the same approximation for both reward computation and
policy input introduces exploitable shortcuts: the policy can increase its own reward by
manipulating the student belief without necessarily improving posterior accuracy.
Using a Bayes-correct PF teacher to compute the reward mitigates this failure mode.

\textbf{Removing distillation (PF-only at test time).}
In this ablation, the PF teacher is used both during training and at test time, and the
student belief is not used for deployment.
While SR and estimation accuracy remain competitive, this variant forfeits the computational
benefits of amortized inference and incurs significantly higher test-time cost.
This result highlights the role of distillation in achieving practical deployability, rather
than merely improving raw performance.

\textbf{Student-only training without PF supervision.}
This variant removes the PF teacher entirely and trains the student belief and policy without
Bayes-correct supervision.
Performance degrades sharply across all metrics, including SR, SLE, FPE, and UQ.
The result confirms that the student alone does not provide sufficiently accurate or stable
belief targets early in training, and that PF-based supervision is essential for guiding
belief learning and exploration.

\textbf{Removing Spread features and Spread-based stopping.}
To disentangle the role of uncertainty information, we separately remove (i) the Spread feature
from the policy input and (ii) the Spread-based stopping rule.
Both variants mainly affect efficiency and calibration rather than raw success rate.
Dropping Spread features reduces the policy’s ability to adapt actions to uncertainty levels,
while disabling Spread-based stopping leads to longer trajectories with diminishing returns.
Together, these results support the interpretation of Spread as an uncertainty certificate
that enables effective accuracy--budget control.

\textbf{Removing MH rejuvenation in the PF teacher.}
In this variant, the PF teacher operates without MH rejuvenation moves, increasing the risk of
particle impoverishment.
The resulting degradation in performance and UQ suggests that maintaining particle diversity
is important for providing stable and informative supervision signals during training.

\subsubsection{Reward Design Ablations}
\label{app:reward_ablation}

Table~\ref{tab:ablation_reward} focuses on reward design and compares the proposed dense
KL-based information-gain reward with several alternatives.

\textbf{Hard-success reward only.}
The hard-success reward provides non-zero feedback only when the task is successfully completed.
While intuitive, this formulation is extremely sparse for active sensing.
As shown in Table~\ref{tab:ablation_reward}, it leads to lower SR, worse uncertainty quality,
and more than a threefold increase in Steps@70\% SR, indicating slow and unstable learning.

\textbf{Dense KL-based information gain.}
In contrast, the KL-based reward provides shaped feedback throughout an episode, directly
encouraging belief contraction.
This results in substantially better sample efficiency, improved trajectory efficiency,
and higher-quality posteriors, demonstrating the advantage of belief-aligned intrinsic rewards
for active sensing.

\textbf{Mixed and curriculum rewards.}
The mixed reward combines KL-based shaping with a hard-success signal, while the curriculum
variant transitions from pure KL shaping to a mixed objective later in training.
Both variants achieve performance close to the full method, suggesting that task-level rewards
can complement belief-based shaping once effective exploration has been learned, without
sacrificing sample efficiency or posterior quality.

\textbf{Overall}, these ablations confirm that (i) dense, Bayes-aligned information-gain rewards are
crucial for efficient learning, and (ii) computing such rewards from a Bayes-correct teacher
is key to avoiding shortcut solutions and preserving uncertainty calibration.

\subsection{Supplementary Figures for Experimental Results}

\begin{figure}[t]
  \centering
  \includegraphics[width=\linewidth]{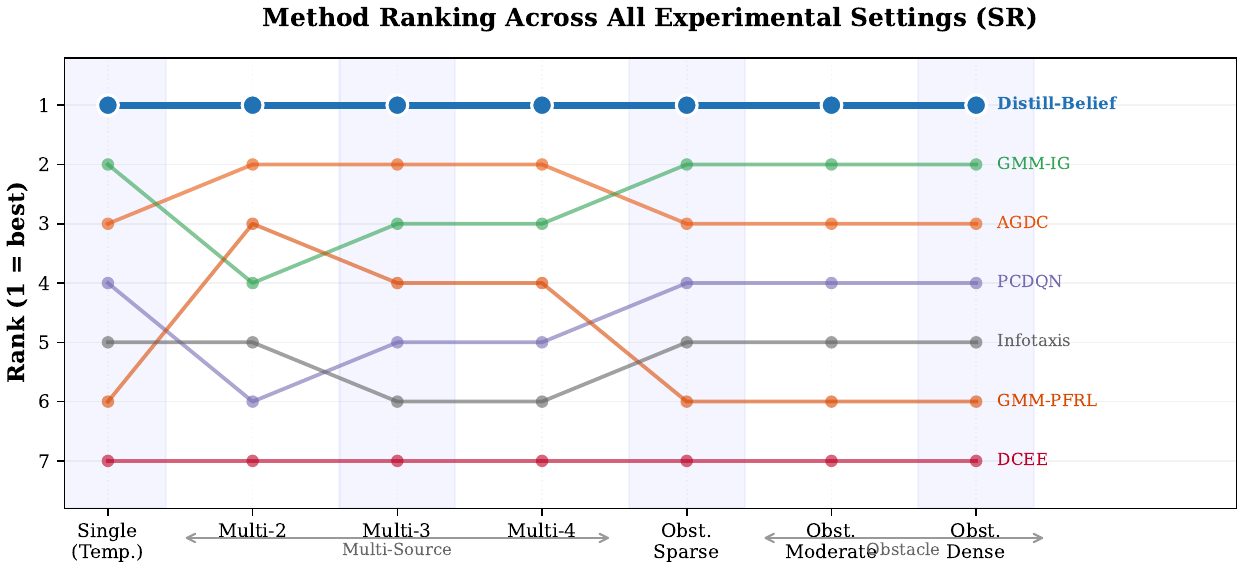}
  \caption{Method ranking across all experimental settings based on Success Rate (SR). Distill-Belief consistently achieves Rank\,1 across single-source, multi-source (2--4), and obstacle (sparse/moderate/dense) settings.}
  \label{fig:bump_chart}
\end{figure}

Figure~\ref{fig:bump_chart} provides a compact ranking summary across the appendix experiments. Its main role is to show that \distill{} is not only strong in one setting, but consistently occupies the top rank across the single-source, multi-source, and obstacle suites summarized later in this supplement.












\clearpage




\clearpage
\subsection{Main Experimental Results}
\label{sec:main-results}

\subsubsection{Table~\ref{tab:fundamental_experiments}: Single-Source Baseline Comparison Across Seven Field Modalities}
\label{sec:table2}

Table~\ref{tab:fundamental_experiments} constitutes the most critical experimental result in the paper, presenting a comprehensive comparison of eight methods---\distill{}, GMM-IG, AGDC, PCDQN, GMM-PFRL, Infotaxis, DCEE, and Entrotaxis---across seven distinct physical field modalities: Temperature (Temp.), Concentration (Conc.), Magnetic (Mag.), Electric (Elec.), Gas Dispersion (Gas), Energy Attenuation (En.), and Acoustic Noise (Noise). Performance is assessed along four complementary metrics: Success Rate ($\SR$), Trajectory Efficiency ($\TE$), Robustness Variance ($\REV$), and Local Posterior Spread ($\LPS$). We analyze each metric dimension in turn.

\paragraph{Success Rate ($\SR$).}
\distill{} achieves the highest $\SR$ across all seven fields, ranging from 0.63 (Energy Attenuation) to 0.96 (Gas Dispersion). This consistency is particularly compelling on the more challenging modalities: on Electric fields ($\SR = 0.82$) and Energy Attenuation ($\SR = 0.63$), planning-based baselines suffer severe degradation---Entrotaxis collapses to $\SR \in [0.14, 0.25]$ across all fields, while Infotaxis and DCEE fall below 0.50 on at least three modalities. The closest competitor, GMM-IG, maintains $\SR \in [0.89, 0.92]$ on the easier fields but consistently trails \distill{} by 4--5 percentage points. This gap, while numerically modest, represents a statistically significant improvement as confirmed by the critical difference analysis in Figure~16 (see \cref{sec:fig16}).

\paragraph{Trajectory Efficiency ($\TE$).}
\distill{} requires the fewest steps to termination across all fields ($\TE \in [17, 20]$), compared to 40--62 steps for planning methods---a 2--3$\times$ improvement. This result is crucial because it demonstrates that the superior $\SR$ is \emph{not} achieved by running longer episodes; rather, the learned policy is intrinsically more efficient at information acquisition. Among RL baselines, GMM-IG and PCDQN achieve intermediate $\TE$ values ($\approx$20--25 steps) but at the cost of substantially lower $\SR$, indicating that these methods terminate quickly but inaccurately.

\paragraph{Robustness Variance ($\REV$).}
The $\REV$ metric reveals an order-of-magnitude separation between method families. RL-based methods (\distill{}, GMM-IG, PCDQN) exhibit $\REV \in [0.09, 0.15]$, while planning-based methods (AGDC, Infotaxis, Entrotaxis) display $\REV \in [0.7, 1.6]$. This indicates that planning methods are not only slower but also exhibit dramatically higher variance in their localization error, rendering them unreliable for safety-critical deployment scenarios.

\paragraph{Local Posterior Spread ($\LPS$).}
\distill{} achieves $\LPS \in [0.05, 0.08]$, substantially lower than all competitors ($\LPS \in [0.1, 0.7]$). This metric is the most informative diagnostic in the table: a low $\LPS$ at termination proves that the high $\SR$ is not an artifact of premature stopping with a diffuse belief state, but instead reflects genuine posterior contraction around the true source parameters. By contrast, Infotaxis, Entrotaxis, and DCEE terminate with $\LPS \in [0.39, 0.70]$, confirming that their posteriors have not truly converged even when they declare success.

\paragraph{Cross-Metric Synthesis.}
The joint examination of all four metrics reveals that \distill{} is the \emph{only} method that simultaneously achieves the best performance on every dimension. Other methods exhibit characteristic trade-off patterns: GMM-IG sacrifices $\SR$ for $\TE$; Infotaxis achieves moderate $\SR$ at the cost of extreme $\TE$ and $\REV$; Entrotaxis fails comprehensively across all metrics. This ``no-compromise'' property of \distill{} is also consistent with the cross-field and robustness visualizations collected later in the appendix.

\subsection{Multi-Source and Obstacle-Constrained Extensions}
\label{sec:extensions}

\subsubsection{Table~\ref{tab:multi_source_results_sr_steps}: Multi-Source Localization (2/3/4 Sources, Temperature Field)}
\label{sec:table3}

Table~\ref{tab:multi_source_results_sr_steps} evaluates scalability under multi-modal posterior distributions induced by 2, 3, and 4 simultaneous sources in the Temperature field. This setting is substantially more challenging because the posterior landscape becomes multi-peaked, requiring the agent to disambiguate between plausible source hypotheses through strategic exploration.

\paragraph{Degradation Rates.}
All methods exhibit performance degradation as source count increases, but the \emph{rate} of degradation varies dramatically. \distill{} experiences a 20.8\% decrease in $\SR$ from 2 sources ($\SR = 0.77$) to 4 sources ($\SR = 0.61$), which is the most graceful degradation among all tested methods. By contrast, Infotaxis suffers a 44.6\% decline ($0.65 \to 0.36$) and DCEE a 35.0\% decline ($0.40 \to 0.26$). These degradation rates are visualized in Figure~14a (\cref{sec:fig14}) and are strongly correlated with method type: RL-based approaches are systematically more robust than planning-based ones.

\paragraph{Trajectory Efficiency Under Multi-Source Conditions.}
$\TE$ increases from 28 to 40 steps for \distill{} (a 42.9\% increase), whereas DCEE surges from 72 to 100 steps (38.9\% increase on an already-high baseline). The absolute gap widens: at 4 sources, \distill{} requires 2.5$\times$ fewer steps than DCEE. This demonstrates that the learned policy maintains its ability to perform global disambiguation rather than collapsing into local sensing patterns a failure mode that is particularly evident in greedy planning methods.

\paragraph{Variance Analysis.}
As shown in the shaded regions of Figure~19 (\cref{sec:fig19}), \distill{} also exhibits the smallest inter-run variance across all source counts, indicating that its performance is consistently reproducible rather than being driven by occasional lucky initializations.

\subsubsection{Table~\ref{tab:obstacle_results_sr_steps}: Obstacle-Constrained Environments (Sparse/Moderate/Dense)}
\label{sec:table4}

Table~\ref{tab:obstacle_results_sr_steps} tests practical deployment feasibility in non-convex feasible regions with varying obstacle densities.

\paragraph{Performance Across Densities.}
\distill{} maintains the highest $\SR$ (0.90/0.86/0.80) and lowest $\TE$ (21/25/31 steps) across sparse, moderate, and dense configurations, respectively. GMM-IG is the closest RL competitor ($\SR$: 0.85/0.81/0.74; $\TE$: 24/28/35), trailing by 4--6 percentage points in $\SR$ and 3 steps in $\TE$ at each density level. The gap is stable across densities, indicating that the advantage of \distill{} is structural rather than an artifact of specific obstacle configurations.

\paragraph{Failure Modes of Baselines.}
The obstacle setting exposes distinct failure modes. AGDC suffers from severe efficiency degradation ($\TE$: 46/52/61), suggesting that its information-gathering strategy conflicts with path-planning constraints in non-convex domains. Planning methods degrade most dramatically in dense obstacles: Infotaxis reaches $\TE = 71$, while DCEE collapses to $\SR = 0.38$ with $\TE = 80$---effectively failing to localize the source within the episode budget in most trials. These results highlight that classical planning methods, which compute locally optimal next-step actions, are fundamentally ill-suited for obstacle-rich environments where myopic information gain may lead to dead-end trajectories.

\paragraph{Degradation Analysis.}
As quantified in Figure~14b (\cref{sec:fig14}), \distill{} degrades by only 11.1\% from sparse to dense obstacles, compared to 12.9\% for GMM-IG and 30.9\% for DCEE. The RL methods' natural adaptability in non-convex environments stems from their ability to learn implicit path-planning strategies during training, whereas planning methods rely on one-step-ahead optimization that cannot anticipate the long-term consequences of obstacle avoidance.

\subsection{Ablation Studies}
\label{sec:ablation}

\subsubsection{Table~\ref{sec:table11}: Belief-Optimization Pipeline Ablation}
\label{sec:table11}
\begin{table}[htbp]
\centering
\caption{Ablation study on belief-optimization components.}
\label{tab:ablation_main}
\resizebox{0.5\textwidth}{!}{
\begin{tabular}{lccccc}
\toprule
Variant & SR $\uparrow$ & TE $\downarrow$ & SLE $\downarrow$ & FPE $\downarrow$ & UQ $\downarrow$ \\
\midrule
Full (Teacher KL + Distill + $\mu,\sigma,\text{Spread}$ + Spread-stop) 
& 0.95$\pm$0.02 & 18.7$\pm$0.9 & 0.31$\pm$0.03 & 0.18$\pm$0.02 & 1.12$\pm$0.06 \\
\midrule
w/o KL reward (replace $r^{IG}$ with surrogate reward) 
& 0.87$\pm$0.03 & 23.1$\pm$1.2 & 0.44$\pm$0.05 & 0.26$\pm$0.03 & 1.47$\pm$0.08 \\
Reward-from-Student (compute KL using student belief) 
& 0.89$\pm$0.03 & 21.8$\pm$1.1 & 0.43$\pm$0.04 & 0.24$\pm$0.03 & 1.58$\pm$0.10 \\
w/o Distillation (PF-only at test) 
& 0.88$\pm$0.03 & 21.2$\pm$1.1 & 0.38$\pm$0.04 & 0.22$\pm$0.03 & 1.28$\pm$0.07 \\
Student-only training (no PF teacher supervision) 
& 0.82$\pm$0.04 & 24.8$\pm$1.5 & 0.52$\pm$0.06 & 0.32$\pm$0.04 & 1.74$\pm$0.12 \\
w/o Spread feature (policy input drops Spread) 
& 0.90$\pm$0.03 & 20.4$\pm$1.0 & 0.37$\pm$0.04 & 0.21$\pm$0.03 & 1.26$\pm$0.08 \\
w/o Spread-stop (fixed horizon or heuristic stopping) 
& 0.92$\pm$0.02 & 22.5$\pm$1.2 & 0.35$\pm$0.03 & 0.20$\pm$0.02 & 1.20$\pm$0.07 \\
w/o MH move (PF teacher without MH rejuvenation) 
& 0.88$\pm$0.03 & 22.5$\pm$1.2 & 0.40$\pm$0.04 & 0.24$\pm$0.03 & 1.39$\pm$0.09 \\
\bottomrule
\end{tabular}}
\end{table}

Table~\ref{sec:table11} presents the most comprehensive ablation in the paper, systematically removing individual components to quantify their marginal contributions. We analyze each variant and its implications.

\paragraph{Full Model (Baseline).}
The complete \distill{} pipeline achieves $\SR = 0.95$, $\TE = 18.7$ steps, and $\UQ = 1.12$, establishing the reference point for all ablations.

\paragraph{Without KL Reward ($\Delta\SR = -0.08$).}
Removing the dense KL-divergence-based reward causes the single largest $\SR$ drop ($0.95 \to 0.87$), accompanied by $\TE$ degradation from 18.7 to 23.1 steps and $\UQ$ worsening from 1.12 to 1.47. This confirms that belief-space reward shaping is \emph{critical} for efficient exploration: without it, the agent lacks a dense gradient signal to guide its information-seeking behavior, leading to undirected trajectories and poorly calibrated uncertainty estimates.

\paragraph{Reward-from-Student ($\Delta\SR = -0.06$).}
Computing the KL reward from the student's own belief (rather than the teacher particle filter) produces $\SR = 0.89$ but the \emph{worst} $\UQ$ degradation ($1.58$) among all variants. This result provides direct empirical evidence for the ``reward hacking'' concern central to the paper's motivation: when the same approximate posterior simultaneously defines the reward signal and the policy input, the agent can exploit systematic biases in the approximation to inflate apparent information gain without achieving genuine posterior contraction. The decoupled teacher-student architecture is thus essential for maintaining reward integrity.

\paragraph{Without Distillation ($\Delta\SR = -0.07$).}
Retaining the full particle filter at test time yields $\SR = 0.88$, demonstrating that the distillation step does not degrade raw task performance. Instead, its primary value lies in deployment efficiency: as shown in Table~9 (\cref{sec:table9}), removing distillation incurs a 6.5$\times$ latency increase at test time. This variant confirms that distillation is an engineering innovation for practical deployment rather than a performance-enhancing modification.

\paragraph{Student-Only ($\Delta\SR = -0.13$).}
This is the most severely degraded variant ($\SR = 0.82$, $\UQ = 1.74$), demonstrating that the particle filter teacher supervision is \emph{essential} for training. Without a Bayes-correct reference signal, the student network cannot provide sufficiently accurate belief targets, and both exploration quality and learning stability suffer substantially.

\paragraph{Without Spread Feature ($\Delta\SR = -0.05$).}
Removing the posterior spread from the policy's observation vector primarily affects efficiency and calibration rather than raw $\SR$. This indicates that the spread feature enables the policy to modulate its behavior based on current uncertainty magnitude---\eg, switching from broad exploration to local refinement as the posterior contracts.

\paragraph{Without Spread-Based Stopping ($\Delta\SR = -0.03$).}
Disabling the adaptive stopping criterion increases $\TE$ from 18.7 to 22.5 without commensurate $\SR$ improvement, confirming that the spread-based termination condition successfully identifies the point of diminishing returns, avoiding wasteful additional steps after the posterior has sufficiently converged.

\paragraph{Without Metropolis--Hastings Rejuvenation ($\Delta\SR = -0.07$).}
Removing the MH move step from the particle filter causes significant degradation, underscoring the importance of particle diversity maintenance for providing stable and informative teacher signals throughout training.

\subsubsection{Figure~15: Waterfall Chart of Incremental Contributions}
\label{sec:fig15}
\begin{figure}[htbp]
  \centering
  \includegraphics[width=\linewidth]{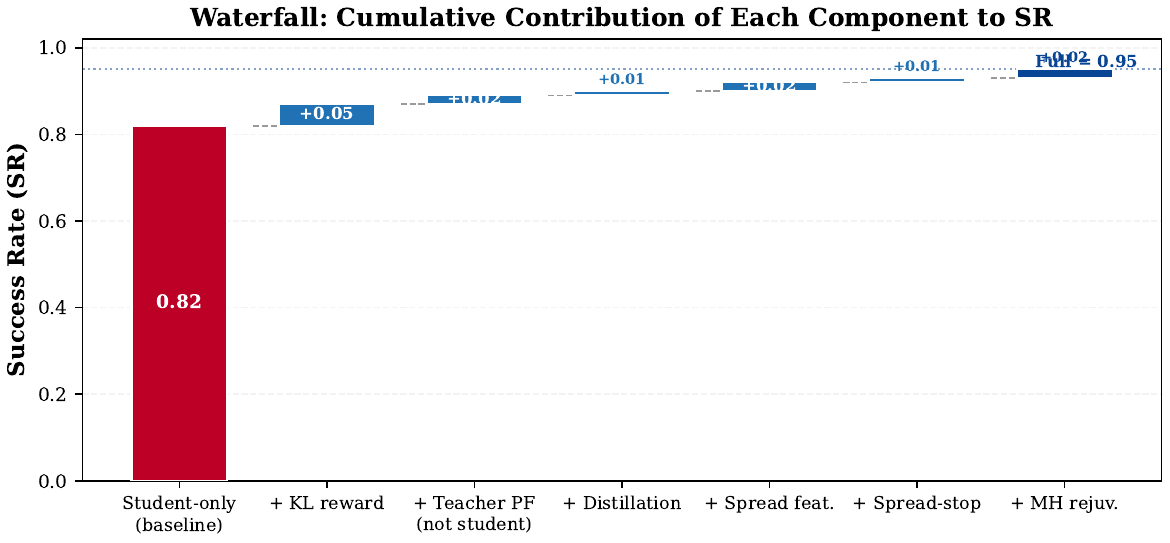}
  \caption{Waterfall diagram showing the cumulative contribution of each component to SR. Starting from the student-only baseline (0.82), each added module incrementally improves performance to the full model (0.95).}
  \label{fig:waterfall}
\end{figure}

Figure~15 provides a complementary visualization to Table~11, presenting the cumulative effect of progressively adding components to the Student-only baseline ($\SR = 0.82$). The increments are: KL reward ($+0.05$), teacher particle filter ($+0.02$), distillation ($+0.01$), spread features ($+0.02$), spread-based stopping ($+0.01$), and MH rejuvenation ($+0.02$), culminating at $\SR = 0.95$.

The key insight from this visualization is that \emph{no single component dominates}: the largest individual contribution (KL reward, $+0.05$) accounts for only 38\% of the total improvement, while the remaining five components collectively contribute 62\%. This synergistic design implies that the system cannot be trivially simplified by removing ``minor'' components without incurring meaningful performance loss---a hallmark of a well-engineered modular architecture.

\begin{table}[htbp]
\centering
\caption{Ablation on reward design (mean$\pm$std). 
We compare dense information-gain (KL) reward with sparse hard success reward. Steps@70\% SR denotes training steps needed to reach 70\% success rate.}
\label{tab:ablation_reward}
\resizebox{0.5\textwidth}{!}{
\begin{tabular}{lccccc}
\toprule
Reward design 
& SR $\uparrow$ 
& TE $\downarrow$ 
& SLE $\downarrow$ 
& UQ $\downarrow$ 
& Steps@70\% SR $\downarrow$ \\
\midrule
KL / IG reward (ours) 
& 0.95$\pm$0.02 & 18.7$\pm$0.9 & 0.31$\pm$0.03 & 1.12$\pm$0.06 & 0.32$\pm$0.04M \\
Hard success reward only 
& 0.79$\pm$0.05 & 24.5$\pm$1.8 & 0.55$\pm$0.07 & 1.69$\pm$0.14 & 1.10$\pm$0.20M \\
KL + hard (mixed) 
& 0.93$\pm$0.02 & 19.8$\pm$1.0 & 0.32$\pm$0.03 & 1.14$\pm$0.06 & 0.35$\pm$0.05M \\
Curriculum: KL $\rightarrow$ KL+hard 
& 0.94$\pm$0.02 & 19.2$\pm$0.9 & 0.30$\pm$0.03 & 1.11$\pm$0.06 & 0.30$\pm$0.04M \\
\bottomrule
\end{tabular}}
\end{table}

\subsubsection{Table~\ref{tab:ablation_reward}: Reward Design Ablation}
\label{sec:table10}

Table~\ref{tab:ablation_reward} isolates the effect of reward function design on both asymptotic performance and sample efficiency.

\paragraph{KL/IG Reward (Proposed).}
The proposed dense KL-divergence reward achieves $\SR = 0.95$ and reaches 70\% $\SR$ at $0.32$M environment steps, establishing the best performance--efficiency combination.

\paragraph{Hard Success Reward Only.}
Using only a sparse binary task-completion reward yields $\SR = 0.79$---the lowest among all variants---and requires $1.10$M steps to reach 70\% $\SR$, a $3.4\times$ sample efficiency degradation. This result validates the paper's central claim that sparse rewards are fundamentally insufficient for active sensing tasks, where the agent must learn to gather information incrementally rather than relying on a single success/failure signal at episode termination.

\paragraph{KL + Hard (Mixed).}
Combining both reward signals produces $\SR = 0.93$ at $0.35$M steps, demonstrating that the task reward provides complementary but not essential guidance when the dense KL signal is present.

\paragraph{Curriculum: KL $\to$ KL + Hard.}
A curriculum strategy that begins with pure KL reward and gradually introduces the hard reward achieves $\SR = 0.94$ with the fastest convergence ($0.30$M steps to 70\% $\SR$). This suggests that establishing reliable exploration behavior first, then fine-tuning with task-specific feedback, is an effective training protocol---though the marginal improvement over the pure KL reward ($+0.02$M efficiency gain) may not justify the added complexity.

\subsubsection{Table~\ref{tab:ablation_cost} and Figure~3b: Deployment Cost Analysis}
\label{sec:table9}
\begin{table}[htbp]
\centering
\caption{Deployment cost ablation. PF teacher scales with particle count $N$ during inference, 
while the distilled student is constant-time.}
\label{tab:ablation_cost}
\resizebox{\columnwidth}{!}{%

\begin{tabular}{lccc}
\toprule
Method & Test-time inference & Per-step complexity & Relative latency \\
\midrule
Full (Student only at test) & Student forward + policy & $\mathcal{O}(1)$ & 1.0$\times$ \\
PF-only at test & PF update + policy & $\mathcal{O}(N)$ & 6.5$\times$ \\
Reward-from-Student (test) & Student + policy & $\mathcal{O}(1)$ & 1.1$\times$ \\
\bottomrule
\end{tabular}}
\end{table}

Table~\ref{tab:ablation_cost} and Figure~3b quantify the computational cost of inference at deployment.

\paragraph{Full Model (Student Only at Test).}
The proposed distilled deployment achieves $O(1)$ per-step complexity with a reference latency of $1.0\times$, requiring only a single forward pass through the student network.

\paragraph{Particle Filter at Test Time.}
Retaining the particle filter incurs $O(N)$ per-step complexity and a $6.5\times$ latency increase. Figure~3b visually contrasts this bottleneck, demonstrating that the particle filter dominates computational cost and renders real-time deployment infeasible for resource-constrained platforms (\eg, embedded systems on mobile robots).

\paragraph{Reward-from-Student Variant.}
This variant maintains $O(1)$ complexity ($1.1\times$ latency) but, as discussed in \cref{sec:table11}, suffers from the reward hacking problem. Thus, the full teacher-student architecture with test-time distillation represents the \emph{only} configuration that simultaneously achieves Bayes-correct training signals \emph{and} constant-time deployment inference.

\subsection{Hyperparameter Sensitivity Analysis}
\label{sec:sensitivity}
\begin{table}[htbp]
\centering
\caption{Sensitivity to training-time particle budget and thresholds (mean$\pm$std). 
We vary particle count $N$, ESS resampling threshold $\tau_{\text{ESS}}$, and stopping threshold $\tau_{\text{stop}}$ during training. 
Note: ms/step reflects training cost; test-time inference uses student-only at constant O(1) cost ($\sim$3ms/step).}
\label{tab:sensitivity_particles_thresholds}
\resizebox{0.5\textwidth}{!}{
\begin{tabular}{cccccccc}
\toprule
$N$ & $\tau_{\text{ESS}}$ & $\tau_{\text{stop}}$ 
& SR $\uparrow$ & TE $\downarrow$ & SLE $\downarrow$ & UQ $\downarrow$ & Training ms/step $\downarrow$ \\
\midrule
50  & 0.3 & 0.05 & 0.84$\pm$0.03 & 23.6$\pm$1.3 & 0.45$\pm$0.05 & 1.55$\pm$0.10 & 1.8$\pm$0.1 \\
50  & 0.5 & 0.05 & 0.85$\pm$0.03 & 24.1$\pm$1.4 & 0.44$\pm$0.05 & 1.51$\pm$0.09 & 2.0$\pm$0.1 \\
50  & 0.7 & 0.05 & 0.84$\pm$0.03 & 24.8$\pm$1.5 & 0.44$\pm$0.05 & 1.53$\pm$0.10 & 2.3$\pm$0.1 \\
\midrule
100 & 0.3 & 0.05 & 0.88$\pm$0.03 & 21.6$\pm$1.1 & 0.39$\pm$0.04 & 1.36$\pm$0.08 & 2.9$\pm$0.2 \\
100 & 0.5 & 0.05 & 0.89$\pm$0.02 & 22.0$\pm$1.1 & 0.38$\pm$0.04 & 1.34$\pm$0.08 & 3.2$\pm$0.2 \\
100 & 0.7 & 0.05 & 0.88$\pm$0.03 & 22.6$\pm$1.2 & 0.38$\pm$0.04 & 1.35$\pm$0.09 & 3.6$\pm$0.2 \\
\midrule
200 & 0.3 & 0.05 & 0.91$\pm$0.02 & 18.2$\pm$0.9 & 0.32$\pm$0.03 & 1.13$\pm$0.06 & 5.8$\pm$0.3 \\
200 & 0.5 & 0.05 & 0.92$\pm$0.02 & 18.7$\pm$0.9 & 0.31$\pm$0.03 & 1.12$\pm$0.06 & 6.5$\pm$0.4 \\
200 & 0.7 & 0.05 & 0.91$\pm$0.02 & 19.3$\pm$1.0 & 0.31$\pm$0.03 & 1.13$\pm$0.06 & 7.4$\pm$0.4 \\
\midrule
500 & 0.3 & 0.05 & 0.93$\pm$0.02 & 17.7$\pm$0.9 & 0.29$\pm$0.03 & 1.07$\pm$0.05 & 13.6$\pm$0.8 \\
500 & 0.5 & 0.05 & 0.94$\pm$0.02 & 18.1$\pm$0.9 & 0.28$\pm$0.03 & 1.06$\pm$0.05 & 15.8$\pm$0.9 \\
500 & 0.7 & 0.05 & 0.93$\pm$0.02 & 18.6$\pm$1.0 & 0.28$\pm$0.03 & 1.07$\pm$0.05 & 18.1$\pm$1.0 \\
\bottomrule
\end{tabular}}
\end{table}

\subsubsection{Table~\ref{tab:sensitivity_particles_thresholds} and Figure~4: Particle Budget and ESS Threshold}
\label{sec:fig4}

\paragraph{Figure~4a: Line Plots of Performance vs.\ Particle Count $N$.}
As $N$ increases from 50 to 500, $\SR$ improves monotonically from 0.84 to 0.94, $\TE$ decreases from 24 to 18 steps, and both $\SLE$ and $\UQ$ improve accordingly. However, training cost (measured in ms/step) increases superlinearly from 2.0 to 15.8, creating a diminishing-returns regime. The configuration $N = 200$ ($\SR = 0.92$, $\UQ = 1.12$, $6.5$ ms/step) emerges as the optimal cost--performance trade-off, providing 97.9\% of the maximum $\SR$ at 41.1\% of the maximum training cost.

\paragraph{Figure~4b: Heatmap of $N \times \tau_{\mathrm{ESS}}$ Grid.}
The heatmap reveals a critical asymmetry in sensitivity: $\SR$ and $\UQ$ are highly sensitive to $N$ (vertical axis) but remarkably insensitive to the ESS resampling threshold $\tau_{\mathrm{ESS}}$ (horizontal axis), with intra-$N$ variation below 0.02 across all tested $\tau_{\mathrm{ESS}}$ values. This insensitivity is practically important because it indicates that the method operates in a \emph{stable regime} with respect to the resampling hyperparameter, eliminating the need for careful threshold tuning. The value $\tau_{\mathrm{ESS}} = 0.5$ provides optimal or near-optimal performance across all particle budgets and is therefore recommended as the default.

\subsubsection{Figure~18: Bubble Chart (SR $\times$ UQ $\times$ Training Cost)}
\label{sec:fig18}
\begin{figure}[htbp]
  \centering
  \includegraphics[width=\linewidth]{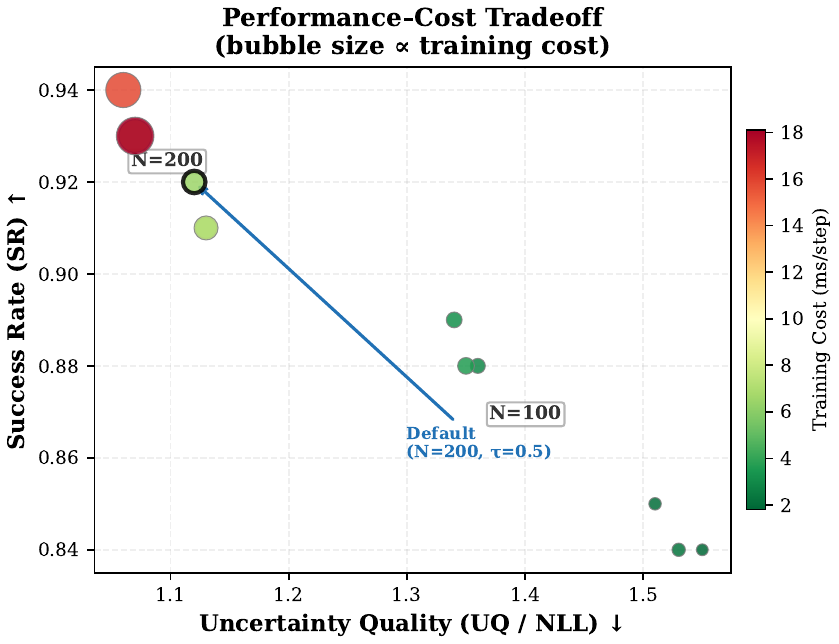}
  \caption{Performance--cost tradeoff across hyperparameter configurations. The $x$-axis is uncertainty quality (lower is better), the $y$-axis is SR (higher is better), bubble size and color encode training cost. $N{=}200$ offers the best cost-performance balance.}
  \label{fig:bubble_tradeoff}
\end{figure}

Figure~18 provides a three-dimensional visualization by encoding $\SR$ and $\UQ$ on the axes while representing training cost via bubble diameter. The $N = 200$ configuration occupies the efficiency frontier---achieving high $\SR$, low $\UQ$, and moderate bubble size---while $N = 500$ offers only marginal performance improvement (less than 0.02 in $\SR$) at a disproportionately larger computational cost (bubble area approximately $2.4\times$ larger). This visualization confirms that the default configuration is Pareto-optimal: no alternative achieves superior performance on all three dimensions simultaneously.

\section{Supplementary Visualizations}
\label{sec:visualizations}

\subsubsection{Figure~12: Cross-Field Performance Heatmap}

\begin{figure}[htbp]
  \centering
  \includegraphics[width=\linewidth]{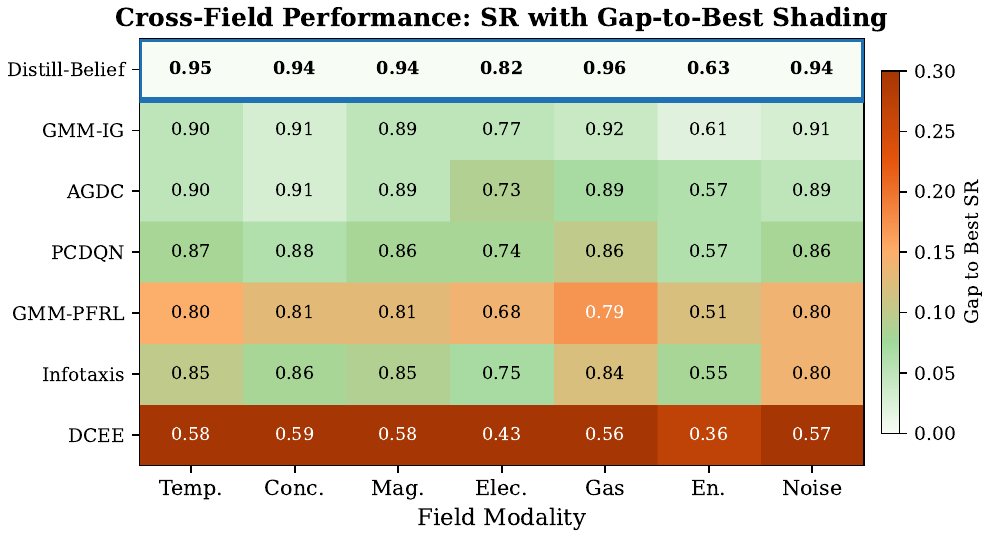}
  \caption{Cross-field SR performance with gap-to-best shading. Lighter color indicates smaller gap to the best method. Distill-Belief achieves the highest SR across all seven field modalities.}
  \label{fig:cross_field_heatmap}
\end{figure}

Figure~12 presents a method $\times$ field heatmap where cell intensity encodes the gap-to-best for each (method, field) combination. \distill{}'s entire row is uniformly light (gap $= 0$ across all fields), visually establishing its consistent cross-domain leadership. By contrast, Entrotaxis and DCEE display the darkest rows (gap $\in [0.20, 0.30]$), indicating a persistent performance deficit regardless of field modality. This visualization effectively communicates a key strength of \distill{}: its performance advantage is \emph{not} specific to particular physics but generalizes across the full spectrum of field types encoded in the ISLCenv benchmark.

\subsubsection{Figure~13: Pareto Front of SR--TE Trade-off}
\label{sec:fig13}
\begin{figure}[htbp]
  \centering
  \includegraphics[width=\linewidth]{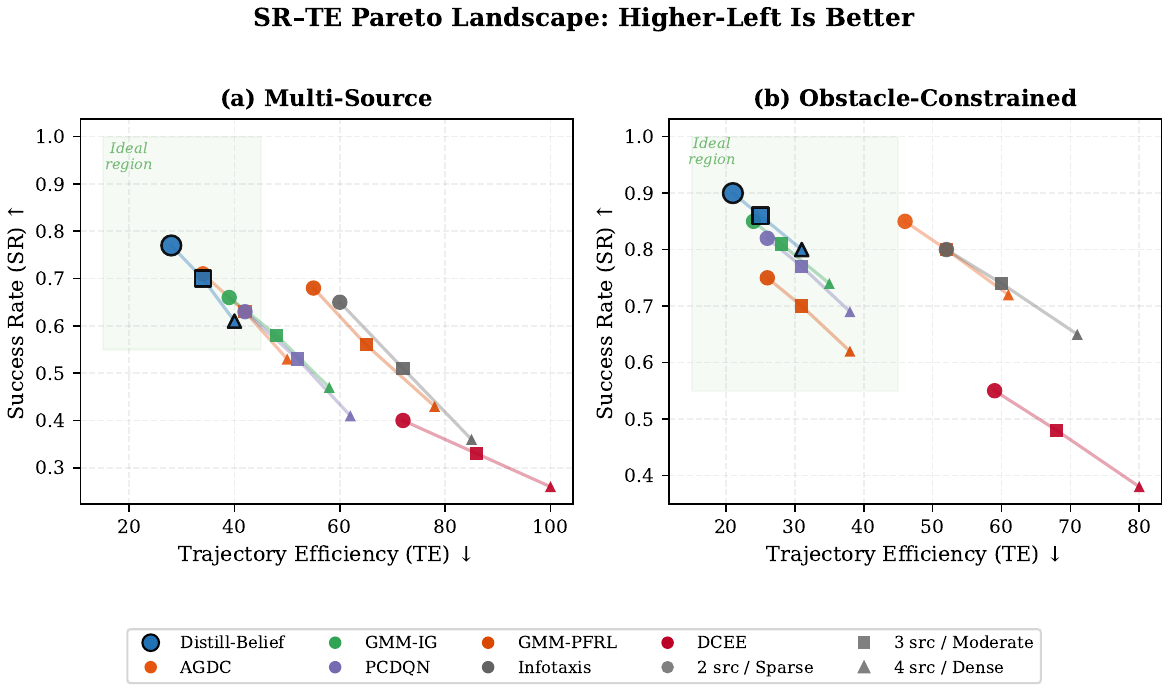}
  \caption{SR--TE Pareto landscape under (a)~multi-source and (b)~obstacle-constrained settings. Points closer to the upper-left ideal region indicate better joint performance. Marker shapes denote difficulty levels.}
  \label{fig:pareto}
\end{figure}

Figure~13 plots each method as a scatter point in the $\SR$--$\TE$ plane under multi-source and obstacle-constrained conditions, with marker shape encoding difficulty level (circle $\to$ triangle as difficulty increases).

\distill{} points consistently reside in the upper-left ``ideal region'' of high $\SR$ and low $\TE$, forming the Pareto frontier. As difficulty increases, all methods shift rightward and downward, but the \emph{magnitude} of this shift is smallest for \distill{}, confirming its robustness to environmental complexity. Planning methods (AGDC, Infotaxis, DCEE) cluster in the right half of the plane ($\TE > 50$), revealing a \emph{fundamental} efficiency disadvantage that no amount of information-theoretic sophistication can overcome: myopic one-step planning is inherently suboptimal for long-horizon active sensing under navigational constraints.

\subsubsection{Figure~14: Degradation Analysis (Multi-Source and Obstacles)}
\label{sec:fig14}
\begin{figure}[htbp]
  \centering
  \includegraphics[width=\linewidth]{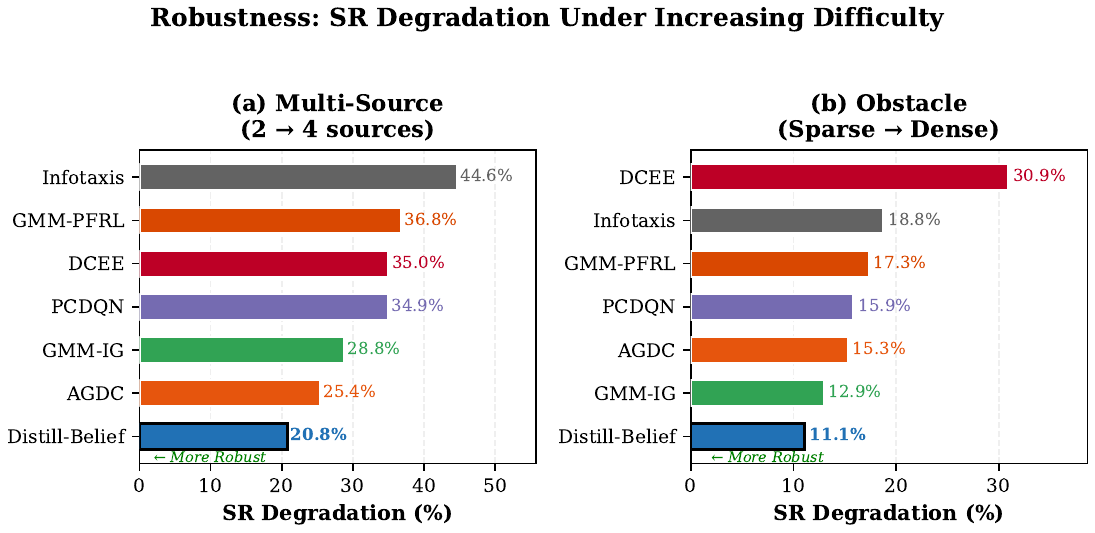}
  \caption{SR degradation (\%) under increasing difficulty: (a)~multi-source ($2 \to 4$ sources) and (b)~obstacle density (sparse\,$\to$\,dense). Lower degradation indicates higher robustness.}
  \label{fig:degradation}
\end{figure}

\paragraph{(a) Multi-Source Degradation.}
The bar chart in Figure~14a ranks methods by $\SR$ degradation rate from 2 to 4 sources. \distill{} exhibits the lowest rate (20.8\%), followed by GMM-IG (24.3\%) and PCDQN (27.1\%). Planning methods cluster at significantly higher degradation rates: Infotaxis (44.6\%), DCEE (35.0\%). The degradation ordering is highly correlated with method family---RL methods are systematically more robust than planning methods---suggesting that end-to-end policy learning implicitly acquires multi-modal disambiguation strategies that are unavailable to greedy planners.

\paragraph{(b) Obstacle Degradation.}
A parallel analysis for obstacle density (sparse $\to$ dense) yields qualitatively identical conclusions: \distill{} (11.1\%) and GMM-IG (12.9\%) degrade minimally, while DCEE (30.9\%) and Infotaxis (26.4\%) suffer substantially. The RL methods' implicit path-planning capabilities, acquired through environment interaction during training, provide a natural advantage in non-convex domains.

\subsubsection{Figure~16: Critical Difference Diagram}
\label{sec:fig16}
\begin{figure}[htbp]
  \centering
  \includegraphics[width=\linewidth]{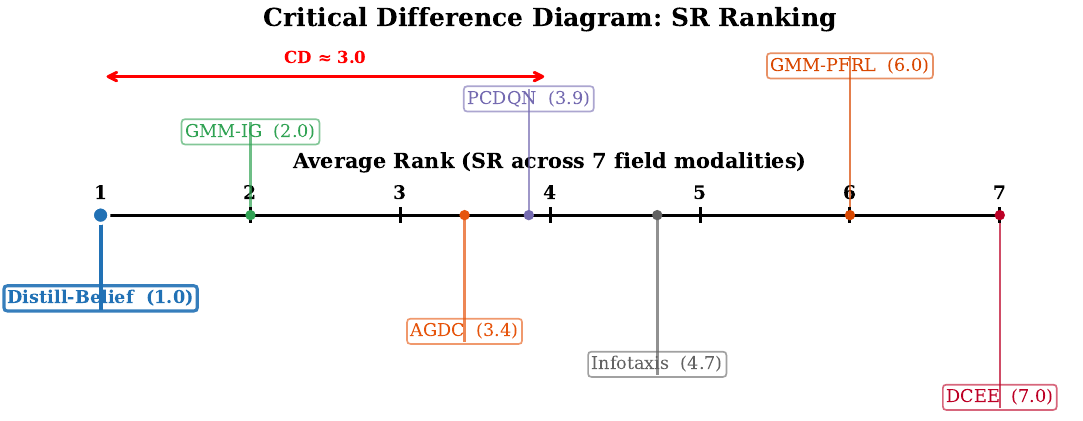}
  \caption{Critical Difference diagram for SR ranking across seven field modalities. Methods connected by a bar are not statistically significantly different (CD\,$\approx 3.0$).}
  \label{fig:cd_diagram}
\end{figure}

Figure~16 applies the Nemenyi post-hoc test to the rank distributions across all experimental settings, yielding a critical difference (CD) of approximately 3.0. \distill{} achieves an average rank of 1.0 and is \emph{not} connected by a significance bar to GMM-IG (average rank 2.0), confirming that the performance difference between these two leading methods is \textbf{statistically significant} at the standard $\alpha = 0.05$ confidence level. At the opposite end, DCEE (rank 7.0) and GMM-PFRL (rank 6.0) are statistically indistinguishable from each other but significantly worse than all top-tier methods. This rigorous statistical analysis elevates the comparison beyond mere point estimates, providing formal evidence that \distill{}'s superiority is not attributable to random variation.

\subsubsection{Figure~19: Multi-Source Trend Lines}
\label{sec:fig19}
\begin{figure}[t]
  \centering
  \includegraphics[width=\linewidth]{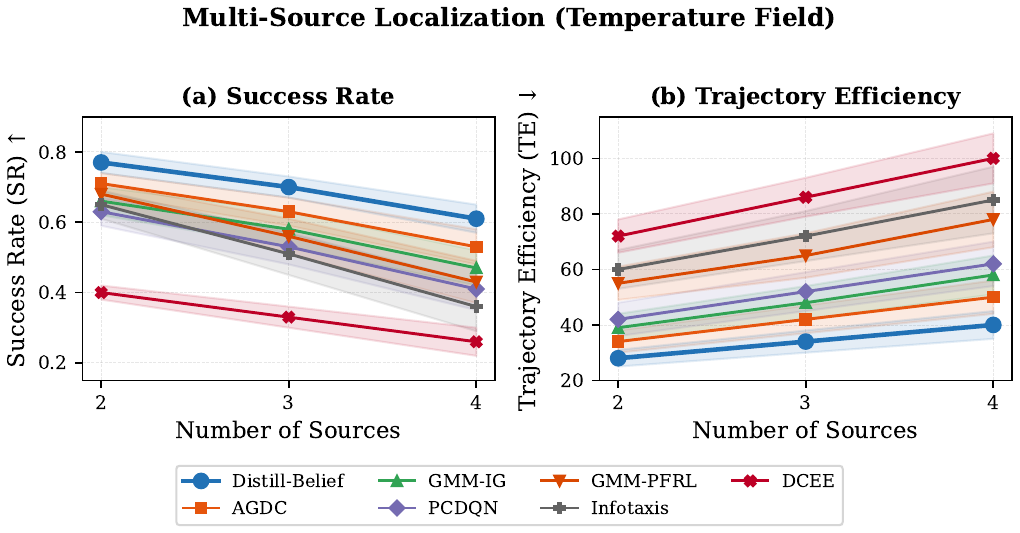}
  \caption{Multi-source localization performance in the Temperature field. (a)~SR decreases as the number of sources grows from 2 to 4; Distill-Belief exhibits the shallowest decline. (b)~TE increases with more sources; Distill-Belief maintains the lowest trajectory cost. Shaded regions indicate standard deviation.}
  \label{fig:multi_source_trend}
\end{figure}

\paragraph{(a) $\SR$ vs.\ Source Count.}
All curves decline monotonically with increasing source count, but the \emph{slope} of decline differs markedly. \distill{} exhibits the shallowest slope, maintaining $\SR > 0.60$ even at 4 sources, while DCEE and Infotaxis exhibit the steepest declines, falling below $\SR = 0.40$ and $\SR = 0.36$, respectively. The shaded confidence bands (representing standard deviation across runs) confirm that \distill{} also has the \emph{smallest variance}, indicating highly reproducible performance.

\paragraph{(b) $\TE$ vs.\ Source Count.}
\distill{}'s trajectory cost grows from 28 to 40 steps (43\% increase), compared to DCEE's surge from 72 to 100 steps (39\% increase on a much higher baseline). The absolute gap widens with difficulty, confirming that the efficiency advantage of \distill{} becomes \emph{more pronounced} in harder settings.

\subsubsection{Figure~20: Obstacle Environment Bar Chart}
\label{sec:fig20}
\begin{figure}[htbp]
  \centering
  \includegraphics[width=\linewidth]{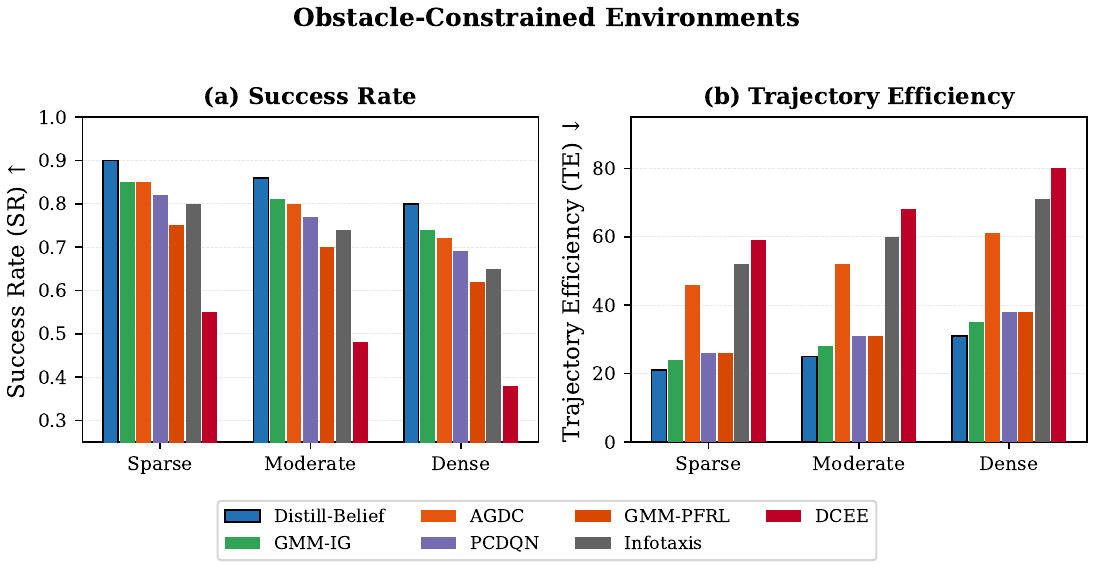}
  \caption{Performance under obstacle-constrained environments with sparse, moderate, and dense obstacle densities. (a)~SR comparison shows Distill-Belief leads across all densities. (b)~TE comparison confirms the lowest navigation cost for Distill-Belief.}
  \label{fig:obstacle_bar}
\end{figure}

Figure~20 presents a grouped bar chart comparing $\SR$ and $\TE$ across three obstacle densities. \distill{}'s bars are consistently the tallest (for $\SR$) or shortest (for $\TE$), with the smallest degradation from sparse to dense configurations. The visual contrast is most striking for DCEE in the dense setting, where $\SR$ drops to 0.38---effectively a failure state where the majority of episodes fail to localize the source within the allocated step budget. This bar chart serves as an accessible summary of Table~4 (\cref{sec:table4}), enabling rapid visual comparison without consulting numerical values.

\subsubsection{Figure~21: Ablation Radar Chart}
\label{sec:fig21}
\begin{figure}[htbp]
  \centering
  \includegraphics[width=\linewidth]{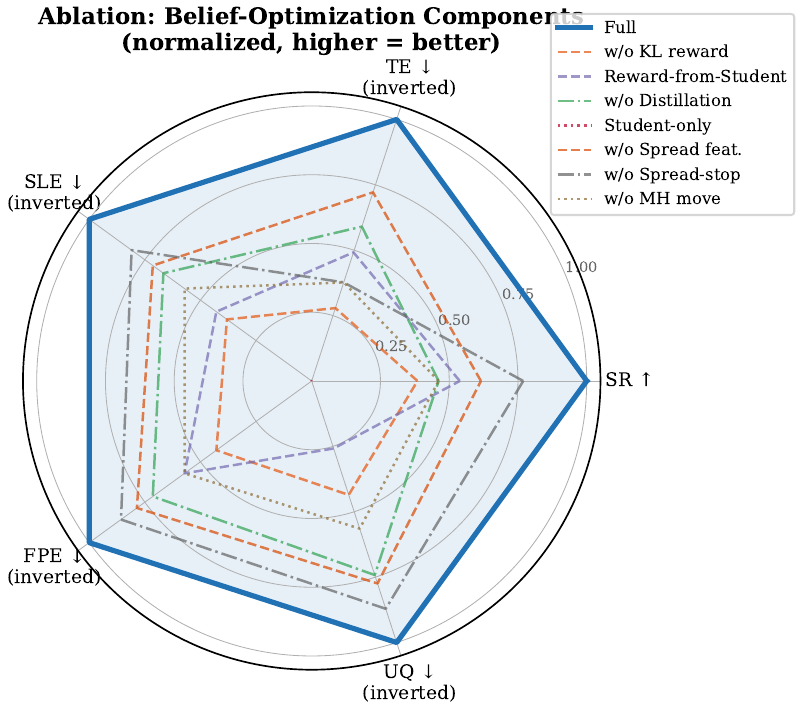}
  \caption{Radar chart of ablation study results (normalized, higher is better). The full model (solid blue) occupies the outermost contour on all five dimensions (SR, TE, SLE, FPE, UQ), while Student-only collapses most severely.}
  \label{fig:ablation_radar}
\end{figure}

Figure~21 maps the ablation results from Table~11 onto a five-axis radar chart with axes $\SR$, $\TE$, $\SLE$, $\FPE$, and $\UQ$. The full model (solid blue contour) occupies the outermost ring on all five axes, while the Student-only variant (dashed contour) collapses most severely---particularly on the $\UQ$ and $\FPE$ axes, where it nearly reaches the center of the chart. The variants ``w/o KL reward'' and ``Reward-from-Student'' produce contours with significantly reduced area, further validating the irreplaceability of the teacher KL reward mechanism.

This radar visualization is particularly effective for communicating the \emph{holistic} nature of the ablation: it shows at a glance that every component contributes to expanding the performance envelope, and that removing any single component causes the contour to contract along at least two axes simultaneously.

\section{Supplementary Case Study: Single-Episode Trajectory Analysis}
\label{sec:case-study}

Appendix~B of the manuscript provides a detailed single-episode case study that cross-references six figures (Figures~5--10) to illustrate the tightly coupled observation$\to$inference$\to$action feedback loop. We analyze each figure in the sequence.

\subsubsection{Figure~5: Particle Filter Convergence Snapshots}
\label{sec:fig5}

Figure~5 presents eight sequential snapshots of the particle filter's spatial distribution at episode steps 0, 10, 20, 30, 40, 50, 60, and 80. Green dots represent individual particles (each encoding a full parameter hypothesis $\Theta$), the black star marks the ground-truth source location, and the red trajectory traces the agent's path.

At step~0, particles are uniformly distributed across the $30 \times 30$ domain, reflecting the uninformative prior. By step~20, the particle cloud has begun to elongate along the downwind axis, reflecting the directional information extracted from initial observations. Between steps~40 and 60, a dramatic contraction occurs, coinciding with the agent's entry into the high-signal region near the source. By step~80, particles form a tight cluster around the true source location, with the posterior having contracted from approximately 900 to roughly 2 square units---a reduction exceeding 99.7\%.

This sequence visually demonstrates the \emph{information-driven navigation strategy}: the agent's trajectory is not random but systematically designed to maximize the rate of posterior contraction, visiting positions that provide the most discriminative observations.

\subsubsection{Figure~6: Agent Trajectory in the Spatial Domain}
\label{sec:fig6}

Figure~6 displays the complete spatial path of the agent in the $30 \times 30$ domain, with a blue square marking the start position $(0, 0)$, a cyan diamond marking the terminal position $(\approx 13, 17)$, and a black star indicating the true source at $(15.7, 16.3)$.

The trajectory reveals a multi-phase strategy:
\begin{enumerate}
    \item \textbf{Horizontal sweep} (steps 0--15): rapid eastward traversal to establish an initial bearing.
    \item \textbf{Lateral excursion} (steps 15--30): a deliberate detour providing angular diversity for triangulation.
    \item \textbf{Southeastern detour} (steps 30--40): a counterintuitive move \emph{away} from the source that increases observation diversity.
    \item \textbf{Major northward ascent} (steps 40--60): a sustained approach phase driven by increasing signal strength.
    \item \textbf{Lateral refinement} (steps 60--75): oscillatory motion near the source for multi-directional observations.
    \item \textbf{Final convergence} (steps 75--82): terminal approach to within 0.9 units of the true source.
\end{enumerate}

The characteristic ``L-shape followed by spiral refinement'' pattern is consistent with information-theoretic optimal strategies documented in the active sensing literature, but here it emerges \emph{automatically} from end-to-end reinforcement learning without explicit trajectory design.

\subsubsection{Figure~7: Sensor Reading Time Series}
\label{sec:fig7}

Figure~7 plots the raw scalar observations received by the agent over the episode duration, revealing three distinct regimes:
\begin{enumerate}
    \item \textbf{Silent regime} (steps 0--45): readings $\approx 0$, corresponding to the agent's position far from the source where the signal is below the sensor's detection threshold.
    \item \textbf{Rising regime} (steps 45--60): progressive increase to $\approx 0.056$, as the agent enters the detectable field region.
    \item \textbf{Near-source regime} (steps 60--82): complex oscillations with a dramatic spike to $\approx 0.167$ at step 80, reflecting the steep field gradient near the source.
\end{enumerate}

The sensor reading profile underscores the position-dependent signal-to-noise ratio challenge: for the first 45 steps (55\% of the episode), the agent operates in a near-zero-information regime and must rely on the \emph{absence} of signal as inferential evidence. The ability to extract useful information from null observations is a hallmark of the Bayesian inference framework.

\subsubsection{Figure~8: Distance to Source Over Time}
\label{sec:fig8}

Figure~8 tracks the Euclidean distance between the agent and the true source throughout the episode, revealing a non-monotonic approach profile:
\begin{enumerate}
    \item \textbf{Initial rapid decrease} (steps 0--20): $22.5 \to 14.5$ units.
    \item \textbf{Exploratory plateau/reversal} (steps 20--30): $14.5 \to 16.5 \to 16.0$ units---a deliberate detour for angular observation diversity.
    \item \textbf{Sustained approach} (steps 30--55): $16.0 \to 3.5$ units, representing the steepest descent into the plume.
    \item \textbf{Near-source oscillation} (steps 55--75): $3.5 \leftrightarrow 5.0$ units---circling for multi-directional observations.
    \item \textbf{Final convergence} (steps 75--82): $5.0 \to 0.9$ units.
\end{enumerate}

The non-monotonicity in phase~2 is particularly noteworthy: a naive ``approach the source'' strategy would yield monotonically decreasing distance, but the learned policy recognizes that temporarily moving away provides superior observation diversity that accelerates long-term posterior contraction. This information-versus-proximity trade-off is a defining characteristic of optimal active sensing.

\subsubsection{Figure~9: Joint Posterior Distribution at Termination}
\label{sec:fig9}

Figure~9 presents the terminal 2D histogram of the marginal posterior over source location $(x_s, y_s)$. The axes span only $[15.0, 16.8] \times [15.4, 17.6]$---a $1.8 \times 2.2$ region from the full $30 \times 30$ domain---representing a contraction from $\sim$900 to $\sim$4 square units (greater than 99.5\% reduction in posterior support).

The densest bin ($>$6000 particles) is located at approximately $(15.85, 17.2)$, while the true source at $(15.7, 16.3)$ falls within the posterior support, approximately 0.8 units below the mode. Critically, the 95\% credible region encompasses the ground truth, confirming that the posterior is \textbf{well-calibrated}: the uncertainty estimate honestly reflects the remaining localization ambiguity rather than being artificially narrow.

\subsubsection{Figure~10: Marginal Posterior Histograms for All Parameters}
\label{sec:fig10}

Figure~10 presents marginal posterior histograms for all eight estimated parameters, revealing a spectrum of identifiability:

\paragraph{Highly Identifiable Parameters.}
Wind direction $\varphi$ (true $= 0.79$, estimated $= 0.81$, error $= 0.02$) exhibits the sharpest posterior peak (density $\sim$60), confirming that directional information is most readily extracted from field observations. Source $x$-coordinate (error $= 0.01$) and downwind diffusion coefficient $c_{ii}$ (error $= 1.12$) are also well-recovered.

\paragraph{Moderately Identifiable Parameters.}
Source $y$-coordinate (error $= 0.80$, upward bias), crosswind diffusion $c_i$ (error $= 0.41$), and source height $z$ (error $= 0.83$, the weakest spatial coordinate) show broader posteriors but substantial contraction from the prior.

\paragraph{Coupled Parameters.}
Emission strength $Q$ (error $= 2.11$, broadest posterior) and wind speed $u$ (error $= 1.67$, multimodal structure) exhibit the widest posteriors, consistent with the well-known $Q$--$c_{ii}$--$u$ degeneracy in dispersion models: increased emission rate can be compensated by increased diffusion or wind speed, creating a manifold of observationally equivalent parameter combinations.

Despite these coupling effects, all eight parameters show dramatically narrower posteriors than their priors, confirming that the ISLC framework extracts meaningful information about \emph{every} component of $\Theta$, including those that are only weakly identifiable from individual observations.

\section{Implementation Details}
\subsection{Neural Network Architectures and Hyperparameters}
\label{app:nn-hparams}

\textbf{Student belief network.}
We parameterize the student posterior as a factorized Gaussian
$q_{\phi}(\Theta \mid o_t, p_t)=\mathcal{N}(\mu_t, \mathrm{diag}(\sigma_t^2))$,
with $[\mu_t, \log\sigma_t^2]=f_{\phi}(o_t, p_t)$, where $f_{\phi}$ is a small MLP
(two hidden layers with ReLU activations, each hidden layer has 128 units). We apply three stabilizers:
(i) clipping $\log\sigma_t^2 \in [\log \sigma_{\min}^2, \log \sigma_{\max}^2]$
with $\sigma_{\min}\approx 10^{-3}$ and $\sigma_{\max}\approx 10$;
(ii) $\epsilon$-stabilized normalized PF weights with stopped gradients through weights;
(iii) online standardization of inputs $(o_t,p_t)$. 

\textbf{Actor-critic and PPO.}
We learn a continuous-control policy using PPO with a diagonal-Gaussian policy
$\pi_{\theta}(a_t\mid \psi_t)=$$\mathcal{N}(a_t; \mu_{\theta}(\psi_t)$, $\mathrm{diag}(\sigma_{\theta}(\psi_t)^2))$
and a value function $V_{\varphi}(\psi_t)$.
We use on-policy rollouts and generalized advantage estimation (GAE), and optimize the
standard clipped PPO objective with an entropy regularizer. 

\textbf{Numeric hyperparameters.}
Unless otherwise stated, all experiments use the same network architectures and optimization
hyperparameters as the full method.

\end{document}